\documentclass{article}

\usepackage{iclr2027_conference,times}
\usepackage{amsmath,amssymb,amsfonts}
\usepackage{amsthm}
\usepackage{graphicx}
\usepackage{booktabs}
\usepackage{longtable}
\usepackage{float}
\floatstyle{ruled}
\newfloat{algorithm}{htbp}{loa}
\floatname{algorithm}{Algorithm}
\usepackage{hyperref}
\usepackage{url}
\usepackage[table]{xcolor}
\usepackage{tcolorbox}
\usepackage{pifont}
\tcbuselibrary{skins,breakable}
\usepackage{multirow}

\usepackage{makecell}               

\newcommand{\acc}[2]{\makecell{#1\\#2}}                              
\newcommand{\gain}[1]{\textcolor{green!45!black}{$\uparrow$#1}}
\newcommand{\loss}[1]{\textcolor{red!65!black}{$\downarrow$#1}}
\newcommand{\nochg}{\textcolor{black!35}{0.0}}

\definecolor{bcmag}{HTML}{D55E00}  

\definecolor{bcsign}{HTML}{0072B2} 
\definecolor{coreorange}{HTML}{FF8C00} 
\definecolor{coreblue}{HTML}{0000FF}   
\newtcolorbox{badcasebox}[2][]{%
  enhanced, breakable,
  colback=#2!5!white, colframe=#2!70!black,
  boxrule=0.6pt, arc=1.2mm,
  left=6pt, right=6pt, top=4pt, bottom=5pt,
  title={#1}, fonttitle=\bfseries\small,
  coltitle=white, colbacktitle=#2!70!black,
  attach boxed title to top left={yshift=-2mm, xshift=4mm},
  boxed title style={arc=0.8mm, boxrule=0pt, left=4pt, right=4pt, top=2pt, bottom=2pt},
  before skip=10pt, after skip=8pt
}
\usepackage{subcaption}
\usepackage{multirow}
\usepackage{enumitem}
\makeatletter
\def\section{\@startsection {section}{1}{\z@}{-1.2ex plus -0.4ex minus -.2ex}{0.9ex plus 0.2ex minus 0.2ex}{\large\sc\raggedright}}
\def\subsection{\@startsection{subsection}{2}{\z@}{-1.1ex plus -0.3ex minus -.2ex}{0.5ex plus .2ex}{\normalsize\sc\raggedright}}
\def\subsubsection{\@startsection{subsubsection}{3}{\z@}{-1.0ex plus -0.3ex minus -.2ex}{0.4ex plus .2ex}{\normalsize\sc\raggedright}}
\def\paragraph{\@startsection{paragraph}{4}{\z@}{0.9ex plus 0.4ex minus .2ex}{-1em}{\normalsize\bf}}
\g@addto@macro\normalsize{%
  \abovedisplayskip=5pt plus 2pt minus 3pt
  \belowdisplayskip=5pt plus 2pt minus 3pt
  \abovedisplayshortskip=2pt plus 1pt
  \belowdisplayshortskip=3pt plus 1pt minus 2pt}
\def\thm@space@setup{\thm@preskip=4pt plus 1pt minus 2pt
  \thm@postskip=\thm@preskip}
\makeatother
\usepackage{alphalph}

\newtheorem{theorem}{Theorem}
\newtheorem{proposition}{Proposition}

\newtheorem{corollary}{Corollary}
\newtheorem{definition}{Definition}
\newtheorem{remark}{Remark}

\usepackage{bm}
\newcommand{\vtheta}{\bm{\theta}}
\newcommand{\vz}{\bm{z}}
\newcommand{\vg}{\bm{g}}
\newcommand{\vp}{\bm{p}}
\newcommand{\vq}{\bm{q}}
\newcommand{\vm}{\bm{m}}
\newcommand{\ve}{\bm{e}}
\newcommand{\vu}{\bm{u}}
\newcommand{\vv}{\bm{v}}
\newcommand{\vd}{\bm{d}}
\newcommand{\vh}{\bm{h}}
\newcommand{\vs}{\bm{s}}
\newcommand{\vxi}{\bm{\xi}}
\newcommand{\vdelta}{\bm{\delta}}
\newcommand{\vmu}{\bm{\mu}}
\newcommand{\mJ}{\mathbf{J}}
\newcommand{\mK}{\mathbf{K}}
\newcommand{\mH}{\mathbf{H}}
\newcommand{\mF}{\mathbf{F}}
\newcommand{\mI}{\mathbf{I}}

\newcommand{\mOmega}{\bm{\Omega}}
\newcommand{\cJ}{\mathcal{J}}
\newcommand{\ind}{\mathbf{1}}

\newcommand{\cL}{\mathcal{L}}
\newcommand{\expect}{\mathbb{E}}

\newcommand{\cB}{\mathcal{B}}

\newcommand{\cD}{\mathcal{D}}

\newcommand{\cV}{\mathcal{V}}

\title{When Sparse Reward Meets Dense Distillation: Training Dynamics of On-Policy Distillation}

\author{
  \textbf{Xinke Jiang\textsuperscript{1,2,3}\thanks{Equal contribution.}},
    \textbf{Tao Feng\footnotemark[1]},
  \textbf{Zhibang Yang\textsuperscript{1,2,3}\footnotemark[1]},
  \textbf{Zhixin Zhang\textsuperscript{1,2,3}},\\
  \textbf{Weixuan Xu\textsuperscript{1}},
  \textbf{Haoyu Zhang}, 
  \textbf{Xu Chu\textsuperscript{2,3,4}\thanks{Corresponding author.}}\\[3pt]
  \textsuperscript{1}National Engineering Research Center of Software Engineering, Peking University, Beijing, China\\
  \textsuperscript{2}School of Computer Science, Peking University, Beijing, China\\
  \textsuperscript{3}Key Laboratory of High Confidence Software Technologies, Ministry of Education, Beijing, China\\
  \textsuperscript{4}Center on Frontiers of Computing Studies, Peking University, Beijing, China\\[2pt]
  \small{\texttt{\{xinkejiang, yangzb\}@stu.pku.edu.cn}}
}

\iclrfinalcopy

\begin{document}

\maketitle

\fancyhead{}

\begin{abstract}
Reinforcement learning with verifiable rewards provides a \emph{\textbf{sparse}} post-training signal: a single binary outcome evaluates the entire rollout, and every token receives the same sequence-level advantage regardless of its individual contribution.
To complement this sparse supervision, a growing family of methods adds a scalar-weighted teacher KL term to the policy-gradient objective, providing \emph{\textbf{dense}} token-level guidance that may be unreliable at some positions.
Despite the benefits of combining these signals, their interaction during optimization can destabilize joint training.
To understand how this instability develops, we study the learning dynamics of hybrid reward--distillation training through a neural tangent kernel (NTK) analysis.
We introduce the \emph{\textbf{cross-signal NTK}} $K_{DR}(n)$, a token-level statistic that measures the alignment between reward and distillation gradients at position $n$.
Through this analysis, we identify two failure modes:
\emph{\ding{182}~\textbf{Magnitude drowning}}, where the reward gradient exceeds the distillation gradient by orders of magnitude, so that even weak directional conflict can cause the distillation loss to \emph{\textbf{rise}} despite its explicit inclusion in the training objective; and
\emph{\ding{183}~\textbf{Localized directional conflict}}, where the sequence-level advantage and the teacher's position-specific distribution induce opposing updates at the same token ($K_{DR}(n)\!<\!0$).
The severity of these effects depends on the optimization regime: the gradient-norm ratio $\kappa\!=\!\|\nabla\mathcal{L}_R\|/\|\nabla\mathcal{L}_D\|$ varies by roughly an order of magnitude across tasks, and our experiments reveal an empirical threshold beyond which naive mixing can lead to persistent training collapse.
Motivated by these findings, we introduce the \emph{\textbf{M3 family}}, which combines magnitude normalization with three strategies for coordinating dense teacher supervision and sparse reward updates: a hard NTK-based mask that retains compatible teacher signals (M3-Select), a continuous relaxation of this mask (M3-Soft), and a fast--slow extragradient step that temporally separates teacher shaping from reward correction (M3-EG).
Experiments across four model backbones and four benchmarks show that M3 maintains stable training dynamics and achieves superior performance in high-$\kappa$ regimes where scalar-mixing baselines collapse.
\end{abstract}
\newcommand{\magvar}[1]{\textcolor{coreorange}{#1}}
\newcommand{\dirvar}[1]{\textcolor{coreblue}{#1}}
\section{Introduction}
\label{sec:intro}
\underline{\textbf{R}}einforcement \underline{\textbf{L}}earning with \underline{\textbf{V}}erifiable \underline{\textbf{R}}ewards (\textbf{RLVR}) has become a central approach to post-training reasoning-capable large language models~\citep{deepseekr1, openai2024o1}.
However, its supervision is \textbf{\textit{sparse}}: a single outcome reward evaluates the entire sequence, and every token receives the same sequence-level advantage regardless of its individual contribution.
Process-reward methods partially address this limitation by evaluating intermediate reasoning steps, but step-level supervision does not directly distinguish the contributions of individual tokens.
As the complementary, \textbf{teacher distillation} provides \textbf{\textit{dense}} supervision with target distribution at each token position, although the teacher's guidance may be unreliable at some positions.
Therefore, a growing family of hybrid methods therefore augments the policy-gradient objective with a scalar-weighted teacher KL term~\citep{zhao2024opsd, agarwal2024gkd}, aiming to provide token-level guidance.

Despite their empirical success~\citep{zhao2024opsd, agarwal2024gkd}, these hybrid methods can exhibit \textbf{\textit{unstable training dynamics}} and, in some cases, \textbf{\textit{catastrophic collapse}} (Figure~\ref{fig:intro_case_study}c).
The underlying difficulty is that a reliable but sparse outcome signal and a dense but imperfect proxy signal do not necessarily complement each other. Under naive scalar mixing, \textbf{they may instead compete}: \ding{182} one signal can overwhelm the other, \ding{183} or their opposing updates can cancel, progressively homogenizing the policy, eliminating reward diversity, and ultimately inducing entropy collapse.

\textbf{\textit{To understand how this instability develops, we study the learning dynamics of hybrid reward--distillation training through an NTK analysis.}}
Building on the neural tangent kernel (NTK) framework~\citep{ren2024ntk}, we examine how reward and distillation updates interact through the model's shared parameters.
We introduce the \textbf{\textit{cross-signal NTK}} $K_{DR}(n)$, the inner product between the parameter gradients contributed by the two objectives at token position $n$, which measures their local alignment while accounting for the mapping from token-level residuals to parameter updates through the model's Jacobian.
Our analysis identifies two failure modes:
\textbf{\textit{\ding{182}~Magnitude drowning.}}
The reward gradient can exceed the distillation gradient by orders of magnitude, as measured by the norm ratio $\kappa = \|\vg_R\|/\|\vg_D\|$.
In this regime, even weak negative alignment can make the increase in distillation loss caused by the reward update exceed the decrease produced by the distillation update itself.
Consequently, the reward or distillation loss can rise despite its explicit inclusion in the training objective (Figure~\ref{fig:intro_case_study}a).
\textbf{\textit{\ding{183}~Localized directional conflict.}}
The two objectives assign updates using different information: RL broadcasts a sequence-level advantage to every token, whereas distillation uses a position-specific teacher distribution.
At positions where $K_{DR}(n) < 0$ (Figure~\ref{fig:intro_case_study}b), the resulting gradient contributions oppose each other.
For a positive-advantage rollout, the teacher update can decrease the probability of a sampled token that the reward update seeks to reinforce.
For a negative-advantage rollout, it can instead reinforce a token that the reward update seeks to suppress.
Their contributions can therefore cancel in aggregate diagnostics, obscuring local conflicts.
Over longer training horizons, repeated conflicting updates can reduce diversity among rollouts and diminish within-group reward variation.
When all rollouts in a group receive the conflict reward, their relative advantages and policy gradient vanish.
Our experiments exhibit a corresponding progression from initial reward improvement to reduced reward diversity and abrupt collapse (Figure~\ref{fig:intro_case_study}c).

\textbf{\textit{The severity of these failure modes depends on the optimization regime.}}
A hybrid run may initially appear stable, with reward increasing even as the distillation loss rises.
Across architectures and tasks, $\kappa$ varies by roughly an order of magnitude, and our experiments reveal an empirical threshold beyond which naive mixing becomes prone to persistent collapse.
These findings motivate examining teacher supervision at two levels.
At the \emph{run level}, $\kappa$ indicates the degree of magnitude imbalance and helps distinguish settings where standard scalar mixing remains effective from those where it collapses.
At the \emph{token level}, the sign of $K_{DR}(n)$ distinguishes locally compatible teacher updates from conflicting ones.
Guided by this diagnosis, we introduce the \textbf{M3} family, which combines magnitude normalization with three strategies for coordinating teacher supervision and reward updates:
\textbf{M3-Select} applies a hard NTK-based mask, retaining teacher supervision only at positions where $K_{DR}(n) \geq 0$;
\textbf{M3-Soft} replaces this mask with a continuous, temperature-controlled gate;
and \textbf{M3-EG} temporally separates teacher shaping from reward correction through a fast--slow extragradient step.
We make the following contributions:
\begin{figure}[t]
\centering
\includegraphics[width=0.96\linewidth]{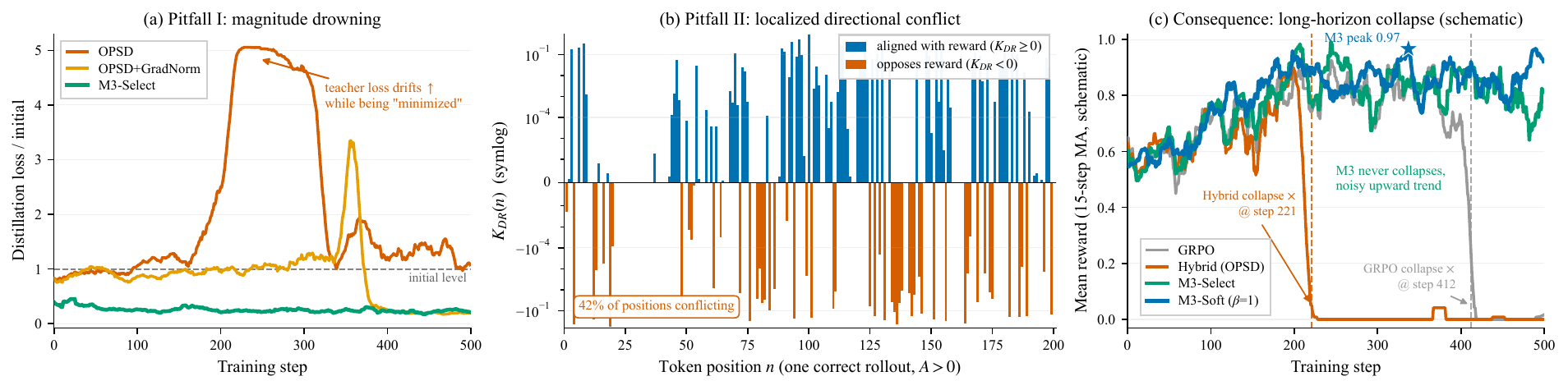}
\caption{\textbf{Two failures of linear reward--distillation mixing} (Qwen3-1.7B, GSM8K). \textbf{(a)} Distillation loss rises under OPSD, indicating magnitude drowning. \textbf{(b)} Positive and negative $K_{DR}(n)$ interleave in a correct rollout, revealing token-local conflict hidden by aggregation. \textbf{(c)} Baselines collapse by step $500$, while M3-Select and M3-Soft remain stable.}
\label{fig:intro_case_study}
\end{figure}
\begin{itemize}[leftmargin=*,itemsep=0pt,topsep=0pt]
    \item 
    We develop an NTK-based framework that characterizes the interaction between RL and distillation through a single per-token statistic, the cross-signal NTK $K_{DR}(n)$. This framework identifies two failure modes of linear mixing and explains why weight-space rebalancing methods.

    \item 
    We show that $\kappa$ remains stable throughout a run but varies by roughly an order of magnitude across tasks. A critical threshold separates a stable regime, in which naive mixing is effective, from a catastrophic regime where it collapses. This yields a simple probe-batch rule for selecting the mixing strategy before full training.

    \item 
    In high-$\kappa$ regimes, the M3 family remains stable over horizons at which scalar-mixing baselines collapse. M3-Select is the most robust variant in the highest-$\kappa$ settings, while M3-Soft recovers catastrophic cases with order-of-magnitude gains. Combined with weight averaging, which removes the gate-variance pathology identified in our analysis, M3-Soft matches or exceeds the strongest baseline across all tested architecture--dataset pairs.
\end{itemize}

\section{Related Work}
\label{sec:related}

\textbf{\textit{\ding{182} LLM Reasoning via RL and Distillation.}}
DeepSeek-R1~\citep{deepseekr1} spurred RL reasoning (GRPO~\citep{shao2024deepseekmath}, DAPO~\citep{yu2025dapo}, Dr.~GRPO~\citep{liu2025drgrpo}); OPSD~\citep{zhao2024opsd} and MiniLLM~\citep{gu2024minillm} make distillation on-policy. Hybrids differ in coupling: SDPO~\citep{hubotter2026sdpo} self-distills from a reprompt, RLSD~\citep{yang2026rlsd} keeps the teacher as magnitude-only reweighting, and HDPO~\citep{ding2026hdpo} and DPKD~\citep{li2024dpkd} interpolate losses. Unlike these methods, we study the dynamics of the coupling, identifying when dense teacher supervision conflicts with or is drowned by sparse reward updates.

\textbf{\textit{\ding{183} Multi-Objective Gradients and NTK.}}
PCGrad~\citep{yu2020pcgrad}, CAGrad~\citep{liu2021cagrad}, MGDA~\citep{sener2018mgda}, and NashMTL~\citep{navon2022nashmtl} operate on aggregate task gradients; GradNorm~\citep{chen2018gradnorm}, UW~\citep{kendall2018uncertainty}, and DWA~\citep{liu2019dwa} balance objectives in loss-weight space. Such global operations do not directly resolve conflicts that alternate across tokens, and we show that GradNorm becomes ineffective at $\kappa\!\gg\!1$ (Proposition~\ref{prop:gradnorm_degeneracy}). Prior work applies NTK to gradient conflict and imbalance~\citep{ren2024ntk,zhu2025ntkmtl}; our cross-signal NTK specializes this perspective to hybrid RL--distillation and links token-level interaction to drowning-induced collapse, complementing known RL failures such as reward hacking~\citep{skalse2022rewardhacking}, entropy collapse~\citep{yu2025dapo}, and length bias~\citep{liu2025drgrpo}.




\section{Preliminaries and Theoretical Analysis}
\label{sec:background}
\label{sec:theory}

\paragraph{Notation.}
In this paper, scalars use italic symbols, vectors bold lowercase or Greek symbols (e.g., $\vg$, $\vdelta$, $\vtheta$), and matrices bold uppercase symbols (e.g., $\mJ$, $\mK$). Calligraphic symbols denote sets and losses; $\mathbb R$ denotes real numbers. The model has $d_{\mathrm{par}}$ parameters, collected in $\vtheta\in\mathbb R^{d_{\mathrm{par}}}$. At position $n$, $\vp_S^n$ is the student's token distribution and $p_S^n(v)=[\vp_S^n]_v$ the probability of token $v$. We distinguish kernel matrix $\mK(n,m)$ from the scalar interaction score $K_{DR}(n)$ and use $\varphi_n$ for alignment angles.

\subsection{Problem Setup and NTK Preliminaries}
\label{sec:problem_setup}
We train student model $\pi_{\vtheta}$ using teacher predicts and response-level rewards. For a prompt--answer pair $(x,y^*)\sim\cD$, the student generates a response $\hat y$ of length $N$; $\hat y_{<n}$ is its prefix before position $n$.

\textbf{\textit{\ding{182} On-Policy Self-Distillation (OPSD).}}
OPSD~\citep{zhao2024opsd} learns from responses sampled by the student itself. A frozen teacher receives the ground-truth answer as additional context, giving $\vp_T^n=\pi_{\vtheta_0}(\cdot\mid x,y^*,\hat y_{<n})$, while $\vp_S^n=\pi_{\vtheta}(\cdot\mid x,\hat y_{<n})$. Their distributions are compared at each position, with each contribution capped at $\tau$:
\begin{equation}
\cL_D(\vtheta) = \expect_{(x,y^*)\sim\cD}\; \expect_{\hat{y}\sim\pi_{\vtheta}(\cdot\,|\,x)}\!\left[\; \frac{1}{N}\sum_{n=1}^{N} \min\!\Big( D_{\lambda}\big(\vp_T^n \,\big\|\, \vp_S^n\big),\; \tau \Big)\;\right],
\end{equation}
Here, $D_\lambda(\vp\|\vq)=\lambda\mathrm{KL}(\vp\|\vm)+(1-\lambda)\mathrm{KL}(\vq\|\vm)$ is the generalized Jensen--Shannon divergence, with $\vm=\lambda\vp+(1-\lambda)\vq$. We use $\lambda=1/2$, giving both distributions equal weight. The teacher thus provides dense and position-specific supervision; KL limits are given in Appendix~\ref{app:distillation_details}.

\textbf{\textit{\ding{183} Group Relative Policy Optimization (GRPO).}}
GRPO~\citep{shao2024deepseekmath}, a critic-free variant of PPO~\citep{schulman2017ppo}, samples $G$ responses per prompt and assigns each a verifiable reward $r^{(i)}=r(\hat y^{(i)},y^*)\in\{0,1\}$. Its advantage $A_i=(r^{(i)}-\bar r)/(\sigma_r+\epsilon)$ compares that reward with the group mean $\bar r$ and standard deviation $\sigma_r$, where $\epsilon>0$ prevents division by zero:
\begin{equation}
  \cL_R(\vtheta)
  =
  -\expect_{(x,y^*)\sim\cD}\;
  \expect_{\{\hat{y}^{(i)}\}_{i=1}^{G}\sim\pi_{\vtheta}(\cdot\,|\,x)}
  \left[
  \frac{1}{G}
  \sum_{i=1}^{G}
  A_i\,
  \frac{1}{N_i}
  \sum_{n=1}^{N_i}
  \log\pi_{\vtheta}
  \!\left(
  \hat{y}^{(i)}_n
  \,\middle|\,
  x,\hat{y}^{(i)}_{<n}
  \right)
  \right],
  \end{equation}
where $N_i=|\hat y^{(i)}|$. Positive advantages encourage sampled responses and negative advantages discourage them. The same advantage weights every token, without identifying which tokens caused success or failure. Gradients hold the sampled responses and advantages fixed.

\textbf{\textit{\ding{184} NTK in Learning Dynamics.}}
Following~\citep{ren2024ntk}, let $\vz^n\in\mathbb R^{|\cV|}$ be the policy's logits, the scores converted by softmax into probabilities over vocabulary $\cV$. The Jacobian $\mJ^n=\nabla_{\vtheta}\vz^n\in\mathbb R^{d_{\mathrm{par}}\times|\cV|}$ describes their dependence on the parameters. For $\cL=N^{-1}\sum_n\ell^n$, the chain rule gives
\begin{equation}
\nabla_{\vtheta}\cL=\frac1N\sum_n\mJ^n\vdelta^n,
\qquad \vdelta^n=\nabla_{\vz^n}\ell^n,
\qquad \mK(n,m)=(\mJ^n)^\top\mJ^m.
  \end{equation}
The residual $\vdelta^n$ describes how the loss changes with each logit. The empirical NTK $\mK(n,m)$~\citep{jacot2018ntk} couples positions through their shared parameters: an update driven by position $m$ can change predictions at $n$. Following \citet{ren2024ntk}, we use the Action--Kernel--Gradient (AKG) decomposition to analyze these changes.

\subsection{Hybrid Update Dynamics and Loss Interactions}
\label{sec:loss_decomposition}
In OPSD, we assume the hybrid loss between policy update and distillation loss is $\cL_H=(1-\alpha)\cL_R+\alpha\cL_D$, where $\alpha\in[0,1]$ weights distillation. A step of size $\eta$ gives $\vtheta_{t+1}=\vtheta_t-\eta\nabla_{\vtheta}\cL_H$. Expanding the logits to first order and summing over $T$ steps yields:
\begin{equation}
\vz^n_T = \vz^n_0 + \sum_{t=0}^{T-1} \Delta\vz^n_t,
\qquad
\Delta\vz^n_t = -\frac{\eta}{N}\sum_{m=1}^{N} \mK_t(n,m)\Big[(1-\alpha)\,\vdelta_R^m + \alpha\,\vdelta_D^m\Big] + O(\eta^2),
\label{eq:logits_dynamics}
\end{equation}
where $\vz_t^n=\vz^n(\vtheta_t)$ and $\Delta\vz_t^n=\vz_{t+1}^n-\vz_t^n$. At each step, the kernel maps the combined residual at every position $m$ to a change in the logits at $n$. The residuals are evaluated at step $t$:
\begin{equation}
\vdelta_R^m=-A(\ve_{\hat y_m}-\vp_S^m),
\qquad \vdelta_D^m=\vp_S^m-\vp_T^m.
\end{equation}
The one-hot vector $\ve_{\hat y_m}$ selects the sampled token. The reward residual encourages or discourages this token according to $A$; the teacher residual compares the full distributions. We use a local forward-KL model of distillation: near agreement, the unclipped JSD gradient is $\lambda(1-\lambda)\vdelta_D^m+O(\|\vp_S^m-\vp_T^m\|^2)$, with the leading constant absorbed into the teacher scale. Clipped tokens contribute zero gradient (Appendix~\ref{app:distillation_details}).

\textbf{\textit{First-Order Loss Changes.}}
With $\vg_R=\nabla_{\vtheta}\cL_R$ and $\vg_D=\nabla_{\vtheta}\cL_D$, the same update gives
\begin{equation}
\Delta \cL_R
\approx -\eta\Big[(1-\alpha)\|\vg_R\|^2 + \alpha\langle \vg_R, \vg_D\rangle\Big],
\qquad
\Delta \cL_D
\approx -\eta\Big[\alpha\|\vg_D\|^2 + (1-\alpha)\langle \vg_D, \vg_R\rangle\Big],
\label{eq:teacher_loss_change}
\end{equation}
Each squared-gradient term describes an objective's own decrease; the inner product describes the other update's effect. Positive alignment helps both objectives, while negative alignment opposes their progress. A loss increases only when this opposing contribution exceeds its own decrease within the first-order approximation.

\subsection{Token-Level Decomposition of Gradient Interactions}

\textbf{\textit{Cross-Position Interactions.}}
The overall inner product can hide local conflicts. Each term pairs the teacher signal at $n$ with the reward signal at $m$ through the shared kernel. The sum includes same-position and cross-position interactions, whose positive and negative contributions can cancel. Expanding both gradients gives that:
\begin{equation}
\langle \vg_D, \vg_R \rangle = \frac{1}{N^2}\sum_{n,m}
\underbrace{(\vdelta_D^n)^\top}_{\text{teacher residual}}
\underbrace{\mK(n,m)}_{\text{ kernel coupling}}
\underbrace{\vdelta_R^m}_{\text{ reward residual}}.
\label{eq:inner_product}
\end{equation}

\textbf{\textit{Gradient Magnitude and Alignment.}}
Expanding the squared norms in the same way gives the magnitude ratio when $\|\vg_D\|>0$:
\begin{equation}
\kappa^2 \;=\; \frac{\|\vg_R\|^2}{\|\vg_D\|^2}
\;=\; \frac{\sum_{n,m}(\vdelta_R^n)^\top\,\mK(n,m)\,\vdelta_R^m}{\sum_{n,m}(\vdelta_D^n)^\top\,\mK(n,m)\,\vdelta_D^m},
\label{eq:kappa_ntk}
\end{equation}
the numerator measures reward-gradient strength and the denominator teacher-gradient strength; $\kappa\gg1$ indicates strong imbalance. For nonzero gradients, the normalized inner product $\Phi_{DR}=\langle\vg_D,\vg_R\rangle/(\|\vg_D\|\|\vg_R\|)\in[-1,1]$ measures direction independently of magnitude. Positive values indicate alignment and negative values interference. These quantities arise from one gradient Gram matrix: its diagonal contains squared norms and its off-diagonal contains the cross inner product (Appendix~\ref{app:interaction_details}).

\subsection{Token-Level Conflict and Magnitude Drowning}
\label{sec:cross_signal_ntk}

We first locate conflict at individual positions, then examine how gradient imbalance amplifies its effect on the teacher loss.

\begin{definition}[Cross-Signal Token-Level NTK]
\label{def:cross_signal_ntk}
Using the local residual model of Section~\ref{sec:loss_decomposition}, define
\begin{equation}K_{DR}(n)=\left\langle \mJ^n\vdelta_D^n,\mJ^n\vdelta_R^n\right\rangle,\qquad \vdelta_D^n=\vp_S^n-\vp_T^n,\qquad\vdelta_R^n=-A\left(\ve_{\hat y_n}-\vp_S^n\right).
\label{eq:cross_signal_score}
\end{equation}
where $A=(r-\bar r)/(\sigma_r+\epsilon)$. This scalar compares the teacher and reward parameter gradients contributed by the same position.
\end{definition}
Positive, negative, and zero scores define the sets $\Omega_+$, $\Omega_-$, and $\Omega_0$. A zero score also includes vanishing gradients. $K_{DR}(n)$ is the $m=n$ summand in Eq.~\ref{eq:inner_product}, before the common factor $N^{-2}$; it does not include cross-position interactions.

\textbf{\textit{\ding{182} Token-Level Conflict.}}
On a positive-advantage response, a negative $K_{DR}(n)$ means that the teacher's same-position contribution lowers the sampled token's log-probability. The full update need not do so, since reward and other-position contributions also matter. Corollary~\ref{cor:logit_degradation} derives this local effect. The conflict rate $C_{\mathrm{NTK}}=|\{n:K_{DR}(n)<0\}|/N$ is approximately $40\%$ in our measurements, an empirical observation rather than a consequence of the definition.

\textbf{\textit{\ding{183} Magnitude Drowning.}}
Write $\cos\varphi=\Phi_{DR}$. Equation~\ref{eq:teacher_loss_change} gives $\Delta\cL_D\approx-\eta\|\vg_D\|^2[\alpha+(1-\alpha)\kappa\cos\varphi]$. For nonzero gradients and $0<\alpha<1$, its first-order sign condition is
\(
\Delta\cL_D>0
\iff \cos\varphi<-\frac{\alpha}{(1-\alpha)\kappa}.
\label{eq:teacher_rises}
\)
For fixed $\alpha$, the threshold approaches zero from below as $O(\kappa^{-1})$: when the reward gradient is large, even weak negative alignment can outweigh the teacher's own descent. Imbalance alone is insufficient; negative alignment must also satisfy the threshold.

As the student approaches the teacher distribution, $\vdelta_D^n$ shrinks, while the reward residual need not shrink at the same rate. The kernel maps both into parameter space, so their magnitudes and directions jointly determine $\kappa$. Our experiments identify magnitude imbalance as the dominant failure mode, with ratios reaching $1.3\times10^4$. Ratios are approximately stable across the tested LoRA ranks $8$--$64$ but vary substantially across tasks (Proposition~\ref{prop:kappa_scaling}). Extended scaling analysis and projection-baseline limitations appear in Appendices~\ref{app:interaction_details} and~\ref{app:pcgrad}.

\subsection{Reward-Signal Degeneration and Training Collapse}
\label{sec:consequence}

Repeated teacher contributions at conflicting positions may suppress successful response patterns. If the resulting responses become less diverse in reward, GRPO receives less information for distinguishing them. This is a possible training mechanism, not a direct consequence of the local gradient identity.
The final step is exact: if all $G$ sampled responses receive the same reward, then $r^{(i)}=\bar r$, $\sigma_r=0$, and every advantage $A_i=0$. All reward residuals vanish, so that group contributes no reward gradient, $\vg_R^{(x)}=0$. Only the teacher term can contribute to its hybrid update. However, one such group does not establish collapse: it may contain all successes or all failures, and later sampling or updates from other prompts may restore reward variation. Persistent failure requires poor responses to remain dominant without recovery of a useful reward signal. The abrupt drops in Figure~\ref{fig:intro_case_study}c are consistent with this mechanism; Section~\ref{sec:drowning} provides empirical tests.

\begin{remark}[Empirical Gate-Selection Threshold]
\label{rem:kappa_threshold}
Our configurations show a transition near $\kappa^*\approx5\times10^3$: hard masking is most useful at high imbalance, while soft gating often retains more useful supervision at lower imbalance. This empirical guideline is not a universal collapse threshold; it depends on the compatible-token fraction, conflict strength, and gate-estimation error. Section~\ref{sec:methods} introduces the gating methods.
\end{remark}

\section{Methodology}
\label{sec:methods}
\label{sec:m3_select}

M3 turns the preceding analysis into two coupled decisions: how strongly each signal enters the update, and where teacher guidance is admitted. We first normalize token residuals, then allocate teacher weight using local compatibility (Figure~\ref{fig:method}). M3-Select and M3-Soft implement this allocation with hard and continuous gates. M3-EG extends the same principle to update timing, letting the teacher shape where the reward direction is evaluated.

\begin{figure}[t]
\centering
\includegraphics[width=\linewidth]{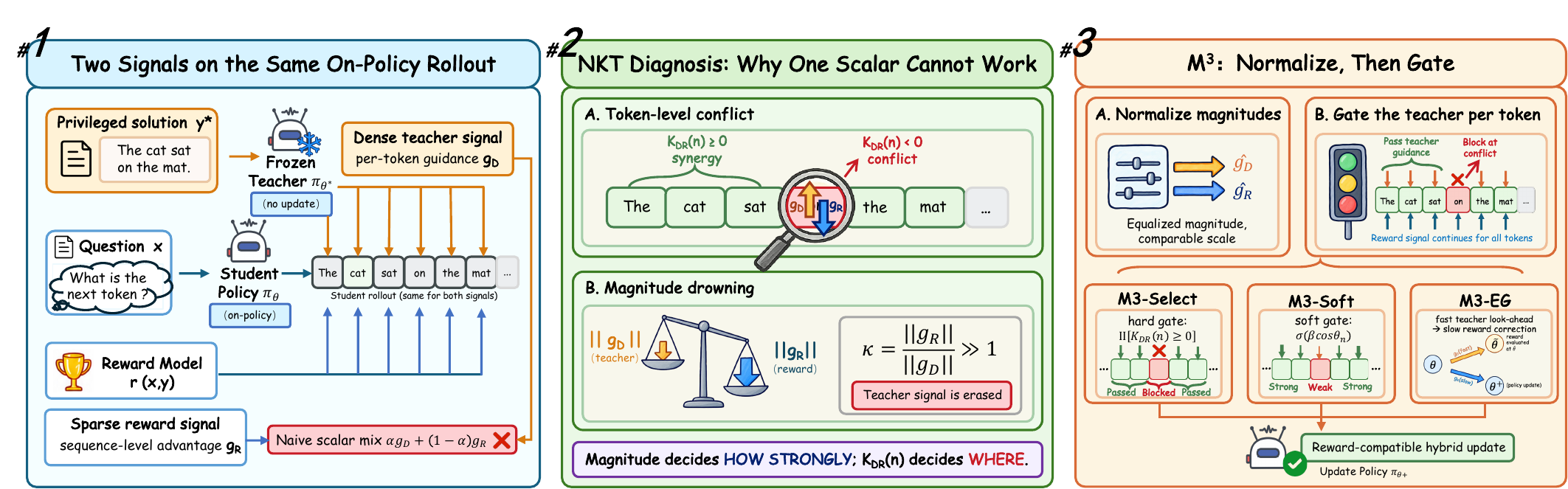}
\caption{\textbf{From interaction analysis to hybrid updates.} Shared parameters couple dense teacher and sparse reward supervision. M3 controls residual scale and allocates teacher influence by token-level compatibility; its extragradient extension separates teacher shaping from reward correction.}
\label{fig:method}
\end{figure}

\subsection{Controlling Signal Scale}
\label{sec:m3_design}

In a global mixture, comparable weighted gradient norms require $\alpha\approx\kappa/(1+\kappa)$; the coefficient must compensate for scale before it can express a preference between the signals. M3 instead controls scale at the residual level, before conversion into a parameter update. Using the residual convention of Section~\ref{sec:loss_decomposition}, define
\begin{equation}
\widehat\vdelta_q^n=\frac{\vdelta_q^n}{\|\vdelta_q^n\|+\epsilon},
\qquad
\vh_q^n=\mJ^n\widehat\vdelta_q^n,
\qquad q\in\{D,R\}.
\label{eq:m3_residual_normalization}
\end{equation}
Here $\epsilon>0$ stabilizes division and keeps zero residuals zero. Normalization reduces the dependence of mixing on raw residual magnitudes. Its positive scaling preserves the sign of $K_{DR}(n)$, but does not equalize parameter-gradient norms after multiplication by $\mJ^n$. It therefore controls residual scale while leaving the directional coordination problem to the gate.
The teacher budget acts on these rescaled contributions, while a common update scale controls their overall size. This separates the intended allocation of supervision from raw magnitude disparity that constrains naive mixing.

\subsection{Allocating Teacher Influence Across Positions}
\label{sec:m3_token_gating}

With a token-dependent teacher weight $\alpha_n\in[0,\alpha_{\max}]$ and $\alpha_{\max}<1$, the common update:
\begin{equation}
\begin{aligned}
\vh_H^n&=\underbrace{\alpha_n\vh_D^n}_{\text{teacher contribution}}
+\underbrace{(1-\alpha_n)\vh_R^n}_{\text{reward contribution}},\\
\vg_H&=\frac1N\sum_{n=1}^N\vh_H^n,
\qquad
\vtheta_{t+1}=\vtheta_t-\eta\bar s\,\vg_H.
\end{aligned}
\label{eq:m3_select_update}
\end{equation}
The shared exponential-moving-average scale $\bar s$ sets the update magnitude; $\alpha_n$ controls its local composition. Teacher and reward weights vary together: rejecting the teacher restores reward weight to one. The remaining choice is how compatibility determines $\alpha_n$.

\paragraph{M3-Select: retain compatible teacher contributions.}
The hard gate directly implements the sign criterion:
\begin{equation}
\alpha_n^*=\alpha_{\max}\mathbb{I}\{K_{DR}(n)\geq0\}.
\label{eq:m3_select_alpha}
\end{equation}
This removes negative teacher projections onto the same-position reward direction. Since normalization preserves signs, the resulting contribution satisfies
\begin{equation}
\begin{aligned}
\langle\vh_H^n,\vh_R^n\rangle
=(1-\alpha_n^*)\|\vh_R^n\|^2
+\alpha_n^*\langle\vh_D^n,\vh_R^n\rangle\geq(1-\alpha_{\max})\|\vh_R^n\|^2.
\end{aligned}
\label{eq:m3_local_projection}
\end{equation}
Rejected positions retain the full local reward contribution. This is a same-position guarantee: cross-position interactions in the aggregate update remain. A zero score is admitted by convention and may simply reflect a vanishing signal.
The gate ceiling controls how much teacher influence an admitted position receives; the score's sign controls whether it receives that influence at all. These two choices need not be tied to one global loss weight.

\paragraph{M3-Soft: vary teacher influence continuously.}
Hard decisions can change abruptly near zero alignment. M3-Soft smooths this transition, retaining partial supervision when compatibility is weak or uncertain:
\begin{equation}
\begin{aligned}
\alpha_n&=\alpha_{\max}\sigma(\beta\cos\varphi_n), \quad \cos\varphi_n=\frac{K_{DR}(n)}{\|\vg_D^n\|\|\vg_R^n\|+\varepsilon}.
\end{aligned}
\label{eq:m3_soft}
\end{equation}
Here $\sigma$ is the logistic sigmoid, $\beta\geq0$ controls sharpness, and $\varepsilon>0$ stabilizes the score. Larger $\beta$ approaches hard selection away from zero; at zero, the weight remains $\alpha_{\max}/2$. Soft gating also retains some negatively aligned teacher contributions, trading strict local exclusion for smoother supervision. It therefore does not inherit Eq.~\ref{eq:m3_local_projection} (Appendix~\ref{app:gate_limits}).
At $\beta=0$, all positions receive the same teacher weight $\alpha_{\max}/2$; increasing sharpness progressively makes allocation depend on compatibility. Both variants thus share the same normalized update, with their distinction confined to the teacher-allocation rule.

\subsection{Coordinating the Signals in Time}

The synchronous variants combine directions evaluated at the same parameters. M3-EG instead uses the gated teacher field $\vm_D(\vtheta)=N^{-1}\sum_n\alpha_n^*\vh_D^n$ to construct a temporary point, then evaluates the normalized reward field $\vh_R(\vtheta)=N^{-1}\sum_n\vh_R^n$ there:
\begin{equation}
\begin{aligned}
\widetilde\vtheta&=\vtheta-\eta_{\mathrm{in}}\vm_D(\vtheta),
&&\text{teacher look-ahead},\\
\vtheta^+&=\vtheta-\eta_{\mathrm{out}}\vh_R(\widetilde\vtheta),
&&\text{reward correction}.
\end{aligned}
\label{eq:m3_eg_update}
\end{equation}
The reward correction is applied from the original parameters, after restoring them; no gradient is propagated through the temporary step. Thus teacher guidance changes the reward evaluation point rather than entering the committed update additively. Appendix~\ref{app:m3_eg} gives the procedure and the conditions on the inner displacement and outer step for local reward descent.

\subsection{Training Rule and Scope}

Each iteration samples student responses, obtains teacher distributions and verifier advantages, and forms the two residual fields. Their compatibility scores determine the gates; normalized, gated contributions then define either the synchronous direction or the extragradient step (Algorithms~\ref{alg:m3_select} and~\ref{alg:m3_eg}). Gates and normalization factors specify update coefficients and are held fixed when applying the direction. Clipped teacher terms and zero-advantage reward terms contribute zero; normalization does not recreate missing supervision.

The local projection bound directly supports the hard selection rule. Aggregate stationarity and conditional variance bounds require a more restrictive population model: smooth lower-bounded reward loss, orthogonal position subspaces, matched teacher and mean reward norms, deterministic conditional teacher directions, and gates fixed before fresh reward noise. These conditions are not enforced by residual normalization. Theorem~\ref{thm:boundary_pareto_convergence} and Proposition~\ref{prop:boundary_variance_reduction} state the resulting guarantees; the experiments assess behavior beyond that model.

\section{Experiments}
\label{sec:experiments}

\begin{figure}[t]
\centering
\includegraphics[width=0.9\linewidth]{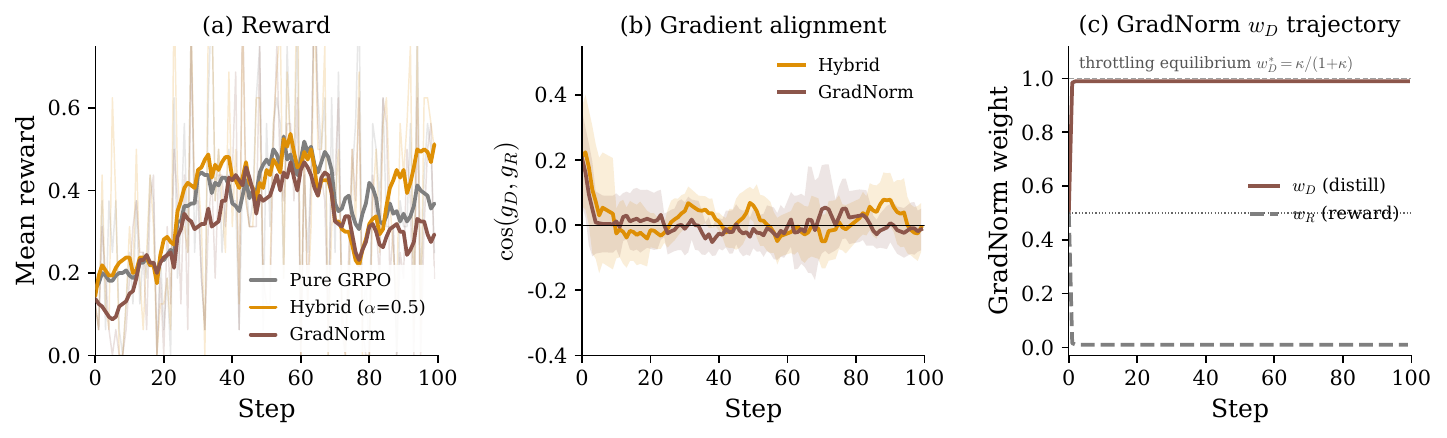}
\caption{GradNorm challenge on InternLM2.5-1.8B GSM8K: training reward (left), gradient cosine (center), and learned distillation weight (right).}
\label{fig:f2_vs_f3}
\end{figure}

\subsection{Setup}
\label{sec:setup}

Our main evaluation covers Qwen3-1.7B, Qwen2.5-1.5B, InternLM2.5-1.8B, and Llama-3.2-1B on GSM8K, SVAMP, and ARC-Challenge, with Qwen3-0.6B and MATH added in supplementary long-horizon runs. We define $\bar{\kappa}$ as the mean of $\kappa_t$ over RL-active steps of the corresponding naive-Hybrid run. Full experimental details are provided in Appendix~\ref{app:experimental_details}.
\subsection{Does Magnitude Drowning Actually Occur?}
\label{sec:drowning}

Figure~\ref{fig:intro_case_study}a shows loss inversion under scalar mixing. With $\bar{\kappa}$ ranging from $1.4\!\times\!10^3$ to $1.3\!\times\!10^4$, the reward update overwhelms distillation; panel~c shows the resulting long-horizon collapse. We next test whether global norm balancing fixes it.

\paragraph{A controlled challenge.}
GradNorm trails both GRPO and naive Hybrid in training reward (Figure~\ref{fig:f2_vs_f3}, left), despite a near-zero aggregate gradient cosine (center). Its distillation weight rapidly approaches $w_D^*=\kappa/(1+\kappa)\approx1$ (right), leaving the effective reward weight at only $1-w_D^*\approx1/\kappa$. GradNorm therefore achieves global norm balance only by nearly turning off RL; because the same weight is applied to every token, it also cannot separate compatible from conflicting teacher updates. This motivates M3's position-dependent gate; further diagnostics are reported in Appendix~\ref{app:mechanistic_validation}.

\subsection{Can Token-Level Gating Beat Global Mixing?}
\label{sec:method_comparison}

\begin{table*}[!t]
\centering
 \caption{Cell reports accuracy (top) and its percentage-point change from GRPO (bottom; $\uparrow/\downarrow$); $^{\dagger}$Hard masking is vacuous for single-token ARC.}
  
\label{tab:cross_arch}
\fontsize{7.5pt}{8.5pt}\selectfont
\setlength{\tabcolsep}{3pt}
\renewcommand{\arraystretch}{1.18}
\setlength{\aboverulesep}{0.2ex}
\setlength{\belowrulesep}{0.3ex}

\resizebox{\textwidth}{!}{%
\begin{tabular}{@{}l|ccc|ccc|ccc|ccc@{}}
\toprule
\rowcolor{gray!20}
 & \multicolumn{3}{c|}{\textbf{Qwen3-1.7B}} &
   \multicolumn{3}{c|}{\textbf{Qwen2.5-1.5B}} &
   \multicolumn{3}{c|}{\textbf{InternLM2.5-1.8B}} &
   \multicolumn{3}{c}{\textbf{Llama-3.2-1B}} \\
\rowcolor{gray!20}
\textbf{Method} &
\textbf{GSM8K} & \textbf{SVAMP} & \textbf{ARC} &
\textbf{GSM8K} & \textbf{SVAMP} & \textbf{ARC} &
\textbf{GSM8K} & \textbf{SVAMP} & \textbf{ARC} &
\textbf{GSM8K} & \textbf{SVAMP} & \textbf{ARC} \\
\midrule
\rowcolor{gray!10}
\multicolumn{13}{@{}l}{\textit{Reference: pure RL}}\\
GRPO
 & \acc{0.696}{\nochg} & \acc{0.940}{\nochg} & \acc{0.732}{\nochg}
 & \acc{0.726}{\nochg} & \acc{0.825}{\nochg} & \acc{0.689}{\nochg}
 & \acc{0.404}{\nochg} & \acc{0.670}{\nochg} & \acc{0.599}{\nochg}
 & \acc{0.552}{\nochg} & \acc{0.730}{\nochg} & \acc{0.533}{\nochg} \\
\midrule
\rowcolor{gray!10}
\multicolumn{13}{@{}l}{\textit{Teacher-augmented baselines}}\\
Hybrid (OPSD)
 & \acc{0.786}{\gain{+9.0}} & \acc{0.932}{\loss{-0.8}} & \acc{\underline{0.796}}{\gain{+6.4}}
 & \acc{0.717}{\loss{-0.9}} & \acc{0.857}{\gain{+3.2}} & \acc{0.693}{\gain{+0.4}}
 & \acc{\underline{0.492}}{\gain{+8.8}} & \acc{0.608}{\loss{-6.2}} & \acc{0.604}{\gain{+0.5}}
 & \acc{\underline{0.563}}{\gain{+1.1}} & \acc{0.722}{\loss{-0.8}} & \acc{0.544}{\gain{+1.1}} \\
OPSD+GradNorm
 & \acc{0.765}{\gain{+6.9}} & \acc{0.917}{\loss{-2.3}} & \acc{0.752}{\gain{+2.0}}
 & \acc{0.735}{\gain{+0.9}} & \acc{0.837}{\gain{+1.2}} & \acc{\underline{0.695}}{\gain{+0.6}}
 & \acc{0.459}{\gain{+5.5}} & \acc{0.687}{\gain{+1.7}} & \acc{0.590}{\loss{-0.9}}
 & \acc{0.547}{\loss{-0.5}} & \acc{0.717}{\loss{-1.3}} & \acc{0.370}{\loss{-16.3}} \\
RLSD
 & \acc{0.8317}{\gain{+13.6}} & \acc{0.923}{\loss{-1.7}} & \acc{0.748}{\gain{+1.6}}
 & \acc{0.725}{\loss{-0.1}} & \acc{0.860}{\gain{+3.5}} & \acc{0.684}{\loss{-0.5}}
 & \acc{0.471}{\gain{+6.7}} & \acc{\underline{0.737}}{\gain{+6.7}} & \acc{0.602}{\gain{+0.3}}
 & \acc{\textbf{0.568}}{\gain{+1.6}} & \acc{0.648}{\loss{-8.2}} & \acc{\underline{0.545}}{\gain{+1.2}} \\
SDPO
 & \acc{0.587}{\loss{-10.9}} & \acc{0.567}{\loss{-37.3}} & \acc{0.742}{\gain{+1.0}}
 & \acc{0.085}{\loss{-64.1}} & \acc{0.612}{\loss{-21.3}} & \acc{0.424}{\loss{-26.5}}
 & \acc{0.316}{\loss{-8.8}} & \acc{0.432}{\loss{-23.8}} & \acc{0.586}{\loss{-1.3}}
 & \acc{0.220}{\loss{-33.2}} & \acc{0.483}{\loss{-24.7}} & \acc{0.432}{\loss{-10.1}} \\
\midrule
\rowcolor{blue!10}
\multicolumn{13}{@{}l}{\textit{\textbf{Ours}: boundary-gated mixing (M3)}}\\
M3-Select$^{\dagger}$
 & \acc{0.803}{\gain{+10.7}} & \acc{0.943}{\gain{+0.3}} & \acc{0.275}{\loss{-45.7}}
 & \acc{0.459}{\loss{-26.7}} & \acc{0.585}{\loss{-24.0}} & \acc{0.571}{\loss{-11.8}}
 & \acc{0.394}{\loss{-1.0}} & \acc{0.538}{\loss{-13.2}} & \acc{0.599}{\nochg}
 & \acc{0.516}{\loss{-3.6}} & \acc{0.673}{\loss{-5.7}} & \acc{0.459}{\loss{-7.4}} \\
\textbf{M3-Soft}
 & \acc{\underline{0.8324}}{\gain{+13.6}} & \acc{\underline{0.945}}{\gain{+0.5}} & \acc{\textbf{0.807}}{\gain{+7.5}}
 & \acc{\textbf{0.752}}{\gain{+2.6}} & \acc{\textbf{0.905}}{\gain{+8.0}} & \acc{\textbf{0.736}}{\gain{+4.7}}
 & \acc{\textbf{0.501}}{\gain{+9.7}} & \acc{\textbf{0.777}}{\gain{+10.7}} & \acc{\underline{0.608}}{\gain{+0.9}}
 & \acc{\textbf{0.568}}{\gain{+1.6}} & \acc{\underline{0.750}}{\gain{+2.0}} & \acc{\textbf{0.556}}{\gain{+2.3}} \\
M3-EG
 & \acc{\textbf{0.8446}}{\gain{+14.9}} & \acc{\textbf{0.952}}{\gain{+1.2}} & \acc{0.731}{\loss{-0.1}}
 & \acc{\underline{0.749}}{\gain{+2.3}} & \acc{\underline{0.868}}{\gain{+4.3}} & \acc{0.692}{\gain{+0.3}}
 & \acc{0.440}{\gain{+3.6}} & \acc{0.468}{\loss{-20.2}} & \acc{\textbf{0.6135}}{\gain{+1.5}}
 & \acc{0.552}{\nochg} & \acc{\textbf{0.758}}{\gain{+2.8}} & \acc{0.534}{\gain{+0.1}} \\
\midrule
\multicolumn{13}{@{}l}{\textbf{M3-Soft} $\geq$ best non-M3 baseline in \textbf{12/12} cells ($11$ strict wins, $1$ exact tie).}\\
\bottomrule
\end{tabular}}

\end{table*}

\begin{table}[t]
\centering
\caption{Long-horizon reward on GSM8K after $500$ steps (last-50 mean). $\downarrow$ denotes collapse.}
\label{tab:500step_gsm8k}

\fontsize{8.2pt}{9.6pt}\selectfont
\setlength{\tabcolsep}{7.5pt}
\renewcommand{\arraystretch}{1.16}
\setlength{\aboverulesep}{0.2ex}
\setlength{\belowrulesep}{0.3ex}

\begin{tabular}{
l|ccc|
>{\columncolor{blue!4}}c
>{\columncolor{blue!4}}c
}
\toprule

\rowcolor{gray!20}
&
\multicolumn{3}{c|}{
\textit{Reference and adaptive baselines}
}
&
\multicolumn{2}{
>{\columncolor{blue!10}}c
}{
\textit{\textbf{Ours}: boundary-gated mixing}
} \\

\rowcolor{gray!10}
\textbf{Architecture}
& \textbf{GRPO}
& \textbf{Hybrid}
& \textbf{GradNorm}
& \cellcolor{blue!10}\textbf{M3-Select}
& \cellcolor{blue!10}\textbf{M3-Soft} \\
\midrule

Qwen3-1.7B
& $\downarrow\!0.003$
& $\downarrow\!0.000$
& $\downarrow\!0.000$
& $0.809$
& $\mathbf{0.866}$ \\

Qwen3-0.6B
& $\downarrow\!0.003$
& $\downarrow\!0.000$
& $\downarrow\!0.053$
& $0.644$
& $\mathbf{0.673}$ \\

Llama-1B
& $\downarrow\!0.004$
& $\downarrow\!0.003$
& $\downarrow\!0.328$
& $\mathbf{0.497}$
& $0.485$ \\

\bottomrule
\end{tabular}
\end{table}

\begin{figure*}[t]
\centering
\begin{minipage}[t]{0.49\textwidth}
\centering
\includegraphics[width=\linewidth]{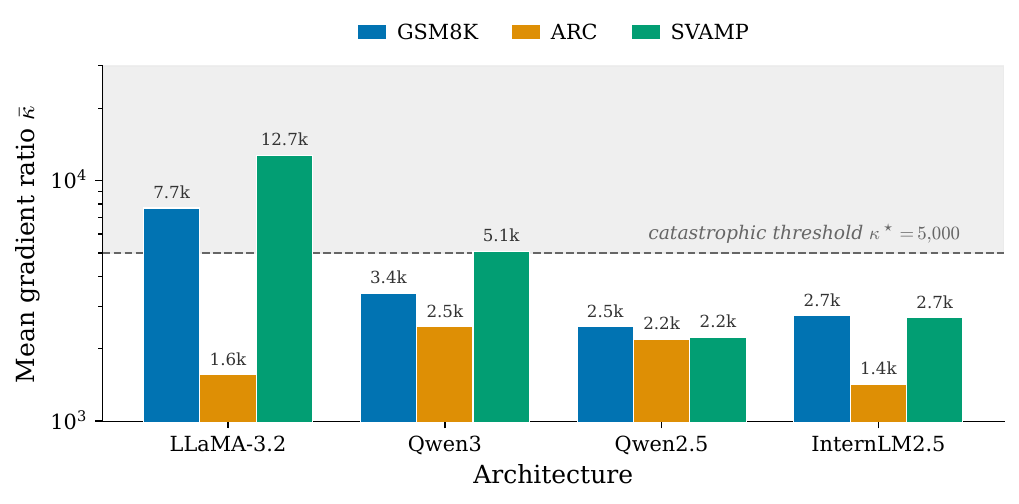}
\caption{Mean gradient-magnitude ratio $\bar{\kappa}$ for $12$ architecture--dataset pairs under naive Hybrid; dashed line: empirical threshold $\kappa^*=5{,}000$.}
\label{fig:kappa_bar}
\end{minipage}\hfill
\begin{minipage}[t]{0.49\textwidth}
\centering
\includegraphics[width=\linewidth]{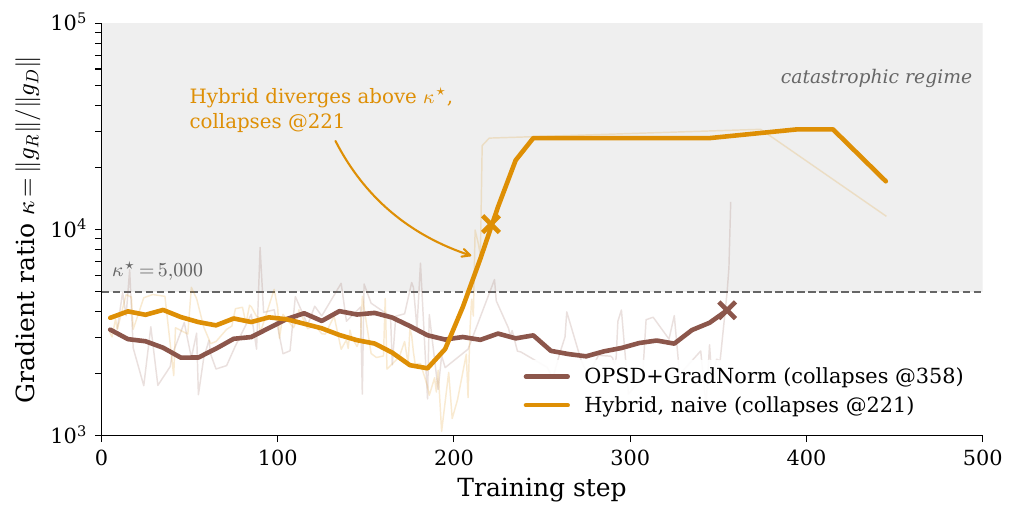}
\caption{Runtime magnitude ratio. Hybrid and GradNorm cross $\kappa^*=5{,}000$ before collapsing at steps $221$
and $358$.}
\label{fig:kappa_dynamics}
\end{minipage}
\end{figure*}

\begin{table}[htbp]
\centering
\caption{One-step sampled-token log-probability change by pre-update conflict region (30 trajectories, 8{,}492 tokens). Negative values indicate suppression by the update.}
\label{tab:one_step}
\small
\begin{tabular}{lccc}
\toprule
\rowcolor{gray!10}
Region & OPSD & GRPO & Hybrid \\
\midrule
Compatible ($\dirvar{K_{DR}}>\epsilon$, $n$=851) & $+1.6\times10^{-3}$ & $+2.7\times10^{-3}$ & $+2.0\times10^{-3}$ \\
Neutral ($|\dirvar{K_{DR}}|\leq\epsilon$, $n$=5191) & $\approx 0$ & $\approx 0$ & $\approx 0$ \\
Conflicting ($\dirvar{K_{DR}}<-\epsilon$, $n$=2450) & $+0.3\times10^{-4}$ & $+7.3\times10^{-4}$ & $-2.4\times10^{-4}$ \\
\bottomrule
\end{tabular}
\end{table}

\paragraph{Accuracy across architectures.}
Across four architectures and three datasets (Table~\ref{tab:cross_arch}), M3-Soft matches or exceeds the strongest non-M3 baseline in all $12$ cells under the per-cell best-observed protocol ($11$ strict wins, one tie). The largest gains occur on Qwen2.5-SVAMP ($0.905$, $+4.5$pp over RLSD), Qwen2.5-ARC ($0.736$, $+4.1$pp over OPSD+GradNorm), and InternLM-SVAMP ($0.777$, $+4.0$pp over RLSD). On the two Llama arithmetic cells, M3-Soft ties RLSD on GSM8K ($0.568$) and improves over GRPO on SVAMP ($0.750$, $+2.0$pp). The Qwen3 margins are tighter: $0.8324$ on GSM8K ($+0.07$pp over RLSD) and $0.945$ on SVAMP ($+0.5$pp over GRPO).

\paragraph{Long-horizon stability.}
\label{sec:500step}
Table~\ref{tab:500step_gsm8k} compares last-50 training reward after $500$ steps on GSM8K. Across all three architectures, every reference or globally balanced baseline triggers the collapse criterion, whereas M3-Select and M3-Soft remain non-collapsed. M3-Soft attains the highest final reward on both Qwen3 scales, while M3-Select is marginally higher on Llama-3.2-1B. Detailed phase-wise trajectories and collapse times for Qwen3-1.7B are reported in Table~\ref{tab:500step} of Appendix~\ref{app:long_horizon}.


\paragraph{Regime Dependence of Token-Level Gating.}
\label{sec:cross_arch}
Across the $12$ architecture--dataset cells, $\bar{\kappa}$ ranges from $1{,}423$ to $12{,}705$ (Figure~\ref{fig:kappa_bar}). Figure~\ref{fig:kappa_dynamics} provides the within-run counterpart: on Qwen3-1.7B GSM8K, Hybrid crosses the empirical threshold $\kappa^*=5{,}000$ near step $200$ before collapsing at step $221$, while GradNorm crosses later and collapses at step $358$. In the two high-$\kappa$ Llama arithmetic cells ($7{,}690/12{,}705$), Hybrid's reference-rate training reward falls to $0.031/0.000$ on GSM8K/SVAMP, whereas M3-Soft reaches $0.647/0.822$ (Appendix~\ref{app:cross_dataset}). At lower $\kappa$, hard masking can be overly restrictive: on Qwen3-ARC ($\bar{\kappa} =2{,}465$), M3-Select reaches $0.275$ accuracy, compared with $0.807$ for M3-Soft (Table~\ref{tab:cross_arch}). Thus, $\bar{\kappa}$ indicates the severity of magnitude imbalance and the appropriate gate strength, rather than directly predicting accuracy.
Together, these results support $\kappa$ as a regime indicator for instability and appropriate gate strength, rather than a predictor of absolute accuracy; exact per-cell $\bar{\kappa}$ values and reference-rate rewards are in Appendix~\ref{app:cross_dataset}.

\paragraph{Local Consequence of Token-Level Conflict.}
We test the sign-specific prediction on $30$ correct Qwen3-1.7B GSM8K trajectories comprising $8{,}492$ tokens. Tokens are partitioned by the pre-update cross-signal NTK, after which we apply one OPSD, GRPO, or hybrid update. On compatible positions, the hybrid update increases the sampled-token log-probability by $2.0\times10^{-3}$; on conflicting positions, it decreases it by $2.4\times10^{-4}$ (Table~\ref{tab:one_step}). Over the same conflicting subset, the pure OPSD and GRPO updates yield positive changes. Because the partition is computed before the update, the result tests the local sign prediction rather than defining conflict from the observed logit change.

\section{Conclusion and Future Work}
\label{sec:conclusion}

We study when dense teacher supervision can complement sparse verifiable rewards in reasoning-model post-training. Our NTK analysis separates their interaction into token-level compatibility, captured by the cross-signal NTK $K_{DR}(n)$, and scale imbalance, captured by $\kappa$, exposing localized directional conflict and magnitude drowning. This diagnosis leads to the M3 family, which combines magnitude normalization with hard, soft, or temporally decoupled compatibility gating. Across four model families and three reasoning benchmarks, M3-Soft matches or exceeds the strongest non-M3 baseline, while M3 variants remain stable over $500$-step GSM8K training where non-gated baselines collapse.
In the future, we plan to scale this application to support larger and more complex agentic scenarios such as coding and deep research.

\subsection*{AI use statement}

In this work, large language models (LLMs) were used for language polishing, figure design assistance, coding support, and mathematical proof assistance. Specifically, LLMs were used to improve the clarity, grammar, and readability of the manuscript, refine its stylistic quality, and suggest alternative phrasings to reduce redundancy. They also provided suggestions for figure design, visualization layouts, and graphical presentation; all final figures were created, verified, and curated by the authors using the authors' experimental data. In addition, LLMs assisted with code development, intermediate mathematical derivations, proof construction, and consistency checking. All assumptions, formal statements, derivations, proofs, code, figures, and other LLM-assisted content were independently reviewed, verified, and revised by the authors before inclusion. The research questions, core methodology, scientific contributions, experimental design, critical analyses, and final decisions were independently determined by the authors.

\subsection*{Ethics statement}

All experiments in this work were carried out using publicly available language models and standard reasoning benchmarks, including GSM8K, MATH, SVAMP, and ARC-Challenge, in accordance with their respective licenses and terms of use. The study does not involve human or animal subjects, and we did not collect, use, or disclose any personally identifiable information or private user data.

\subsection*{Reproducibility statement}
All experiments use publicly available base models (Qwen3-0.6B, Qwen3-1.7B, Qwen2.5-1.5B, InternLM2.5-1.8B, and Llama-3.2-1B) and standard benchmarks (GSM8K, MATH, SVAMP, and ARC-Challenge). Full hyperparameters, training protocols, held-out split construction, and the rerun/best-of-candidates selection procedure are specified in Appendix~\ref{app:experimental_details} and Section~\ref{sec:setup}. The cross-signal NTK diagnostic $K_{DR}(n)$ and the M3 gating rules are specified in closed form in Sections~\ref{sec:cross_signal_ntk}--\ref{sec:methods} and Algorithm~\ref{alg:m3_select}; no proprietary data or infrastructure is required to reproduce the main results.

\bibliography{references}
\bibliographystyle{iclr2027_conference}

\appendix
\renewcommand{\thesection}{\AlphAlph{\value{section}}}

\clearpage

\section{Experimental Details}
\label{app:experimental_details}

\paragraph{Rollout and optimization protocol.}
We use the on-policy OPSD construction of Section~\ref{sec:problem_setup}. With LoRA disabled, the frozen teacher receives a privileged prompt containing the ground-truth solution, whereas the student receives only the problem and its sampled prefix. Unless otherwise stated, experiments use the defaults in Table~\ref{tab:experimental_hyperparameters}.

\paragraph{Evaluation metrics and regime statistics.}
Training reward is the last-20-step mean unless a table states otherwise, and  accuracy is measured by greedy decoding on GSM8K ($n=1319$), SVAMP ($n=600$), and ARC ($n=1172$). The SVAMP split excludes all prompts used by the 100- and 200-step training runs. For each architecture--dataset pair, $\bar\kappa$ is the mean of $\kappa_t=\lVert\vg_R(t)\rVert/\lVert\vg_D(t)\rVert$ over RL-active steps of the corresponding naive-Hybrid run and is used as a regime indicator. A run is marked as collapsed when reward remains below $0.05$ for ten consecutive steps.

\paragraph{Candidate selection and post-processing.}
For each method--cell pair, the evaluated candidates include the reference run and, where available, seed replicates, checkpoints every $20$ steps for horizons up to $T\in\{100,200,400\}$, and conservative-rate reruns at $1$--$2\times10^{-5}$. M3-Soft additionally includes the evaluated $(\beta,\alpha_{\max},T)$ sweep. SDPO uses its selected conservative-rate candidate; on InternLM-ARC, this is the $20$-step pre-collapse checkpoint. Conservative-rate candidates are made available to all methods on Llama GSM8K and SVAMP. Collapsed runs are retained and reported at their measured reward and  accuracy rather than excluded. SWA is applied to the Qwen3- and InternLM-GSM8K chains and to the Llama-GSM8K conservative-rate seed-42 chain using last-two-checkpoint averaging. Because candidates and checkpoints are chosen using  accuracy and no independent validation split is defined, these entries are reported as \emph{best observed} rather than validation-selected results. Seed-level robustness and SWA values are reported in Section~\ref{sec:multiseed}.

\begin{table}[t]
\centering
\caption{Default experimental hyperparameters.}
\label{tab:experimental_hyperparameters}
\small
\setlength{\tabcolsep}{7pt}
\renewcommand{\arraystretch}{1.08}
\begin{tabular}{lll}
\toprule
\rowcolor{gray!10}
\textbf{Component} & \textbf{Hyperparameter} & \textbf{Value} \\
\midrule
LoRA & Rank $r$ & $64$ \\
LoRA & Scaling $\alpha_{\mathrm{LoRA}}$ & $128$ \\
Distillation & Divergence & Generalized JSD \\
Distillation & JSD coefficient $\lambda$ & $\tfrac12$ \\
Distillation & Token clipping $\tau$ & $0.05$ \\
Reinforcement learning & Rollouts per group $G$ & $8$ \\
Hybrid baseline & Default mixing coefficient $\alpha$ & $0.5$ \\
Optimization & Reference learning rate & $5\times10^{-5}$ \\
\bottomrule
\end{tabular}
\end{table}

\section{Extended Experiments}
\label{app:experiments_extended}

\subsection{Detailed Cross-Architecture Results}
\label{app:full_comparison}

\paragraph{Accuracy.}
Table~\ref{tab:cross_arch_testacc} directly compares M3-Soft with the strongest non-M3 baseline in each architecture--dataset cell. Under the per-cell best-observed protocol, M3-Soft matches or exceeds the strongest baseline in all $12$ cells ($11$ strict wins, one tie). The largest gains occur on Qwen2.5-SVAMP ($0.905$, $+4.5$pp over RLSD), Qwen2.5-ARC ($0.736$, $+4.1$pp over OPSD+GradNorm), and InternLM-SVAMP ($0.777$, $+4.0$pp over RLSD); Llama-GSM8K is the exact tie with RLSD at $0.568$.

\paragraph{Seed robustness and SWA.}
\label{sec:multiseed}
Seed replicates qualify the best-observed results in Table~\ref{tab:cross_arch_testacc}. \textbf{Wide-margin cells.} The ranking is stable on Qwen2.5-SVAMP, where the M3-Soft seed mean remains $+2.4$pp above RLSD, and on Llama-SVAMP, where both replicates exceed GRPO by $+7.7$ and $+10.2$pp. Qwen2.5-GSM8K also wins for both available replicates, while two of three Qwen2.5-ARC seeds exceed the strongest baseline and the third ties it. \textbf{Tight-margin cells.} The Qwen3-GSM8K replicates remain within $1.2$pp below RLSD before SWA, and the InternLM-GSM8K win is likewise obtained only after averaging. SWA raises the selected Qwen3-, InternLM-, and Llama-GSM8K chains by $0.013$, $0.015$, and $0.004$, respectively, yielding accuracies of $0.8324$, $0.501$, and $0.568$; it is not uniformly beneficial, decreasing the Llama seed-43 chain from $0.562$ to $0.552$. Llama-GSM8K is seed-fragile at the reference learning rate, where two of three replicates collapse, but both conservative-rate replicates survive ($0.568/0.552$). We therefore treat the tight cells as best-observed parity rather than seed-robust separation.
\begin{table}[t]
\centering
\caption{Accuracy of M3-Soft versus the strongest non-M3 baseline.}
\label{tab:cross_arch_testacc}
\footnotesize
\setlength{\tabcolsep}{5pt}
\renewcommand{\arraystretch}{1.12}

\begin{tabular}{l|ccc|ccc|c}
\toprule
\rowcolor{gray!10}
& \multicolumn{3}{c|}{\textbf{Best non-M3 baseline}}
& \multicolumn{3}{c|}{\textbf{M3-Soft}}
& \textbf{M3-Soft} \\
\rowcolor{gray!10}
\textbf{Architecture}
& \textbf{GSM8K} & \textbf{SVAMP} & \textbf{ARC}
& \textbf{GSM8K}
& \textbf{SVAMP}
& \textbf{ARC}
& \textbf{$\geq$ Baseline} \\
\midrule
Qwen3-1.7B
& 0.8317 & 0.940 & 0.796
& \textbf{0.8324}
& \textbf{0.945}
& \textbf{0.807}
& 3/3 \\

Qwen2.5-1.5B
& 0.735 & 0.860 & 0.695
& \textbf{0.752}
& \textbf{0.905}
& \textbf{0.736}
& 3/3 \\

InternLM2.5-1.8B
& 0.492 & 0.737 & 0.604
& \textbf{0.501}
& \textbf{0.777}
& \textbf{0.608}
& 3/3 \\

Llama-3.2-1B
& \textbf{0.568} & 0.730 & 0.545
& \textbf{0.568}
& \textbf{0.750}
& \textbf{0.556}
& 3/3 \\
\midrule
\rowcolor{blue!5}
\multicolumn{7}{r|}{\textbf{Total: 11 strict wins and 1 exact tie}}
& \textbf{12/12} \\
\bottomrule
\end{tabular}
\end{table}

\paragraph{Training reward and magnitude ratio.}
\label{app:cross_dataset}
Table~\ref{tab:cross_dataset} reports per-cell training rewards and the exact $\bar\kappa$ measured on the corresponding naive-Hybrid runs. The largest reference-rate separations occur on Llama-GSM8K ($\bar\kappa=7{,}690$; M3-Soft $0.647$, $20.9\times$ over Hybrid $0.031$) and Llama-SVAMP ($\bar\kappa=12{,}705$; M3-Soft $0.822$, $11.4\times$ over GRPO $0.072$, while Hybrid reaches $0.000$). At the low-$\kappa$ end, Qwen3-ARC ($\bar\kappa=2{,}465$) favors Hybrid in training reward ($0.919$ vs.\ $0.775$ for the extended-training M3-Soft candidate). For Llama-GSM8K, this table reports the reference-rate M3-Soft run ($0.647$), whereas Table~\ref{tab:cross_arch} uses the conservative-rate SWA chain selected by  accuracy (reward $0.591$, accuracy $0.568$). Because the M3-Soft column includes tuned and, where marked, extended-training candidates, this table is an optimization diagnostic rather than a matched-budget  comparison.

\begin{table}[htbp]
\centering
\caption{Per-cell training reward (last-20 mean) and mean gradient-magnitude ratio $\bar{\kappa}$. M3-Soft reports the selected evaluated configuration; $^{*}$ denotes extended  training, and $\bar{\kappa}$ is measured on the corresponding naive-Hybrid run.}
\label{tab:cross_dataset}
\footnotesize
\setlength{\tabcolsep}{4pt}
\begin{tabular}{llccccc|c}
\toprule
\rowcolor{gray!10}
Architecture & Dataset & GRPO & OPSD & OPSD+GradNorm & M3-Select & M3-Soft & $\bar{\magvar{\kappa}}$ \\
\midrule
\multirow{3}{*}{Qwen3-1.7B}
  & GSM8K & $0.838$ & $0.825$ & $0.731$ & $0.844$ & $\mathbf{0.881}$ & $3{,}404$ \\
  & SVAMP & $0.922$ & $0.894$ & $0.894$ & $0.947$ & $\mathbf{0.969}$ & $5{,}080$ \\
  & ARC   & $0.741$ & $\mathbf{0.919}$ & $0.769$ & $0.263$ & $0.775^{*}$ & $2{,}465$ \\
\midrule
\multirow{3}{*}{Qwen2.5-1.5B}
  & GSM8K & $0.806$ & $0.747$ & $0.653$ & $0.153$ & $\mathbf{0.828}^{*}$ & $2{,}467$ \\
  & SVAMP & $0.769$ & $0.766$ & $0.681$ & $0.331$ & $\mathbf{0.778}^{*}$ & $2{,}234$ \\
  & ARC   & $0.659$ & $0.650$ & $0.644$ & $0.319$ & $\mathbf{0.700}$ & $2{,}184$ \\
\midrule
\multirow{3}{*}{InternLM2.5-1.8B}
  & GSM8K & $0.369$ & $0.428$ & $0.306$ & $0.100$ & $\mathbf{0.472}^{*}$ & $2{,}745$ \\
  & SVAMP & $0.594$ & $0.563$ & $0.594$ & $0.366$ & $\mathbf{0.688}^{*}$ & $2{,}690$ \\
  & ARC   & $0.616$ & $0.634$ & $0.666$ & $0.416$ & $\mathbf{0.682}^{*}$ & $1{,}423$ \\
\midrule
\multirow{3}{*}{Llama-3.2-1B}
  & GSM8K & $0.000$ & $0.031$ & $0.006$ & $0.597$ & $\mathbf{0.647}$ & $7{,}690$ \\
  & SVAMP & $0.072$ & $0.000$ & $0.038$ & $0.488$ & $\mathbf{0.822}$ & $12{,}705$ \\
  & ARC   & $0.522$ & $0.469$ & $0.353$ & $0.406$ & $\mathbf{0.650}^{*}$ & $1{,}562$ \\
\bottomrule
\end{tabular}
\end{table}

\subsection{Long-Horizon Stability and Runtime Diagnostics}
\label{app:robustness}

\paragraph{Cross-dataset long-horizon results.}
Table~\ref{tab:500step_datasets} extends the $500$-step evaluation to MATH, SVAMP, and ARC-Challenge using Qwen3-1.7B. Under the reference configurations, prolonged training can still trigger collapse beyond GSM8K: GRPO and OPSD+GradNorm collapse on MATH, while Hybrid collapses on SVAMP. In contrast, both M3 variants remain non-collapsed across all three datasets.

\begin{table}[t]
\centering
\caption{Long-horizon reward across datasets on Qwen3-1.7B after $500$ steps (last-50 mean). M3-Soft uses the best evaluated gate configuration per dataset; other methods use the
reference configuration. Bold marks the column maximum, and $\downarrow$ denotes collapse.}
\label{tab:500step_datasets}

\small
\setlength{\tabcolsep}{9pt}
\renewcommand{\arraystretch}{1.14}
\setlength{\aboverulesep}{0.2ex}
\setlength{\belowrulesep}{0.3ex}

\begin{tabular}{l|ccc}
\toprule
\rowcolor{gray!20}
\textbf{Method}
& \textbf{MATH}
& \textbf{SVAMP}
& \textbf{ARC-Challenge} \\
\midrule

\rowcolor{gray!10}
\multicolumn{4}{l}{\textit{Reference and adaptive baselines}} \\
GRPO
& $\downarrow\!0.018$
& $0.900$
& $0.773$ \\
Hybrid (OPSD)
& $0.448$
& $\downarrow\!0.170$
& $0.765$ \\
OPSD+GradNorm
& $\downarrow\!0.015$
& $\mathbf{0.973}$
& $\mathbf{0.790}$ \\
\midrule

\rowcolor{blue!10}
\multicolumn{4}{l}{\textit{\textbf{Ours}: boundary-gated mixing (M3)}} \\
M3-Select
& $0.367$
& $0.943$
& $0.282$ \\
M3-Soft
& $\mathbf{0.523}$
& $0.950$
& $0.667$ \\
\bottomrule
\end{tabular}
\end{table}

\paragraph{Phase-wise collapse dynamics.}
\label{app:long_horizon}
Table~\ref{tab:500step} resolves the Qwen3-1.7B GSM8K runs into 100-step phases. Collapse denotes reward below $0.05$ for ten consecutive steps. Unless noted otherwise, the runs use the reference learning rate $5\!\times\!10^{-5}$; all entries are training rewards rather than held-out accuracies.SDPO collapses first at step $50$, followed by Hybrid at $221$, OPSD+GradNorm at $358$, RLSD at $362$, and GRPO at $412$. In contrast, M3-Select does not trigger the collapse criterion within $500$ steps. Phase-wise M3-Soft results are unavailable, but its last-50 reward is reported in Table~\ref{tab:500step_gsm8k}.
\begin{table}[t]
\centering
\caption{Phase-wise training reward over $500$ steps on Qwen3-1.7B GSM8K. ``Full'' is the all-step mean, ``Trend'' gives the first collapse step, and $^\dagger$ denotes batch size $1$.}
\label{tab:500step}
\footnotesize
\setlength{\tabcolsep}{3pt}
\resizebox{\linewidth}{!}{%
\begin{tabular}{lccccccc}
\toprule
\rowcolor{gray!10}
Method & 1--100 & 101--200 & 201--300 & 301--400 & 401--500 & Full & Trend \\
\midrule
Pure GRPO & $0.840$ & $0.801$ & $0.809$ & $0.802$ & $0.075$ & $0.665$ & collapse @412 \\
Hybrid (OPSD, $\alpha\!=\!0.5$) & $0.859$ & $0.762$ & $0.114$ & $0.006$ & $0.001$ & $0.348$ & collapse @221 \\
OPSD+GradNorm$^\dagger$ & $0.877$ & $0.823$ & $0.868$ & $0.464$ & $0.000$ & $0.606$ & collapse @358 \\
RLSD~\citep{yang2026rlsd} & $0.855$ & $0.874$ & $0.871$ & $0.259$ & $0.005$ & $0.573$ & collapse @362 \\
SDPO~\citep{hubotter2026sdpo} & $0.306$ & $0.000$ & $0.000$ & $0.000$ & $0.000$ & $0.061$ & collapse @50 \\
M3-Select ($\alpha_{\max}\!=\!0.15$) & $0.846$ & $0.818$ & $0.853$ & $0.846$ & $0.810$ & $0.835$ & stable \\
\bottomrule
\end{tabular}}
\end{table}


\paragraph{Post-collapse conflict diagnostic.}
After Hybrid collapses, its measured conflict rate falls to zero because the RL gradient vanishes, not because the two signals become compatible. M3-Select instead maintains an active conflict rate near $\dirvar{C_{\mathrm{NTK}}}=0.3$ while preserving reward (Figure~\ref{fig:conflict_rate}).

\begin{figure}[t]
\centering
\includegraphics[width=0.78\linewidth]{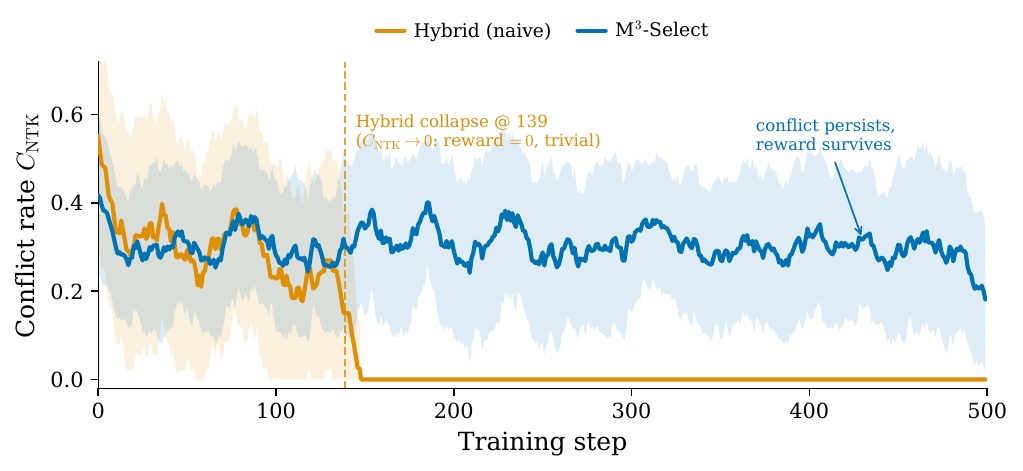}
\caption{Runtime token-level conflict on Qwen3-0.6B GSM8K. Hybrid's post-collapse drop reflects a vanishing RL
gradient; M3-Select remains active near $\dirvar{C_{\mathrm{NTK}}}=0.3$.}
\label{fig:conflict_rate}
\end{figure}

\subsection{Gate Sensitivity Across Regimes}
\label{app:gate_sensitivity}



\paragraph{Controlled sharpness sweep.}
\label{sec:soft_sweep}
We first isolate gate sharpness on Qwen3-1.7B GSM8K at $100$ steps. With $\alpha_{\max}=0.5$, the gentle $\beta=0.5$ gate attains the largest last-20 reward ($0.881$), compared with $0.856$ for $\beta=1$, $0.850$ for GRPO, and $0.844$ for hard M3-Select (Table~\ref{tab:soft_sweep}). The soft-gate configurations have similar values of the count-based budget proxy $\alpha_{\mathrm{eff}}$ ($0.439$--$0.446$), consistent with gate shape, rather than a large change in this proxy---driving the observed differences.

\begin{table}[htbp]
\centering
\caption{M3-Soft sharpness sweep on Qwen3-1.7B GSM8K at $100$ steps ($\alpha_{\max}=0.5$). $\alpha_{\mathrm{eff}}$ denotes the mean teacher weight. For hard selection, the measured conflict rate is $C_{\mathrm{NTK}}\approx0.40$, so approximately $60\%$ of tokens retain weight $\alpha_{\max}$, giving $\alpha_{\mathrm{eff}}\approx0.5\times0.60=0.30$.}
\label{tab:soft_sweep}
\begin{tabular}{lcccl}
\toprule
\rowcolor{gray!10}
Method & $\beta$ & Last-10 & Last-20 & $\alpha_{\mathrm{eff}}$ \\
\midrule
GRPO & --- & $0.858$ & $0.850$ & --- \\
OPSD (uniform) & --- & $0.833$ & $0.825$ & $0.500$ \\
\midrule
M3-Select (hard) & $\infty$ & $0.855$ & $0.844$ & $\sim\!0.30$ \\
M3-Soft & $10$ & $0.875$ & $0.847$ & $0.446$ \\
M3-Soft & $5$ & $0.813$ & $0.838$ & $0.439$ \\
M3-Soft & $1$ & $0.906$ & $0.856$ & $0.444$ \\
\textbf{M3-Soft} & $\mathbf{0.5}$ & $\mathbf{0.913}$ & $\mathbf{0.881}$ & $0.446$ \\
\bottomrule
\end{tabular}
\end{table}

\paragraph{Cross-architecture reward sensitivity.}
\label{app:beta_sensitivity_reward}
Table~\ref{tab:beta_sensitivity} fixes the horizon at $100$ steps and compares $\beta\in\{0.5,1,5\}$ across all $12$ architecture--dataset cells. A gentle gate ($\beta\in\{0.5,1\}$) is the best evaluated M3-Soft setting in $11/12$ cells. Both high-$\kappa$ Llama arithmetic cells favor $\beta=1$, whereas InternLM-SVAMP is the sole $\beta=5$ exception. This table diagnoses sensitivity within M3-Soft; it is not the source of the best-observed  headline in Table~\ref{tab:cross_arch}.

\begin{table}[htbp]
  \centering
\caption{M3-Soft sharpness sensitivity at $100$ steps (last-20 mean reward). $\alpha_{\max}=0.05$ by default and $0.025$ on Qwen3-SVAMP.}
  \label{tab:beta_sensitivity}
  \footnotesize
  \setlength{\tabcolsep}{5pt}
  \renewcommand{\arraystretch}{1.12}
  \setlength{\aboverulesep}{0.2ex}
  \setlength{\belowrulesep}{0.3ex}

  \begin{tabular}{@{}llcccc@{}}
  \toprule
  \rowcolor{gray!10}
  \textbf{Model} &
  \textbf{Dataset} &
  $\boldsymbol{\beta=0.5}$ &
  $\boldsymbol{\beta=1}$ &
  $\boldsymbol{\beta=5}$ &
  \textbf{GRPO} \\
  \midrule

  \multirow{3}{*}{Qwen3-1.7B}
  & GSM8K & $0.834$ & $\mathbf{0.856}$ & $0.850$ & $0.838$ \\
  & SVAMP & $\mathbf{0.991}$ & $0.969$ & $0.981$ & $0.953$ \\
  & ARC   & $\mathbf{0.516}$ & $0.350$ & $0.356$ & $0.741$ \\

  \midrule
  \multirow{3}{*}{Llama-3.2-1B}
  & GSM8K & $0.569$ & $\mathbf{0.647}$ & $0.353$ & $\downarrow\!0.000$ \\
  & SVAMP & $0.456$ & $\mathbf{0.822}$ & $0.609$ & $0.072$ \\
  & ARC   & $\mathbf{0.559}$ & $0.500$ & $0.388$ & $0.522$ \\

  \midrule
  \multirow{3}{*}{Qwen2.5-1.5B}
  & GSM8K & $0.697$ & $\mathbf{0.728}$ & $0.169$ & $0.806$ \\
  & SVAMP & $\mathbf{0.794}$ & $0.691$ & $0.463$ & $0.769$ \\
  & ARC   & $\mathbf{0.700}$ & $0.475$ & $0.369$ & $0.659$ \\

  \midrule
  \multirow{3}{*}{InternLM2.5-1.8B}
  & GSM8K & $0.234$ & $\mathbf{0.269}$ & $0.138$ & $0.369$ \\
  & SVAMP & $0.594$ & $0.366$ & $\mathbf{0.744}$ & $0.594$ \\
  & ARC   & $0.469$ & $\mathbf{0.575}$ & $0.472$ & $0.616$ \\

  \bottomrule
  \end{tabular}
  \end{table}

\paragraph{Sensitivity.}
\label{app:beta_sensitivity}
Training reward does not by itself select the best gate for  accuracy. On Qwen2.5-SVAMP, the very-soft recipe $(\beta,\alpha_{\max})=(0.1,0.001)$ reaches $0.884\!\pm\!0.026$ across three seeds ($0.855/0.892/0.905$), exceeding RLSD's $0.860$ by $2.4$pp in the seed mean. On InternLM-SVAMP, the same recipe gives $0.710/0.777/0.722$: the mean ($0.736$) is at parity with RLSD ($0.737$), while the best observed seed reaches $0.777$. Thus, the Qwen2.5 gain is seed-robust, whereas the InternLM headline is best-observed rather than a mean separation.

\subsection{Mechanistic Validation}
\label{app:mechanistic_validation}

The following diagnostics test the mechanism at progressively coarser levels. We first verify the predicted one-step effect after partitioning tokens by their pre-update cross-signal NTK, then compare response-level and token-level notions of conflict, and finally summarize the aggregate behavior across settings.

\paragraph{Response-Level Semantic Conflict.}
\label{app:response_conflict}
\begin{table}[htbp]
\centering
\caption{Response-level semantic conflict (Qwen3-0.6B, 208 trajectories): on-policy = teacher evaluates the student's own trajectory, off-policy = ground-truth contexts; true opposition = wrong trajectories with opposing RL and distillation gradients.}
\label{tab:response_conflict}
\begin{tabular}{lccccc}
\toprule
\rowcolor{gray!10}
Mode & Conflict \% & $\overline{\cos}(\vg_R,\vg_D)$ & True opposition & $\magvar{\kappa}_{\mathrm{median}}$ & $\bar{\cL}_D$ \\
\midrule
On-policy OPSD & $75\%$ & $+0.08$ & $25/157$ ($16\%$) & $388$ & $0.098$ \\
Off-policy & $75\%$ & $-0.01$ & $90/157$ ($57\%$) & $68$ & $0.810$ \\
\bottomrule
\end{tabular}
\end{table}

Response-level labels and token-level interactions answer different questions. For each trajectory in a GRPO group, we measure its reward, normalized advantage, trajectory-level gradient cosine $\cos(\vg_R,\vg_D)$, and magnitude ratio $\magvar{\kappa}$. Across $20$ analysis steps with Qwen3-0.6B (batch size $2$, group size $8$), groups with zero reward variance are excluded because GRPO assigns zero advantage and hence no reward gradient, leaving $208$ trajectories. Both modes label $75\%$ of trajectories as semantically conflicting, but they differ sharply in gradient behavior. Among the $157$ incorrect trajectories, true opposition occurs in $25$ cases ($16\%$) on-policy and $90$ cases ($57\%$) off-policy; the corresponding median magnitude ratios are $388$ and $68$. Within the on-policy sample, the mean cosine is $+0.14$ on incorrect trajectories and $-0.11$ on correct trajectories, matching the inversion predicted by Proposition~\ref{prop:cosine_inversion}. Thus, response-level rejection alone does not determine whether the two gradients oppose each other.

\begin{figure}[t]
\centering
\includegraphics[width=\linewidth]{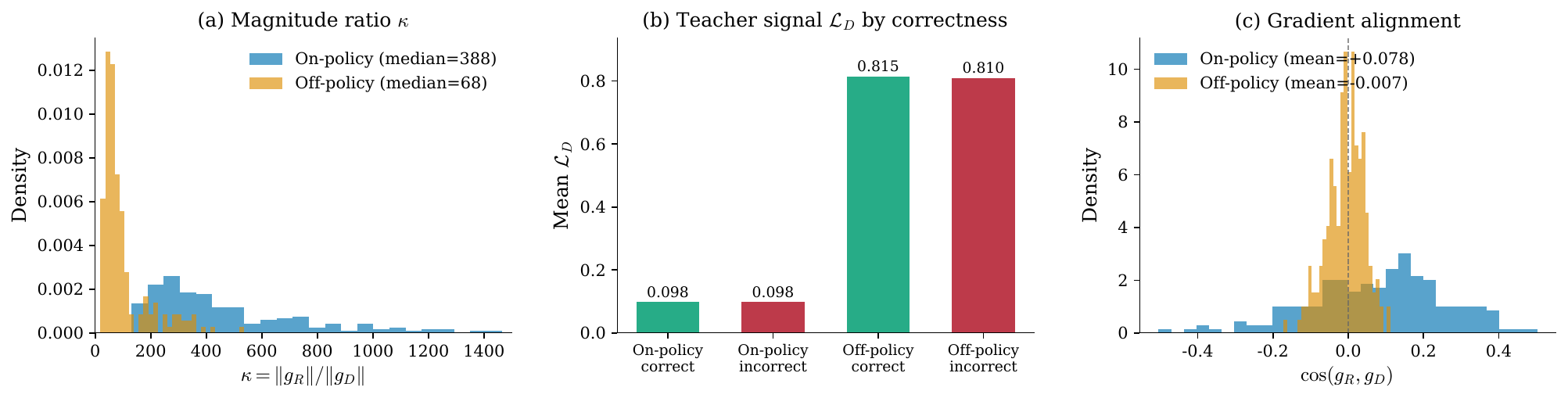}
\caption{On-policy versus off-policy semantic conflict. On-policy distillation has a higher median $\magvar{\kappa}$ but fewer incorrect trajectories with genuinely opposed reward and distillation gradients; off-policy supervision reverses this pattern.}
\label{fig:semantic_conflict_comparison}
\end{figure}

\paragraph{Token-Level NTK Conflict.}
\label{app:token_conflict_maps}

We next measure the same-position cross-signal NTK $\dirvar{K_{DR}(n)}$. Positions with $\dirvar{K_{DR}(n)}\geq0$ are locally compatible, whereas positions with $\dirvar{K_{DR}(n)}<0$ receive a teacher component that opposes reward progress. Across the measured model scales and datasets, the count-based conflict rate ranges from $0.37$ to $0.46$, with a mean of approximately $0.41$ (Figure~\ref{fig:token_conflict_map}). The narrow range shows that the conflict observed in the main-text rollout is not isolated to one model or dataset, without implying that all settings have identical conflict structure.

\begin{figure}[t]
\centering
\includegraphics[width=0.58\linewidth]{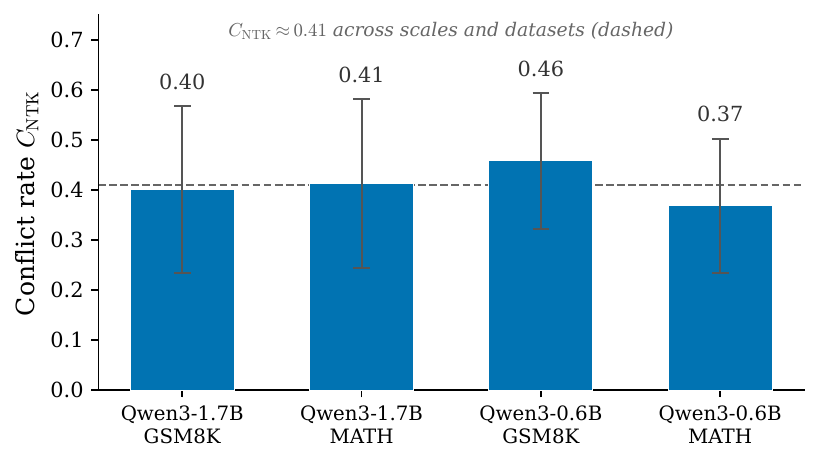}
\caption{Per-position conflict rate $\dirvar{C_{\mathrm{NTK}}}$ across model scales and datasets. The fraction of positions with negative cross-signal NTK remains between $0.37$ and $0.46$ (mean $\approx0.41$, dashed line).}
\label{fig:token_conflict_map}
\end{figure}

\begin{figure}[t]
\centering
\includegraphics[width=0.92\linewidth]{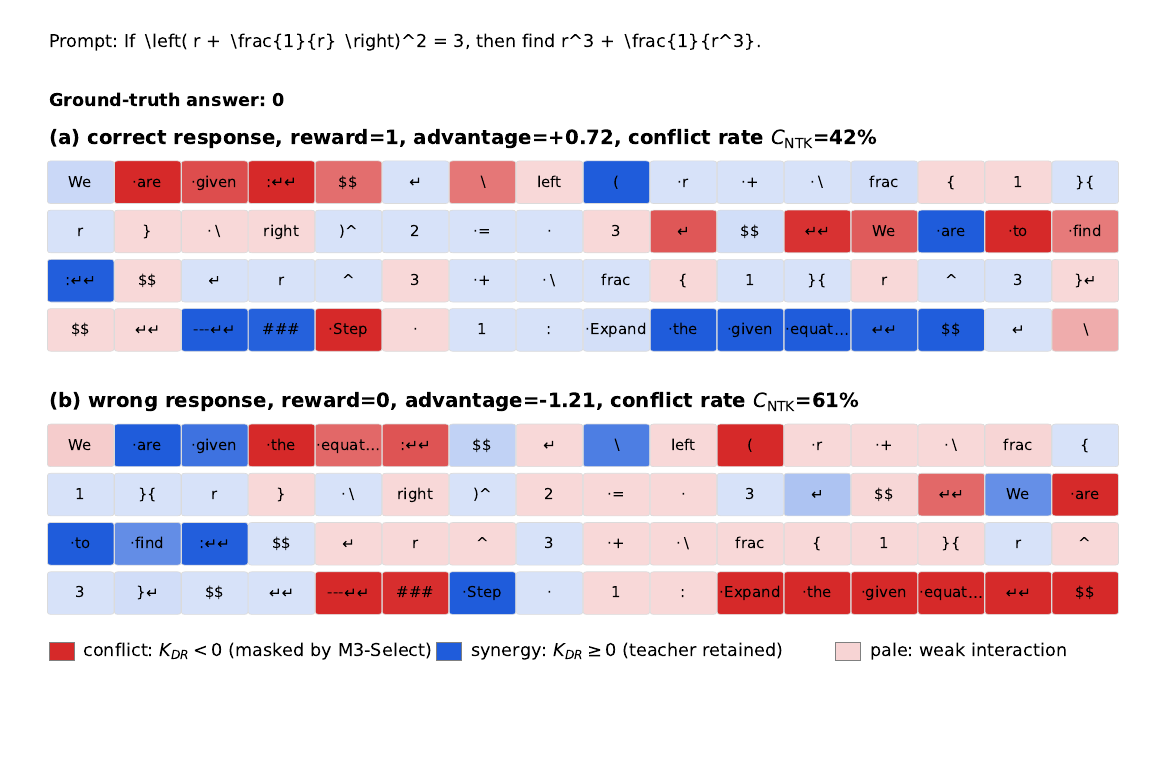}
\caption{A readable token-level case study for a correct ($A>0$) and an incorrect ($A<0$) response to the same prompt. Darker red denotes stronger local conflict ($\dirvar{K_{DR}(n)}<0$), darker blue denotes stronger compatibility, and pale colors denote weak interaction. Both responses interleave the two signal types.}
\label{fig:case_study_readable_conflict}
\end{figure}

\begin{figure}[p]
\centering
\includegraphics[width=0.92\linewidth]{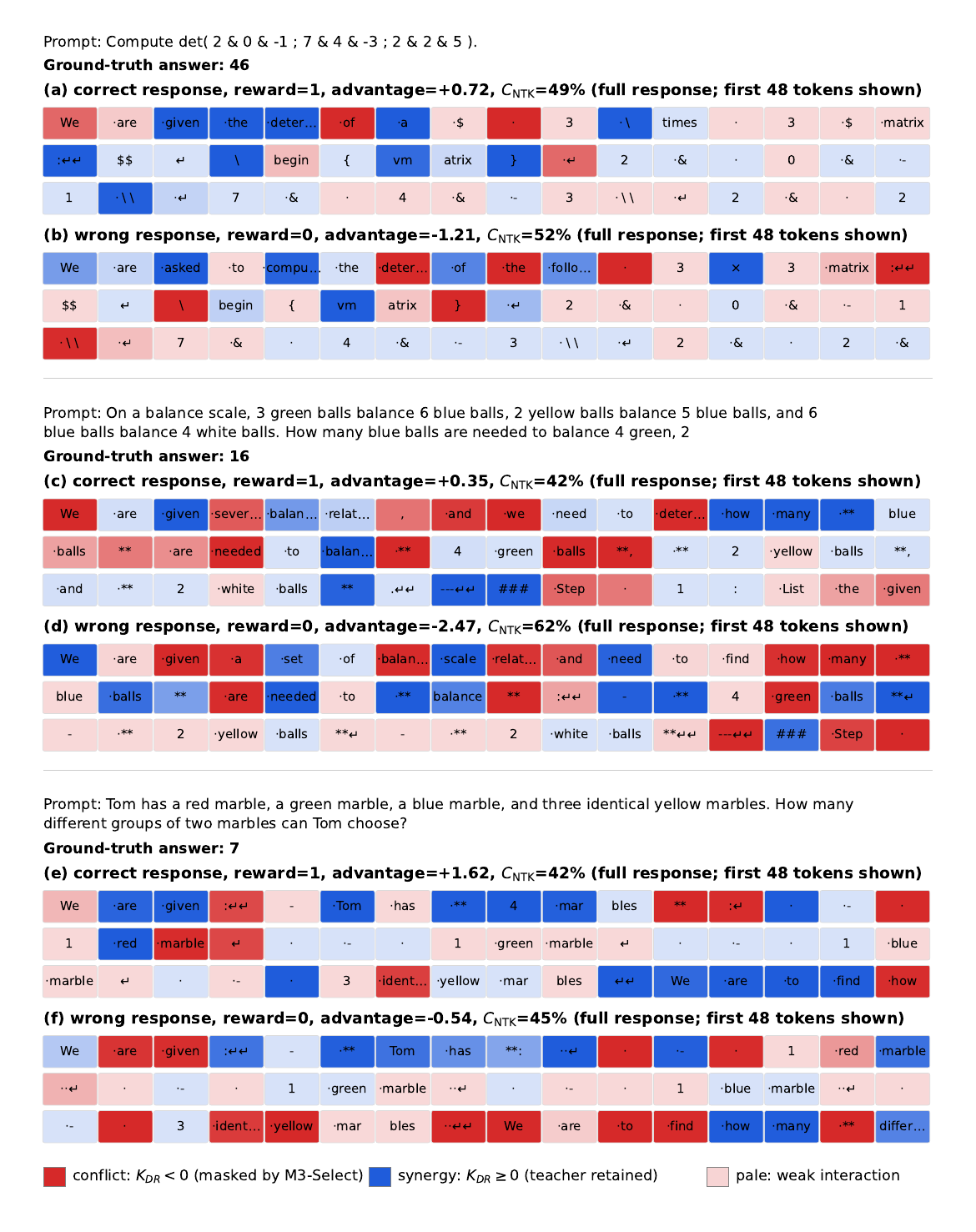}
\caption{Three additional paired token-level case studies. Headers report $\dirvar{C_{\mathrm{NTK}}}$ over each full $200$-token response; the first $48$ tokens are displayed. The incorrect rollout has the higher conflict rate in each pair ($49\%\!\to\!52\%$, $42\%\!\to\!62\%$, and $42\%\!\to\!45\%$), while every rollout contains both compatible and conflicting positions.}
\label{fig:case_study_more}
\end{figure}

Figure~\ref{fig:case_study_readable_conflict} restores token identities for one correct and one incorrect response to the same prompt. Across the four paired case studies, conflict covers $44.3\%$ of positions on correct rollouts and $54.9\%$ on incorrect rollouts, a difference of $10.6$ percentage points. Every rollout nevertheless interleaves compatible and conflicting positions, so these examples support token-level localization rather than a response-level cutoff. The count difference is descriptive for the eight visualized rollouts and is not presented as a population-level estimate.

\begin{figure}[t]
\centering
\includegraphics[width=\linewidth]{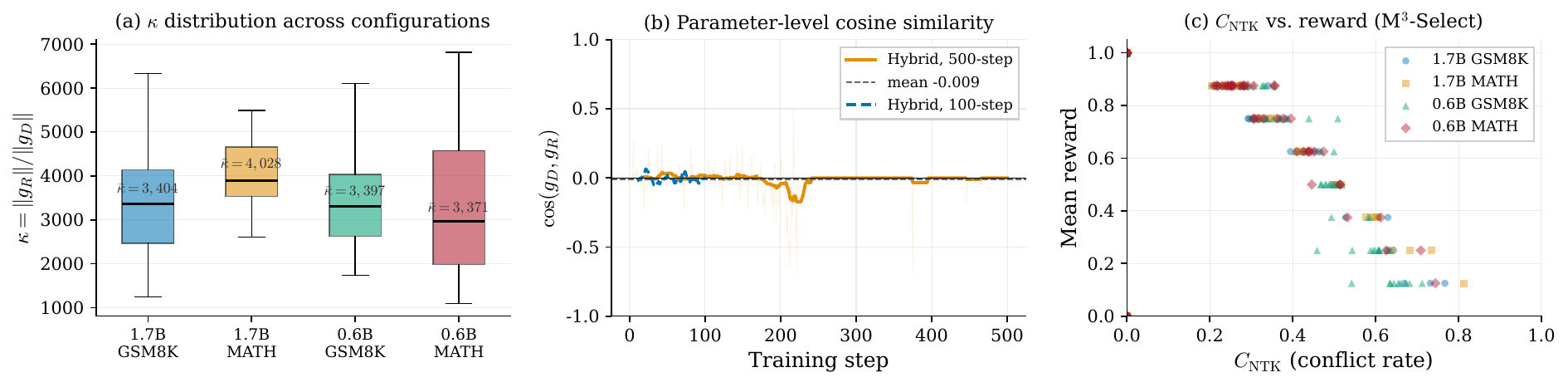}
\caption{Aggregate conflict diagnostics: \emph{(a)} mean magnitude ratio across settings; \emph{(b)} parameter-level gradient cosine; and \emph{(c)} token-level conflict rate versus reward for M$^3$-Select.}
\label{fig:conflict_metrics}
\end{figure}

\paragraph{Aggregate Conflict Diagnostics.}
\label{app:conflict_metrics}

Figure~\ref{fig:conflict_metrics} summarizes the reference Qwen diagnostics. Across the measured scales and datasets, the mean magnitude ratio remains in the $3{,}400$--$4{,}000$ range, while the parameter-level cosine averages $-0.009$. The M$^3$-Select trace also shows higher count-based conflict at lower-reward steps. This last relationship is descriptive: the plot does not establish that conflict rate alone causes or predicts method performance.

\subsection{Additional Ablations and Illustrative Examples}
\label{app:additional_experiments}

\paragraph{LoRA Rank Ablation.}
We vary the LoRA rank from $8$ to $128$ on Qwen3-1.7B GSM8K while holding the remaining M$^3$-Soft configuration fixed. The mean reward ranges from $0.824$ to $0.844$ across this $16\times$ change in trainable rank, with the best value at $r=64$ and a total spread of $0.020$ (Table~\ref{tab:lora_rank}). Thus, performance in this setting is not sensitive to the tested LoRA rank. Because $\magvar{\kappa}$ and the token-level conflict rate were not separately recorded for every rank, this ablation supports reward robustness rather than rank invariance of the underlying diagnostics.

\begin{table}[h]
\centering
\caption{LoRA-rank ablation for M$^3$-Soft on Qwen3-1.7B GSM8K over $100$ training steps.}
\label{tab:lora_rank}
\begin{tabular}{lccccc}
\toprule
\rowcolor{gray!10}
LoRA Rank & $r=8$ & $r=16$ & $r=32$ & $r=64$ & $r=128$ \\
\midrule
Mean Reward & 0.826 & 0.824 & 0.828 & \textbf{0.844} & 0.841 \\
\bottomrule
\end{tabular}
\end{table}

\paragraph{Illustrative Failure Patterns.}
The following cases are stylized examples distilled from qualitative patterns observed in the GSM8K and SVAMP runs; they are not verbatim training trajectories or additional controlled experiments. Their purpose is to show how magnitude imbalance, group-level cancellation, and token-level sign conflict can appear in concrete reasoning traces. In the examples, $+$ and $-$ denote positive and negative group-relative advantages, respectively.

\begin{badcasebox}[Bad Case~1: length inflation under drowning ($\magvar{\kappa} \gg \magvar{\kappa}^*$)]{bcmag}
On Llama-3.2-1B SVAMP ($\bar{\magvar{\kappa}}\!\approx\!1.3\!\times\!10^{4}$), the RL gradient dominates the teacher gradient by roughly four orders of magnitude. Because the verifier rewards only the final numeric answer, the group-normalized advantage rewards \emph{any} continuation that eventually reaches the correct digit and penalizes short-but-wrong ones. When distillation is drowned out, the model no longer receives the teacher's ``brevity + structure'' prior, response length inflates steadily as training proceeds, and reward collapses to $0.072$ (Table~\ref{tab:cross_dataset}, row Llama/SVAMP/GRPO). A stylized post-drowning rollout of the kind we observe:
\begin{quote}\small\ttfamily
Let $x$ be the answer. We are told that Melanie has $3\cdot 4$ apples\dots so $x=12$. But wait, let us re-verify: $3\cdot 4=12$, so indeed $x=12$. Actually, checking once more, $x=12$. \dots Answer: $\mathbf{12}$.
\end{quote}
The re-verification loops are reward-neutral (verifier only checks ``$\mathbf{12}$'') but teacher-negative. Under M3 the magnitude normalization restores the teacher's unit-scale voice---the brevity{+}structure prior is no longer drowned by the $10^4$ norm ratio---and the boundary admits this anti-repetition pressure exactly where it is reward-compatible ($\dirvar{K_{DR}(n)}\!\geq\!0$, e.g.\ in negative-advantage rollouts whose RL residual also pushes against the loops), so length inflation is suppressed and reward reaches $0.822$ (same table row, M3-Soft; the hard-gated M3-Select variant recovers $0.488$). \emph{This is the token-level dual of the length-drift phenomenon reported in prior exploration-boundary work: the drowning threshold acts as an implicit ``eos-suppression'' bias.}
\end{badcasebox}

\begin{badcasebox}[Bad Case~2: shared-prefix contamination in the same rollout group]{bcsign}
GRPO computes advantages per rollout, so a group of $G$ trajectories that share a long prefix will apply \emph{opposite-sign} updates to that prefix whenever their final answers disagree. Consider a GSM8K group of two rollouts from the same prompt:
\begin{quote}\small\ttfamily
Rollout 1 ($+$): \fbox{Let $x$ be the number of apples.} $x=5\cdot 3=15$. Answer: $\mathbf{15}$.\\
Rollout 2 ($-$): \fbox{Let $x$ be the number of apples.} $x=5+3=8$. Answer: $\mathbf{8}$.
\end{quote}
Denote the shared prefix as $s$. The naive RL contribution at any $s$-token $y_n$ is
\[
\mathcal{G}_R^n\big|_{y_n\in s} \;=\; -\bigl[(r_1-b) + (r_2-b)\bigr]\bigl(\ve_{y_n} - \vp_S^n\bigr) \;=\; 0,
\]
since the group-mean baseline $b=\tfrac{1}{2}(r_1+r_2)$ centers the two advantages: the RL signal on the shared prefix cancels exactly in this two-rollout group---and nearly so for general $G$, where unbalanced $+$/$-$ counts and length-normalization weights leave a small residual. But the distillation signal $\mathcal{G}_D^n\!=\!\vp_S^n-\vp_T^n$ remains fully alive on the same tokens, and the teacher pushes $\vp_S$ toward its preferred verbalization---while the small residual RL noise flips sign randomly across steps. The result is an unstable prefix whose gradient variance is dominated by cancellation noise (large $\gamma_n$, Definition~\ref{def:gradient_folding} extension in \S\ref{app:folding_bridges}). M3-Select computes $\dirvar{\cos\varphi_n}$ between $\mathcal{G}_D^n$ and the \emph{per-rollout} $\mathcal{G}_R^n$, detects $\bar c_n\!=\!+1$ on the shared prefix of the winning rollout (teacher and RL prefer the same continuation there), and admits the teacher at its full budget $\alpha_{\max}$: after per-token normalization the distillation direction is the only \emph{persistent} signal on the prefix, anchoring it against the sign-flipping residual noise. The RL term is never masked---its group-averaged contribution self-cancels as shown above---so reward supervision effectively acts only on the divergence region (``$5\!\cdot\!3\!=\!15$'' vs ``$5\!+\!3\!=\!8$''), where the two rollouts' residuals no longer cancel. \emph{This is the on-policy hybrid analogue of the shared-prefix contamination reported by \citet{ren2024ntk} (``I-ate-lunch, happy/unhappy'' example).}
\end{badcasebox}

\begin{badcasebox}[Bad Case~3: correct intermediate step suppressed by wrong final answer]{bcsign}
Even when the group contains only one $(-)$ rollout, an internally-correct intermediate step is penalized because GRPO's outcome reward propagates the trajectory-level sign to every token. Consider a two-step arithmetic rollout:
\begin{quote}\small\ttfamily
Rollout $A$ ($-$): $5\cdot 4=$ \fbox{$20$}, then $20+3=$ \fbox{$24$}. Answer: $\mathbf{24}$.\\
Rollout $B$ ($+$): $5\cdot 4=$ \fbox{$20$}, then $20+3=$ \fbox{$23$}. Answer: $\mathbf{23}$.
\end{quote}
The intermediate token ``$20$'' is arithmetically correct in \emph{both} rollouts, and the teacher's distributional prior $\vp_T$ places $>0.9$ mass on it. Under naive hybrid,
\[
\mathcal{G}_R^n\big|_{y_n=``20''}^{\text{Rollout }A} \;=\; -(r_A-b)(\ve_{20}-\vp_S^n),
\]
with $r_A-b < 0$, so the RL update \emph{lowers} the probability of the correct intermediate. The per-position diagnostic (Proposition~\ref{prop:per_position}) yields $\langle\mathcal{G}_D^n,\mathcal{G}_R^n\rangle \propto -(r_A-b)\sigma_n < 0$, so $\bar c_n\!=\!-1$ and M3-Select withholds teacher supervision from Rollout~$A$'s update at this token---the gate acts on the teacher, never on RL. The repair instead comes from the group: Rollout~$B$'s equal-and-opposite RL contribution cancels Rollout~$A$'s in the aggregate, while on Rollout~$B$ the diagnostic reads $\bar c_n\!=\!+1$ and the teacher is admitted, so the surviving net update pushes $p_S(``20'')\!\uparrow$. Correct-intermediate tokens under negative advantage arise whenever a group mixes right and wrong final answers built on shared sub-computations; protecting them through this admitted teacher vote is a driver of the $0.647$ vs $0.000$ reference-rate gap between M3-Soft and GRPO (Table~\ref{tab:cross_dataset}). \emph{This is the dense--sparse counterpart of the ``$1{+}1{=}2$'' intermediate-step cancellation phenomenon documented in prior exploration-boundary analyses.}
\end{badcasebox}

\begin{badcasebox}[Bad Case~4: directional cancellation on high-entropy scaffold tokens]{bcmag}
Certain tokens are structurally common to almost every correct \emph{and} wrong solution: the ``$=$'' sign, ``Answer:'', ``$\backslash n$'', punctuation, and low-content connectives (``so'', ``then''). The teacher assigns $\vp_T\!>\!0.9$ on these, and the RL signal is noisy and near-zero in expectation (they appear roughly equally in $+$ and $-$ rollouts). Yet the per-step RL contribution is non-zero and has large variance, producing a gradient-cancellation rate
\[
\gamma_n \;=\; 1 - \frac{\bigl\|\tfrac1G\sum_g \nabla_{\vtheta} \ell_g(y_n)\bigr\|^2}{\tfrac1G\sum_g \bigl\|\nabla_{\vtheta} \ell_g(y_n)\bigr\|^2}
\]
approaching $1$ at these positions, in sharp contrast to content tokens, whose group updates are directionally aligned. Here the soft gate is the right tool: M3-Soft assigns these positions a stable intermediate budget $\alpha_n\!=\!\alpha_{\max}\sigma(\beta\dirvar{\cos\varphi_n})\!\approx\!\alpha_{\max}/2$ that, at small $\beta$, is insensitive to the noisy sign of $\widehat{\dirvar{\cos\varphi_n}}$, whereas M3-Select's hard indicator flips with that sign and evicts the teacher on roughly half the scaffold positions at random. Since the group-averaged RL residual is near zero-mean here while the normalized teacher direction is persistent, distillation dominates the \emph{expected} update without any signal being masked---which is why M3-Soft with $\beta\!=\!0.1$ recovers scaffold-token fluency on the low-$\magvar{\kappa}$ Qwen2.5-SVAMP cell (\S\ref{app:beta_sensitivity}).
\end{badcasebox}

Table~\ref{tab:bad_case_taxonomy} summarizes what each example is intended to illustrate. The first case concerns magnitude imbalance, the next two concern sign structure under outcome-level credit assignment, and the fourth concerns uncertainty near the compatibility boundary.


\begin{table}[t]
\centering
\footnotesize
\caption{Stylized failure patterns, observable signatures, and corresponding interventions.}
\label{tab:bad_case_taxonomy}
\setlength{\tabcolsep}{5pt}
\renewcommand{\arraystretch}{1.18}
\begin{tabular}{@{}p{3.8cm}p{5.5cm}p{3.2cm}@{}}
\toprule
Failure pattern & Observable signature & Intervention \\
\midrule
\textbf{Length inflation}\newline
\emph{Magnitude drowning}
& Large $\magvar{\kappa}$; repetitive reasoning not penalized by the final-answer verifier
& Magnitude normalization \\
\textbf{Shared-prefix cancellation}\newline
\emph{Group-level cancellation}
& Weak aggregate reward signal on tokens shared by opposite-advantage rollouts
& Per-token compatibility gate \\
\textbf{Correct-intermediate suppression}\newline
\emph{Token-level sign conflict}
& $\dirvar{K_{DR}(n)}<0$ on a locally correct token inside a rejected response
& M$^3$-Select or M$^3$-Soft \\
\textbf{Scaffold-token instability}\newline
\emph{Boundary uncertainty}
& $|\dirvar{K_{DR}(n)}|\approx0$ on structural or low-content tokens
& M$^3$-Soft \\
\bottomrule
\end{tabular}
\end{table}




\section{Exploration Boundary Framework: Full Statements and Boundary Definition}
\label{app:exploration_boundary}

This appendix states the intra-RL quantities underlying the cross-signal analysis in Section~\ref{sec:cross_signal_ntk} and defines the reward--teacher compatibility boundary.

\subsection{Exploration Boundary and Gradient Folding in RL}
\label{sec:exploration_boundary}

\begin{definition}[Gradient Diversity and Exploration Boundary]
\label{def:gradient_diversity}
For $G$ rollouts, let $\vg^{(i)}=A_i\nabla_{\vtheta}\log\pi_{\vtheta}(y_i\mid x)$, with $A_i=(r^{(i)}-\bar r)/\sigma_r$, and write $\bar{\vg}=G^{-1}\sum_i \vg^{(i)}$. For $\bar{\vg}\ne0$, define
\begin{equation}
\Delta(\vtheta)=\frac{G^{-1}\sum_i\|\vg^{(i)}\|^2}{\|\bar{\vg}\|^2}\geq1,
\qquad B_S(\vtheta)=N_{\max}\Delta(\vtheta).
\end{equation}
The inequality is Jensen's inequality; equality holds when all gradients coincide, while equal-norm orthogonal gradients give $\Delta=G$. Large $\Delta$ measures cancellation relative to the individual gradient energy. Here $N_{\max}$ is the rollout budget per prompt, and $B_S$ is the exploration-boundary index.
\end{definition}

\begin{definition}[Gradient Folding and Cancellation Rate]
\label{def:gradient_folding}
With binary rewards, positive- and negative-advantage trajectories can oppose one another at shared token positions; we call this \emph{gradient folding}. Its cancellation rate and surviving gradient magnitude satisfy
\begin{equation}
\gamma=1-\Delta^{-1}
=1-\frac{\|\sum_i \vg^{(i)}\|^2}{G\sum_i\|\vg^{(i)}\|^2},
\qquad
\|\bar{\vg}\|=\sqrt{1-\gamma}\left(G^{-1}\sum_i\|\vg^{(i)}\|^2\right)^{1/2}.
\end{equation}
Thus $\gamma\in[0,1)$ measures the fraction of mean squared gradient magnitude canceled. Complete cancellation gives $\gamma=1$ and $\Delta=+\infty$, provided the individual gradients are not all zero.
\end{definition}

\begin{definition}[Token-Level NTK]
\label{def:token_ntk}
Let $\mJ^n=\nabla_{\vtheta}\vz^n\in\mathbb R^{d_{\mathrm{par}}\times |\cV|}$. The vocabulary-space kernel is $\mK(s,t)=(\mJ^s)^\top \mJ^t$, and its sampled-token contraction is
\begin{equation}
K_t(\tau,s,t)=\left\langle\nabla_{\vtheta}\log\pi_{\vtheta}(a_s\mid\tau_{<s}),\,
\nabla_{\vtheta}\log\pi_{\vtheta}(a_t\mid\tau_{<t})\right\rangle.
\end{equation}
For an update from position $n$ alone, $\dot{\vz}^n=-\mK(n,n)\vdelta^n$; a summed update gives $\dot{\vz}^n=-\sum_m\mK(n,m)\vdelta^m$. These kernels separate local effects from interactions through shared parameters and provide the geometry for position-level masking.
\end{definition}

\begin{definition}[Reward--Teacher Compatibility Boundary]
\label{def:compatibility_boundary}
For $\dirvar{K_{DR}(n)}=\langle \mJ^n\vdelta_D^n,\mJ^n\vdelta_R^n\rangle$, define
\begin{equation}
\partial\cB=\cB^0=\{n:\dirvar{K_{DR}(n)}=0\},
\qquad \cB^\pm=\{n:\pm\dirvar{K_{DR}(n)}>0\}.
\end{equation}
The compatible region $\cB^+$ contributes positively to the local reward projection, the orthogonal region $\cB^0$ contributes zero, and the conflicting region $\cB^-$ contributes negatively. M3-Select admits the teacher on $\cB^+\cup\cB^0$ and masks $\cB^-$.
\end{definition}

\section{Extended Theoretical Analysis}
\label{app:theory_extended}

This appendix contains extended theoretical results referenced in Sections~\ref{sec:cross_signal_ntk} and~\ref{sec:consequence}.

\begin{corollary}[One-Step Logits Degradation]
\label{cor:logit_degradation}
For advantage $A\ne0$, write $\vg_R^n=\mJ^n\vdelta_R^n$ and $\vg_D^n=\mJ^n\vdelta_D^n$. The first-order contribution of position $n$'s hybrid update to its sampled-token log-probability is
\begin{equation}
\begin{aligned}
\Delta_n\log p_S(\hat y_n\mid\hat y_{<n})
&\approx\left\langle-\frac{\vg_R^n}{A},-\eta[(1-\alpha)\vg_R^n+\alpha \vg_D^n]\right\rangle\\
&=\frac{\eta}{A}\left[(1-\alpha)K_{RR}(n)+\alpha\dirvar{K_{DR}(n)}\right],
\end{aligned}
\end{equation}
where $K_{RR}(n)=\|\vg_R^n\|^2$ and $\nabla_{\vtheta}\log p_S(\hat y_n\mid\hat y_{<n})=-\vg_R^n/A$. At a conflicting position, distillation lowers the sampled-token log-probability when $A>0$ and raises it when $A<0$. For $A>0$, the total diagonal contribution becomes negative exactly when $\alpha|\dirvar{K_{DR}(n)}|>(1-\alpha)K_{RR}(n)$; Section~\ref{sec:one_step_exp} measures this effect.
\end{corollary}

\paragraph{Exact $\magvar{\kappa}$-growth identity.}
For nonzero gradients, direct logarithmic differentiation gives
\begin{equation}
\frac{d\log\magvar{\kappa}}{dt}
=\frac{d\log\|\vg_R\|}{dt}-\frac{d\log\|\vg_D\|}{dt}.
\label{eq:kappa_growth}
\end{equation}
Thus $\magvar{\kappa}$ grows whenever the reward norm decays more slowly than the teacher norm. If the difference of these logarithmic rates is the positive constant $\lambda$, then $\magvar{\kappa}(t)=\magvar{\kappa}(0)e^{\lambda t}$.

\begin{corollary}[Aggregate Synergy Condition]
\label{cor:aggregate_synergy}
Let $\Phi_\tau=\cos(\vg_D(\tau),\vg_R(\tau))$ be the trajectory-level alignment, with fixed conditional means $\mu_0=\expect[\Phi_\tau\mid r=0]>0$ and $\mu_1=\expect[\Phi_\tau\mid r=1]<0$. At accuracy $p$,
\begin{equation}
\expect[\Phi_\tau]=(1-p)\mu_0+p\mu_1>0
\quad\Longleftrightarrow\quad p<p^*:=\frac{\mu_0}{\mu_0-\mu_1}.
\end{equation}
\end{corollary}

\begin{corollary}[Diminishing Synergy Under Increasing Accuracy]
\label{cor:diminishing_synergy}
Under the fixed-conditional-mean model of Corollary~\ref{cor:aggregate_synergy}, $d\expect[\Phi_\tau]/dp=\mu_1-\mu_0<0$, so expected trajectory alignment crosses zero at $p^*$ and approaches $\mu_1$ as $p\to1$. The threshold is determined by the two conditional alignments.
\end{corollary}

\begin{remark}[Why NTK, not just cosine?]
\label{rem:why_ntk}
The aggregate cosine summarizes the same parameter-space inner product represented by the NTK:
$\langle \vg_D,\vg_R\rangle=\sum_{n,m}(\vdelta_D^n)^\top\mK(n,m)\vdelta_R^m$ for summed gradients. The kernel decomposition exposes which positions and cross-position interactions produce that scalar, enabling local gating that an aggregate cosine alone cannot specify.
\end{remark}

The measured threshold separates cells in which hard masking helps from cells in which a softer gate preserves more teacher signal. The local score identifies the teacher's immediate reward projection; gradient magnitude, estimation noise, and cross-position interactions determine how this local decision translates into training progress. The experiments in Sections~\ref{sec:cross_arch} and~\ref{sec:500step} compare these regimes.

\subsection{Conflict as Information Destruction}
\label{sec:info_destruction}

At conflicting positions of a positive-advantage trajectory, the teacher contribution lowers the sampled-token log-probability (Corollary~\ref{cor:logit_degradation}). Repeated contributions of this sign can erode rewarded behavior. The 500-step experiments (Section~\ref{sec:500step}) show collapse under uniform mixing in high-$\magvar{\kappa}$ cells, while M3-Select removes these negative local teacher projections before the update. This links the local mechanism to the observed training trajectories.

\section{Extended Method Theory}
\label{app:method_theory}

This appendix presents the full theoretical analysis of M3 and M3-Select. All formal statements and proof sketches summarized in Section~\ref{sec:methods} are collected here.

\subsection{M3-EG: Fast--Slow Extragradient Update}
\label{app:m3_eg}
The extragradient variant of Section~\ref{sec:m3_select} keeps the boundary gate $\alpha_n^*$ of Eq.~\eqref{eq:m3_select_alpha} but replaces the synchronous mixture with a two-timescale schedule. From the anchor $\vtheta$, the inner \emph{fast} step applies only the boundary-gated teacher to reach a look-ahead point
\begin{equation}
\tilde{\vtheta}=\vtheta-\eta_{\mathrm{in}}\,\tfrac{1}{N}\textstyle\sum_{n=1}^{N} \mJ^n\alpha_n^*\hat{\vdelta}_D^n .
\label{eq:m3_eg_fast}
\end{equation}
The on-policy reward group is then scored \emph{at} $\tilde{\vtheta}$, but its gradient is applied as an outer correction anchored at the original $\vtheta$,
\begin{equation}
\vtheta^+=\vtheta-\eta_{\mathrm{out}}\,\tfrac{1}{N}\textstyle\sum_{n=1}^{N} \mJ^n\hat{\vdelta}_R^n\big|_{\tilde{\vtheta}} ,
\label{eq:m3_eg_slow}
\end{equation}
using a first-order approximation that does not backpropagate through the inner step: in practice one caches the slow-gradient direction $\tfrac{1}{N}\sum_n \mJ^n\hat{\vdelta}_R^n|_{\tilde{\vtheta}}$, restores $\vtheta$, and then applies Eq.~\eqref{eq:m3_eg_slow}. The step sizes $(\eta_{\mathrm{in}},\eta_{\mathrm{out}})$ play the roles of the inner (look-ahead) and outer (correction) rates of a standard extragradient scheme. Algorithm~\ref{alg:m3_eg} states the full procedure.

\begin{algorithm}[htbp]
\caption{M3-EG: Fast--Slow Extragradient Update}
\label{alg:m3_eg}
\begin{enumerate}[label=\arabic*:,leftmargin=1.8em,itemsep=2pt,topsep=3pt]
\item[] \textbf{Input:} gate ceiling $\alpha_{\max}$, inner (look-ahead) rate $\eta_{\mathrm{in}}$, outer (correction) rate $\eta_{\mathrm{out}}$, smoothing $\epsilon$
\item \textbf{For} step $t=1,\ldots,T$ \textbf{do}
    \item Compute per-position residuals $\vdelta_D^n,\vdelta_R^n$, Jacobians $\mJ^n$, and scores $\dirvar{K_{DR}(n)}\leftarrow\langle \mJ^n\vdelta_D^n,\mJ^n\vdelta_R^n\rangle$.
    \item Gate $\alpha_n^*\leftarrow\alpha_{\max}\ind\{\dirvar{K_{DR}(n)}\geq0\}$ (hard) or $\alpha_{\max}\sigma(\beta\dirvar{\cos\varphi_n})$ (soft); normalize $\hat{\vdelta}_D^n,\hat{\vdelta}_R^n$.
    \item \textbf{Fast / look-ahead (teacher only):}\ \ $\tilde{\vtheta}\leftarrow\vtheta_t-\eta_{\mathrm{in}}\,\frac1N\sum_n \mJ^n\alpha_n^*\hat{\vdelta}_D^n$.\hfill$\triangleright$ hard gate masks $\Omega_-$
    \item \textbf{Re-score at $\tilde{\vtheta}$ (reward only):}\ \ $\vh_R(\tilde{\vtheta})\leftarrow\frac1N\sum_n \mJ^n\hat{\vdelta}_R^n\big|_{\tilde{\vtheta}}$ (no backprop through $\tilde{\vtheta}$).
    \item \textbf{Slow / correction (anchored at $\vtheta_t$):}\ \ $\vtheta_{t+1}\leftarrow\vtheta_t-\eta_{\mathrm{out}}\,\vh_R(\tilde{\vtheta})$.
\item[] \textbf{End for}
\end{enumerate}
\end{algorithm}

\paragraph{Conflict isolation.}
Write $\vh_R(\vtheta)=N^{-1}\sum_n\mJ^n\hat{\vdelta}_R^n$ for the normalized reward field and $\vm_D=N^{-1}\sum_n\mJ^n\alpha_n^*\hat{\vdelta}_D^n$ for the gated teacher field. The executed move is $-\eta_{\mathrm{out}} \vh_R(\vtheta-\eta_{\mathrm{in}} \vm_D)$, so teacher information enters through the look-ahead location. Positive rescaling of an unsmoothed teacher residual leaves $\vm_D$ unchanged.

\begin{proposition}[Local Reward Descent of the Extragradient Field]
\label{prop:eg_conflict_isolation}
Let $\vg_R=\nabla\cL_R(\vtheta)$, assume $a_R=\langle \vh_R(\vtheta),\vg_R\rangle>0$, and let $\vh_R$ be $L_h$-Lipschitz along the inner step. Then
\begin{equation}
\begin{aligned}
\langle\vtheta-\vtheta^+,\vg_R\rangle
&=\eta_{\mathrm{out}}\langle \vh_R(\vtheta-\eta_{\mathrm{in}} \vm_D),\vg_R\rangle\\
&\geq\eta_{\mathrm{out}}\left(a_R-\eta_{\mathrm{in}} L_h\|\vm_D\|\|\vg_R\|\right)>0
\end{aligned}
\end{equation}
whenever $\eta_{\mathrm{in}} L_h\|\vm_D\|\|\vg_R\|<a_R$. For sufficiently small outer step $\eta_{\mathrm{out}}$, this is a reward-descent step. When $\vh_R=\vg_R$, its first-order expansion is $\eta_{\mathrm{out}}[\|\vg_R\|^2-\eta_{\mathrm{in}}\langle\mH_R\vm_D,\vg_R\rangle]+O(\eta_{\mathrm{out}}\eta_{\mathrm{in}}^2)$, with $\mH_R=\nabla^2\cL_R(\vtheta)$.
\end{proposition}

The inner normalization controls the teacher field's scale, while the hard gate removes its negative local reward projections. The outer descent condition above quantifies the additional effect of moving the reward evaluation point (empirical comparison: Table~\ref{tab:cross_arch}).

\begin{algorithm}[htbp]
\caption{M3-Select: Boundary-Guided Token-Level Mixing}
\label{alg:m3_select}
\begin{enumerate}[label=\arabic*:,leftmargin=1.8em,itemsep=2pt,topsep=3pt]
\item[] \textbf{Input:} $\alpha_{\max}$, step size $\eta$, smoothing $\epsilon$, EMA update scale $\bar s$
\item \textbf{For} step $t=1,\ldots,T$ \textbf{do}
    \item Compute per-position residuals $\vdelta_D^n,\vdelta_R^n$ and Jacobians $\mJ^n$.
    \item Compute compatibility scores $\dirvar{K_{DR}(n)}\leftarrow\langle \mJ^n\vdelta_D^n,\mJ^n\vdelta_R^n\rangle$.
    \item Set $\alpha_n\leftarrow\alpha_{\max}\ind\{\dirvar{K_{DR}(n)}\geq0\}$ and record $\dirvar{C_{\mathrm{NTK}}}$.
    \item Normalize residuals $\hat{\vdelta}_D^n\leftarrow\vdelta_D^n/(\|\vdelta_D^n\|+\epsilon)$ and $\hat{\vdelta}_R^n\leftarrow\vdelta_R^n/(\|\vdelta_R^n\|+\epsilon)$.
    \item Update $\vtheta_{t+1}\leftarrow\vtheta_t-\eta\bar s\frac{1}{N}\sum_n\mJ^n[\alpha_n\hat{\vdelta}_D^n+(1-\alpha_n)\hat{\vdelta}_R^n]$.
\item[] \textbf{End for}
\end{enumerate}
\end{algorithm}

\paragraph{Quadratic sub-optimality of a fixed mixing coefficient.}
Let $\vd=\vg_D-\vg_R$ and let $\alpha^*\in[0,1]$ minimize $\|\vg_H(\alpha)\|^2$. Expanding around the constrained minimizer gives
\begin{equation}
\begin{aligned}
\|\vg_H(\alpha)\|^2-\|\vg_H(\alpha^*)\|^2
&=\|\vd\|^2(\alpha-\alpha^*)^2
 +2(\alpha-\alpha^*)\langle \vg_H(\alpha^*),\vd\rangle\\
&\geq\|\vd\|^2(\alpha-\alpha^*)^2.
\end{aligned}
\label{eq:quadratic_gap}
\end{equation}
The last term on the first line is nonnegative by constrained optimality and vanishes at an interior minimizer. The curvature is $\|\vd\|^2=\magvar{r_D}^2+\magvar{r_R}^2-2\dirvar{\Phi_{DR}}\magvar{r_D}\magvar{r_R}$.
For a random vector $\vg$, we use the total variance
\[
\operatorname{Var}(\vg)
:=\mathbb{E}\|\vg-\mathbb{E}\vg\|^2
=\operatorname{tr}\operatorname{Cov}(\vg),
\]
with $\operatorname{Var}_t$ denoting conditioning on the history available
before the current gradient sample.
\begin{theorem}[Boundary-Gated Reward Convergence]
\label{thm:boundary_pareto_convergence}

Assume $\cL_R$ is $L$-smooth and bounded below. In the conditional model of Appendix~\ref{app:proof_boundary_pareto}, each position has equal-norm reward and deterministic teacher means $\vu_n,\vd_n$, position subspaces are orthogonal, gates are fixed before fresh reward noise is sampled, and $\nabla\cL_R=N^{-1}\sum_n \vu_n$. Define $c_n=\langle\vu_n,\vd_n\rangle/\|\vu_n\|^2$ and $w_n=\|\vu_n\|^2/\sum_m\|\vu_m\|^2$, taking $c_n=0$ at zero-norm positions. The population-sign gate is $\alpha_n=\alpha_{\max}\ind\{c_n\geq0\}$, and $\rho_{R,t}=\sum_nw_n[(1-\alpha_n)+\alpha_nc_n]$.
For $\alpha_{\max}<1$, choose deterministic bounds $0<\underline\rho_R\leq\rho_{R,t}$ and $\operatorname{Var}_t(\hat{\vg}_H^{\mathrm{sel}})\leq v_\alpha\sigma_R^2$ valid at every iteration and history; $\underline\rho_R=1-\alpha_{\max}$ is admissible. If $\eta\leq\underline\rho_R/L$, then
\begin{equation}
\frac1T\sum_{t=0}^{T-1}\expect\|\nabla\cL_R(\vtheta_t)\|^2
\leq\frac{1}{\underline\rho_R}
\left[\frac{2(\cL_R(\vtheta_0)-\cL_R^*)}{\eta T}
+L\eta v_\alpha\sigma_R^2\right].
\end{equation}
\end{theorem}

\begin{proposition}[Conditional Variance of Boundary-Gated Mixing]
\label{prop:boundary_variance_reduction}
In the preceding conditional model, let $\nu_n=\expect_t\|\vxi_n\|^2$ be the fresh reward-noise variance at position $n$. Then
\begin{equation}
\operatorname{Var}_t(\vg_H^{\mathrm{sel}})
=\frac1{N^2}\sum_n(1-\alpha_n)^2\nu_n
=v_t\sigma_{R,t}^2,
\end{equation}
where $\sigma_{R,t}^2=N^{-2}\sum_n\nu_n$ and $v_t=\sum_n(1-\alpha_n)^2\nu_n/\sum_n\nu_n\leq1$ when noise is nonzero; take $v_t=N^{-1}\sum_n(1-\alpha_n)^2$ when all $\nu_n=0$. Choose a deterministic $v_\alpha\geq v_t$ valid uniformly over iterations and histories; $v_\alpha=1$ is always admissible when $\sigma_{R,t}^2\leq\sigma_R^2$. For equal position variances and admitted fraction $\rho_{\mathrm{adm}}=|\Omega_+\cup\Omega_0|/N$,
$v_t=1-(2\alpha_{\max}-\alpha_{\max}^2)\rho_{\mathrm{adm}}=(1-\bar\alpha)^2+\alpha_{\max}^2\rho_{\mathrm{adm}}(1-\rho_{\mathrm{adm}})$, where $\bar\alpha=\alpha_{\max}\rho_{\mathrm{adm}}$.
\end{proposition}

\noindent The local reward-projection guarantee underlying both statements is
\begin{equation}
\Pi_R(n)\geq (1-\alpha_{\max})\|\mJ^n\hat{\vdelta}_R^n\|^2 \;\text{ for } n\in\Omega_+\cup\Omega_0,
\qquad
\Pi_R(n)=\|\mJ^n\hat{\vdelta}_R^n\|^2 \;\text{ for } n\in\Omega_-.
\label{eq:positive_projection_app}
\end{equation}

\begin{theorem}[M3-Norm Convergence Guarantee]
\label{thm:convergence}
Assume $\cL_R$ is $L$-smooth and bounded below, with an unbiased reward estimator of conditional variance at most $\sigma_R^2$. Conditional on each iterate, let the teacher direction be deterministic, matched to the population reward-gradient norm, and have alignment at least $\phi_*$ (Appendix~\ref{app:proof_convergence}). For fixed $\alpha\in(0,1)$, put $c=1-\alpha+\alpha\phi_*>0$. If $\eta\leq c/L$, then
\begin{equation}
\frac1T\sum_{t=0}^{T-1}\expect\|\nabla\cL_R(\vtheta_t)\|^2
\leq\frac1c\left[\frac{2(\cL_R^0-\cL_R^*)}{\eta T}
+L\eta(1-\alpha)^2\sigma_R^2\right].
\end{equation}
\end{theorem}

\begin{proposition}[Variance Reduction via Distillation]
\label{prop:variance_reduction}
Under the conditional deterministic-teacher model of Theorem~\ref{thm:convergence}, $\operatorname{Var}_t(\hat{\vg}_H)=(1-\alpha)^2\operatorname{Var}_t(\hat{\vg}_R)\leq(1-\alpha)^2\sigma_R^2$.
\end{proposition}

\begin{proposition}[Optimal Per-Position Mixing]
\label{prop:optimal_alpha_n}
For unit parameter-space directions with cosine $c_n$, let $\Pi_R(n;\alpha)=1-\alpha+\alpha c_n$. Maximizing the calibrated objective $\Pi_R(n;\alpha)+\alpha=1+\alpha c_n$ over $[0,\alpha_{\max}]$ gives the hard gate $\alpha_n^*=\alpha_{\max}\ind\{c_n\geq0\}$, with ties assigned to the teacher. Under additive logistic score noise of scale $1/\beta$, its expected allocation is $\alpha_{\max}\sigma(\beta c_n)$, the M3-Soft gate.
\end{proposition}

\begin{proposition}[GradNorm Degeneracy under $\magvar{\kappa} \gg 1$]
\label{prop:gradnorm_degeneracy}
At equal target training rates, impose $w_D+w_R=1$ and norm balance $w_D\|\vg_D\|=w_R\|\vg_R\|$. Then $w_D^*=\magvar{\kappa}/(1+\magvar{\kappa})$ and $w_R^*=1/(1+\magvar{\kappa})$, giving $w_R^*\approx0.03\%$ at $\magvar{\kappa}=3{,}400$. While $w_D/w_R=\Theta(1)$, the teacher's share of weighted gradient magnitude is $w_D/(w_D+w_R\magvar{\kappa})=\Theta(\magvar{\kappa}^{-1})$.
\end{proposition}

A natural loss-level alternative to fixed mixing is to gate $\alpha$ on the positive part of the alignment functional, $\dirvar{\Phi_{DR}}^+ = \max(\dirvar{\Phi_{DR}}, 0)$, yielding the adaptive schedule
\begin{equation}
\alpha_{\text{adaptive}} = \alpha_{\max} \cdot \dirvar{\Phi_{DR}}^+ / \max(\dirvar{\Phi_{DR}}^+, \epsilon).
\label{eq:phi_adaptive}
\end{equation}
When the gate uses an exponentially smoothed alignment estimate, its response to a sign change has the following delay.

\begin{proposition}[Phase Delay in Adaptive Mixing]
\label{prop:phase_delay}
Let $\widehat\Phi_{t+1}=\rho\widehat\Phi_t+(1-\rho)\Phi_t$, with $0<\rho<1$. If alignment changes from $\Phi_+>0$ to $\Phi_-<0$ at $t_1$ and $\widehat\Phi_{t_1}=\Phi_+$, then
\begin{equation}
\widehat\Phi_{t_1+k}=\Phi_-+(\Phi_+-\Phi_-)\rho^k,
\qquad
k_* = \left\lceil\frac{\log[(\Phi_+-\Phi_-)/|\Phi_-|]}{\log(1/\rho)}\right\rceil.
\end{equation}
The estimate remains positive for $k<k_*$, so the teacher gate stays active during that interval.
\end{proposition}

\begin{proposition}[Empirical Invariance of $\magvar{\kappa}$ under LoRA Rank]
\label{prop:kappa_scaling}
Within the measured LoRA-rank range $r\in\{8,16,32,64\}$, $\bar{\magvar{\kappa}}=3{,}476\pm301$ (CV $=8.7\%$). On Qwen3-GSM8K, the corresponding scale comparison is $\magvar{\kappa}\approx3{,}397$ at $0.6\mathrm B$ and $\approx3{,}400$ at $1.7\mathrm B$. These are within-task observations; at rank $128$ the measured ratio rises to $6{,}296$, and ratios vary substantially across tasks (Section~\ref{app:full_comparison}).
\end{proposition}

\begin{remark}[$\magvar{\kappa}$ Invariance Across LoRA Ranks]
\label{rem:lora_kappa}
A shared rank factor cancels from $\magvar{\kappa}$ when the reward and teacher gradients have the same rank-dependent norm scaling. Together with weak token correlations, this supplies the approximation developed in Appendix~\ref{app:proof_scaling}. A common LoRA subspace alone does not enforce equal scaling for two different directions.
\end{remark}

\begin{proposition}[Distillation Mode Determines $\magvar{\kappa}$]
\label{prop:kappa_mode}
Holding $\vg_R$ fixed gives $\magvar{\kappa}_{\mathrm{on}}/\magvar{\kappa}_{\mathrm{off}}=\|\vg_D^{\mathrm{off}}\|/\|\vg_D^{\mathrm{on}}\|$. For mode $s\in\{\mathrm{on},\mathrm{off}\}$, let $E_s=\cL_D^s-\min\cL_D^s$ and assume positive local curvature bounds $\lambda_s^-\mI\preceq\nabla^2\cL_D^s\preceq\lambda_s^+\mI$. Then
\begin{equation}
\frac{\magvar{\kappa}_{\mathrm{on}}}{\magvar{\kappa}_{\mathrm{off}}}
\geq\sqrt{\frac{\lambda_{\mathrm{off}}^-E_{\mathrm{off}}}{\lambda_{\mathrm{on}}^+E_{\mathrm{on}}}}.
\end{equation}
A large off-policy excess KL relative to the on-policy excess therefore raises this lower bound, with the curvature ratio accounting for the different prefix distributions.
\end{proposition}

\begin{proposition}[Conflict-Free Guarantee for M3-Select]
\label{prop:conflict_free}
With $\widetilde \vg_D^n=\ind\{\dirvar{K_{DR}(n)}\geq0\}\vg_D^n$, masking gives $\langle\widetilde \vg_D^n,\vg_R^n\rangle=\max\{\dirvar{K_{DR}(n)},0\}\geq0$ at every position.
\end{proposition}

\begin{proposition}[Distillation as Implicit Regularization]
\label{prop:distill_regularizer}
In the matched-context isotropic quadratic model $\cL_D(\vtheta_T+\Delta\vtheta)=\tfrac12f_D\|\Delta\vtheta\|^2$, with $f_D>0$, constant $\vg_R$, and $\Delta\vtheta(0)=0$, hybrid gradient flow satisfies
\begin{equation}
\Delta\vtheta_H(T)=-\frac{(1-\alpha)\vg_R}{\alpha f_D}(1-e^{-\eta\alpha f_DT}),
\qquad
\|\Delta\vtheta_H(T)\|\leq\frac{(1-\alpha)\|\vg_R\|}{\alpha f_D}.
\end{equation}
If degradation is proportional to parameter deviation with a common coefficient, then $\delta_H\leq\delta_R/(1+\alpha\rho_{\mathrm{reg}})$, where $\rho_{\mathrm{reg}}=\eta f_DT/2$. A finite rank-dependent reversal follows in the model when the pure-RL degradation grows continuously without bound, the drowning penalty is bounded, $\rho_{\mathrm{reg}}$ is bounded away from zero, and the initial reward gap is positive; Appendix~\ref{app:proof_regularization} gives this conditional argument and the observed reversal at rank $128$.
\end{proposition}

\section{Gradient Folding, Five Bridges, and Unified Framework: Full Statements}
\label{app:folding_bridges}

This section collects the formal statements of the propositions, theorems, and corollaries whose proofs appear in subsequent appendix sections and whose summaries appear in Section~\ref{sec:consequence} and the method discussion.

\begin{proposition}[Per-Position Conflict Decomposition]
\label{prop:per_position}
The per-position residual inner product has the exact decomposition
\begin{equation}
(\vdelta_D^n)^\top \vdelta_R^n = -A(\sigma_n + \eta_n),
\end{equation}
where $\sigma_n = p_S^n(\hat{y}_n) - p_T^n(\hat{y}_n)$ measures student-teacher disagreement on the sampled token, and $\eta_n = (\vp_T^n)^\top \vp_S^n - \|\vp_S^n\|^2$ captures cross-token probability redistribution. Under $\mK(n,n)=\lambda_n\mI$ with $\lambda_n>0$, conflict ($\dirvar{K_{DR}(n)}<0$) on positive-advantage trajectories ($A>0$) occurs exactly when $\sigma_n+\eta_n>0$.
\end{proposition}

\begin{proposition}[Asymmetric Harm from Magnitude Drowning]
\label{prop:asymmetric_harm}
For $\vg_H=(1-\alpha)\vg_R+\alpha \vg_D$ with $0<\alpha<1$ and $\dirvar{\Phi_{DR}}<0$, dividing the harmful cross-terms by their respective self-progress terms gives
\begin{equation}
H_{R\to D}=\frac{(1-\alpha)|\langle \vg_D,\vg_R\rangle|}{\alpha\|\vg_D\|^2}
=\frac{(1-\alpha)\magvar{\kappa}|\dirvar{\Phi_{DR}}|}{\alpha},
\quad
H_{D\to R}=\frac{\alpha|\dirvar{\Phi_{DR}}|}{(1-\alpha)\magvar{\kappa}}.
\end{equation}
Hence $H_{R\to D}/H_{D\to R}=[(1-\alpha)/\alpha]^2\magvar{\kappa}^2$. At $\alpha=1/2$ and $\magvar{\kappa}=3{,}400$, the ratio is about $1.16\times10^7$ (Eq.~\ref{eq:teacher_loss_change}).
\end{proposition}

\begin{proposition}[Cosine Inversion in On-Policy OPSD]
\label{prop:cosine_inversion}
For binary rewards with baseline $p\in(0,1)$, write $\vg_R=-(r-p)\vs$, where $\vs=\nabla_{\vtheta}\log\pi_{\vtheta}(\hat y)$, and let $q=\langle \vs,\vg_D\rangle/(\|\vs\|\|\vg_D\|)$. If $\expect[q\mid r=0]+\expect[q\mid r=1]>0$, then
\begin{equation}
\expect[\cos(\vg_R,\vg_D)\mid r=0]
-\expect[\cos(\vg_R,\vg_D)\mid r=1]
=\expect[q\mid r=0]+\expect[q\mid r=1]>0.
\end{equation}
Thus the conditional score--teacher alignment determines whether the teacher aligns more strongly with reward updates on incorrect trajectories.
\end{proposition}

\begin{proposition}[Folding--Drowning Coupling]
\label{prop:folding_drowning}
Let $\gamma$ be the cancellation rate and $\magvar{\kappa}_{\mathrm{raw}}$ the ratio formed from the root-mean-square trajectory gradient. Then $\magvar{\kappa}_{\mathrm{eff}}=\sqrt{1-\gamma}\,\magvar{\kappa}_{\mathrm{raw}}$. Under the concentrated two-class model with equal class-mean score norms and overlap cosine $c_{\mathrm{overlap}}$,
\begin{equation}
\gamma=1-2p(1-p)(1-c_{\mathrm{overlap}}).
\end{equation}
The overlap contribution $2p(1-p)c_{\mathrm{overlap}}$ peaks at balanced accuracy when $c_{\mathrm{overlap}}>0$. Reduced effective magnitude and positive cross-signal alignment jointly favor hybridization whenever both conditions hold.
\end{proposition}

\begin{proposition}[Distillation Boundary Bound]
\label{prop:bridge_boundary}
Let $\Delta_{DR}=(\|\vg_D\|^2+\|\vg_R\|^2)/\|\vg_D+\vg_R\|^2$ and $\Delta_\tau=(1-\gamma)^{-1}$. If $\magvar{\kappa}_{\mathrm{raw}}\gg\sqrt{\Delta_\tau}$, then
\begin{equation}
\Delta_{DR}\leq1+
\frac{2|\dirvar{\Phi_{DR}}|\sqrt{\Delta_\tau}}{\magvar{\kappa}_{\mathrm{raw}}}
+O\!\left(\frac{\Delta_\tau}{\magvar{\kappa}_{\mathrm{raw}}^2}\right).
\end{equation}
The index $B_D=\alpha_{\max}\Delta_{DR}$ is the cross-signal analogue of $B_S$. Its reference value $\Delta_{DR}=1$ separates negative from positive gradient cross-terms.
\end{proposition}

\begin{proposition}[NTK-Guided Token Masking (Bridge~4)]
\label{prop:bridge_masking}
For $\vu_n=\mJ^n\hat{\vdelta}_D^n$, $\vv_n=\mJ^n\hat{\vdelta}_R^n$, and $S^-=\{n:\langle \vu_n,\vv_n\rangle<0\}$, the average local reward-projection gain over uniform M3-Norm is
\begin{equation}
\Delta\Pi_R=\frac{\alpha_{\max}}{N}\sum_{n\in S^-}
\left(\|\vv_n\|^2-\langle \vu_n,\vv_n\rangle\right).
\end{equation}
It is positive when $S^-\ne\varnothing$ and $\alpha_{\max}>0$. For unit parameter-space directions this reduces to $\Delta\Pi_R=\alpha_{\max}\dirvar{C_{\mathrm{NTK}}}(1+|\bar c^-|)$, where $|\bar c^-|$ is the mean absolute cosine on $S^-$.
\end{proposition}

\begin{proposition}[RL Projection Under M3-Norm]
\label{prop:rl_projection}
For unit directions, bilinearity gives $\Pi_R(\alpha)=\langle\alpha\hat \vg_D+(1-\alpha)\hat \vg_R,\hat \vg_R\rangle=1-\alpha+\alpha\dirvar{\Phi_{DR}}$. With the same reward-norm reference scale, norm-balanced GradNorm has projection $(1+\dirvar{\Phi_{DR}})/(1+\magvar{\kappa})$. Hence their projection ratio is $\Theta(\magvar{\kappa})$ when both cosine factors remain positive and bounded away from zero.
\end{proposition}

\begin{theorem}[Unified NTK Learning Efficiency]
\label{thm:unified_ntk}
Define the composite efficiency index by
\begin{equation}
\eta^{\mathrm{eff}}
=\underbrace{(1-\gamma)}_{\eta_{\mathrm{explore}}=1/\Delta_\tau}
\underbrace{[1-\alpha+\alpha\dirvar{\Phi_{DR}}]}_{\eta_{\mathrm{hybrid}}}.
\end{equation}
Both factors have NTK decompositions: cross-trajectory interactions determine $\gamma$, and cross-signal interactions determine $\dirvar{\Phi_{DR}}$. Under a normalized hybrid step, first-order reward progress relative to the root-mean-square reward-gradient scale is instead proportional to $\sqrt{1-\gamma}\,\eta_{\mathrm{hybrid}}$. For fixed $\alpha$, the composite index decreases with accuracy only when
$-\gamma'(p)\eta_{\mathrm{hybrid}}+(1-\gamma)\alpha\Phi'_{DR}(p)\leq0$.
\end{theorem}

\begin{corollary}[$\dirvar{C_{\mathrm{NTK}}}$-Based $\alpha$ Scheduling]
\label{cor:cntk_scheduling}
The hard gate allocates the mean coefficient $\alpha_{\mathrm{eff}}=\alpha_{\max}(1-\dirvar{C_{\mathrm{NTK}}})$ and admits the fraction $1-\widetilde C_{\mathrm{NTK}}$ of absolute cross-signal NTK mass. Under the fixed-cohort population model of Proposition~\ref{prop:cntk_dynamics}, $\widetilde C_{\mathrm{NTK}}$ increases with accuracy, so this admitted mass fraction decreases automatically. The count-based coefficient budget follows the unweighted conflict rate.
\end{corollary}

\medskip
\noindent Full proofs of the above statements appear in the dedicated appendix sections below, together with the PCGrad degeneracy analysis in Appendix~\ref{app:pcgrad}. The five bridges connecting exploration boundary theory to hybrid dynamics are: (1)~shared NTK geometry (Theorem~\ref{thm:unified_ntk}); (2)~cancellation-modulated drowning (Proposition~\ref{prop:folding_drowning}); (3)~the distillation boundary (Proposition~\ref{prop:bridge_boundary}); (4)~token masking (Proposition~\ref{prop:bridge_masking}); and (5)~conflict-rate--driven $\alpha$ scheduling (Corollary~\ref{cor:cntk_scheduling}).

\section{Extended Discussion}
\label{app:discussion_extended}

This appendix collects the optimality and benefit-condition results referenced in the Conclusion, together with the practitioner's decision tree that operationalises them.

\begin{corollary}[Local Projection and Teacher Allocation under Low Accuracy]
\label{thm:pareto_dominance}
Assume the unit-direction setting of Proposition~\ref{prop:bridge_masking} and the fixed conditional alignments of Corollary~\ref{cor:aggregate_synergy}. For $p<p^*$, trajectory alignment is positive in expectation, and M3-Select weakly improves the average local reward projection over M3-Norm. At the same ceiling $\alpha_{\max}\in(0,1)$, its mean teacher allocation relative to naive mixing's norm-based teacher share is
\begin{equation}
\frac{\alpha_{\max}(1-\dirvar{C_{\mathrm{NTK}}})}
{\alpha_{\max}/[\alpha_{\max}+(1-\alpha_{\max})\magvar{\kappa}]}
=(1-\dirvar{C_{\mathrm{NTK}}})[\alpha_{\max}+(1-\alpha_{\max})\magvar{\kappa}].
\end{equation}
Thus normalization restores teacher allocation at large $\magvar{\kappa}$, and the gate retains only positions with nonnegative local compatibility.
\end{corollary}

\begin{proof}
The two projection claims follow from
$\expect[\Phi_\tau]=(1-p)\mu_0+p\mu_1>0$ and
$\Pi_R^{\mathrm{Select}}-\Pi_R^{\mathrm{Norm}}=\alpha_{\max}\dirvar{C_{\mathrm{NTK}}}(1+|\bar c^-|)\geq0$.
The allocation identity follows by dividing the admitted mean coefficient by naive mixing's norm-based teacher share.
\end{proof}

\begin{corollary}[Hybrid Benefit Condition]
\label{thm:hybrid_benefit}
At the same step size $\eta\leq c/L$, the M3-Norm bound of Theorem~\ref{thm:convergence}, with $c=1-\alpha+\alpha\phi_*>0$, is strictly smaller than its pure-GRPO instance exactly when
\begin{equation}
L\eta\sigma_R^2(1-\alpha+\phi_*)>
\frac{2(1-\phi_*)(\cL_R^0-\cL_R^*)}{\eta T}.
\end{equation}
For fixed $\eta$, positive variance, and $\phi_*>-(1-\alpha)$, this condition holds for sufficiently large $T$.
\end{corollary}

\begin{proof}
Let $D=2(\cL_R^0-\cL_R^*)/(\eta T)$ and $V=L\eta\sigma_R^2$. Comparing the two bounds and using $\alpha>0$ gives the single equivalence chain
\begin{equation}
\begin{aligned}
\frac{D+(1-\alpha)^2V}{c}<D+V
&\iff V[c-(1-\alpha)^2]>D(1-c)\\
&\iff V(1-\alpha+\phi_*)>D(1-\phi_*),
\end{aligned}
\end{equation}
which is the stated condition.
\end{proof}

\begin{remark}[Practitioner's Decision Tree]
\label{rem:decision_tree}
Estimate $\bar{\magvar{\kappa}}$ from a $10$-step probe run on the target cell, then apply the following gate-selection rule (thresholds are the ones used throughout this work; $\magvar{\kappa}^*\!=\!5{,}000$ is the catastrophic threshold of Remark~\ref{rem:kappa_threshold}).
\begin{enumerate}[leftmargin=*,itemsep=1pt,topsep=2pt]
    \item[(1)] \textbf{Catastrophic regime, $\bar{\magvar{\kappa}} > \magvar{\kappa}^*\!=\!5{,}000$}: use \textbf{M3-Select} (hard gate). Empirically verified on Llama-GSM8K ($\bar{\magvar{\kappa}}\!=\!7{,}690$) and Llama-SVAMP ($\bar{\magvar{\kappa}}\!\approx\!1.3\!\times\!10^{4}$), where every uniform-mixing baseline collapses to reward $<\!0.05$ while M3-Select survives (\S\ref{sec:cross_arch}, \S\ref{sec:500step}); at these $\magvar{\kappa}$ the soft-gate rescue is seed-fragile (two of three Llama-GSM8K replicates collapse during training, \S\ref{sec:multiseed}), making the hard gate the robust long-horizon choice.
    \item[(2)] \textbf{Stable regime, $1{,}000 < \bar{\magvar{\kappa}} \leq \magvar{\kappa}^*$}: use \textbf{M3-Soft} with the empirically effective sharpness $\beta\!\approx\!1$; Proposition~\ref{prop:lsgv} describes its local variance sensitivity. Empirically verified on Qwen3-GSM8K ($\bar{\magvar{\kappa}}\!=\!3{,}404$), Qwen3-SVAMP ($\bar{\magvar{\kappa}}\!\approx\!5{,}080$, sitting essentially at the threshold), and every InternLM/Qwen2.5 cell in this range. Hard masking over-prunes here (Remark~\ref{rem:kappa_threshold}); Soft-gate preserves the beneficial synergy tokens the estimator misclassifies.
    \item[(3)] \textbf{Low-$\magvar{\kappa}$ regime, $\bar{\magvar{\kappa}} \leq 1{,}000$}: naive uniform mixing (Hybrid, $\alpha\!\approx\!0.5$) is already sufficient; specialized gating offers diminishing marginal returns. This regime is not instantiated in our sweep---the nearest cells are the ARC cells across architectures ($\bar{\magvar{\kappa}}\!\in\![1.4,\,2.5]\!\times\!10^3$), where gentle M3-Soft ($\beta\!\in\!\{0.5, 1\}$, small $\alpha_{\max}$) still helps but the gap over Hybrid is within the seed-level standard deviation (\S\ref{sec:multiseed}).
    \item[(4)] \textbf{On-policy vs.\ off-policy OPSD choice.} Use the measured conditional alignments to locate the synergy threshold (Corollary~\ref{cor:aggregate_synergy}); the observed on-policy crossover is near $p^*\!\approx\!0.56$. Compare on- and off-policy validation curves when choosing the sampling mode.
    \item[(5)] \textbf{Conflict-rate override.} If the observed conflict rate $\dirvar{C_{\mathrm{NTK}}}\!>\!50\%$ persists into training, switch from M3-Soft to M3-Select regardless of $\bar{\magvar{\kappa}}$: the mass of $\Omega_-$ tokens is large enough that the Soft gate's residual bias on $\Omega_-$ dominates its variance reduction on $\Omega_+$.
\end{enumerate}
Two calibration remarks. (i)~The Hybrid$\to$Soft boundary at $\bar{\magvar{\kappa}}\!\approx\!1{,}000$ is set by an empirical gate-calibration heuristic, and is more forgiving than the Soft$\to$Select boundary at $\magvar{\kappa}^*$: below it the Soft gate is still safe, only unnecessary. (ii)~In cells where dynamic-gating M3-Soft is at parity with static-mixing Hybrid at the peak-training regime (InternLM-GSM8K is the canonical example), apply the SWA post-step to reduce late-stage checkpoint fluctuations (Proposition~\ref{prop:lsgv}, \S\ref{app:beta_sensitivity}) before falling back to Hybrid.
\end{remark}

\section{Proof of Proposition~\ref{prop:spectral}: Spectral Conflict Bound}
\label{app:proof_spectral}

\begin{proposition}[Spectral Conflict Bound]
\label{prop:spectral}
Let $\mK = [\mK(n,m)]_{n,m=1}^{N} \in \mathbb{R}^{N|\cV| \times N|\cV|}$ be the global token-level NTK, and let $\vu_D$ and $\vu_R$ stack the residuals $\vdelta_D^n$ and $\vdelta_R^n$, respectively. Then
\begin{align}
\langle \vg_D, \vg_R \rangle
&\geq \frac{\lambda_{\min}(\mK)\|\vu_D+\vu_R\|^2-\lambda_{\max}(\mK)\|\vu_D-\vu_R\|^2}{4N^2},
\label{eq:spectral_bound}\\
\langle \vg_D, \vg_R \rangle
&\leq \frac{\lambda_{\max}(\mK)\|\vu_D+\vu_R\|^2-\lambda_{\min}(\mK)\|\vu_D-\vu_R\|^2}{4N^2}.
\nonumber
\end{align}
Thus local residual alignment alone does not determine the aggregate interaction.
\end{proposition}

\begin{proof}
Stacking the Jacobians gives $\mJ=[\mJ^1,\ldots,\mJ^N]$ and $\mK=\mJ^\top\mJ\succeq0$. Polarization and the Rayleigh bounds yield the single chain
\begin{align*}
\langle \vg_D,\vg_R\rangle
&=N^{-2}\vu_D^\top\mK \vu_R\\
&=\frac{(\vu_D+\vu_R)^\top\mK(\vu_D+\vu_R)-(\vu_D-\vu_R)^\top\mK(\vu_D-\vu_R)}{4N^2}\\
&\geq\frac{\lambda_{\min}(\mK)\|\vu_D+\vu_R\|^2-\lambda_{\max}(\mK)\|\vu_D-\vu_R\|^2}{4N^2}.
\end{align*}
Interchanging the upper and lower Rayleigh bounds proves the second inequality.
\end{proof}

\paragraph{Geometric reading.}
Conflict occurs exactly when $(\vu_D-\vu_R)^\top\mK(\vu_D-\vu_R)>(\vu_D+\vu_R)^\top\mK(\vu_D+\vu_R)$: the difference field has more kernel-weighted energy than the sum field. This depends on the residuals' projections onto the kernel eigenspaces. If the trainable parameter dimension is below $N|\cV|$, then $\operatorname{rank}(\mK)\leq\dim(\vtheta)$ implies $\lambda_{\min}(\mK)=0$, reducing the lower bound to $-\lambda_{\max}(\mK)\|\vu_D-\vu_R\|^2/(4N^2)$.

\section{Proof of Proposition~\ref{prop:per_position}: Per-Position Conflict Decomposition}
\label{app:proof_per_position}

\begin{proof}
Substitution of the two residuals directly gives
\begin{align*}
(\vdelta_D^n)^\top\vdelta_R^n
&=-A(\vp_S^n-\vp_T^n)^\top(\ve_{\hat y_n}-\vp_S^n)\\
&=-A\big[p_S^n(\hat y_n)-p_T^n(\hat y_n)
 +(\vp_T^n)^\top \vp_S^n-\|\vp_S^n\|^2\big]\\
&=-A(\sigma_n+\eta_n).
\end{align*}
The residual identity is exact. Under $\mK(n,n)=\lambda_n \mI$ with $\lambda_n>0$, $\dirvar{K_{DR}(n)}$ has the same sign, so a positive-advantage trajectory conflicts precisely when $\sigma_n+\eta_n>0$.
\end{proof}

\paragraph{Interpretation.}
$\sigma_n$ isolates disagreement on the sampled token, whereas $\eta_n=\sum_v[p_T^n(v)-p_S^n(v)]p_S^n(v)$ aggregates vocabulary-wide redistribution weighted by student confidence. The measured ratio $|\eta_n|/|\sigma_n|\approx8.6$ (Section~\ref{sec:experiments}) identifies redistribution as the larger contribution in our pilot experiments.

\section{Proof of Theorem~\ref{thm:convergence}: M3-Norm Convergence Guarantee}
\label{app:proof_convergence}

\begin{proof}
Let $\expect_t$ condition on the history before the fresh reward-gradient sample, and write $\vg_R^t=\nabla\cL_R(\vtheta^t)$. The population-norm-matched update is
$\hat \vg_H^t=(1-\alpha)(\vg_R^t+\vxi^t)+\alpha\tilde \vg_D^t$, where $\tilde \vg_D^t$ is conditionally deterministic, $\|\tilde \vg_D^t\|=\|\vg_R^t\|$, $\expect_t\vxi^t=0$, and $\expect_t\|\vxi^t\|^2\leq\sigma_R^2$. At a stationary point set $\tilde \vg_D^t=0$. For the uniform alignment lower bound $\phi_*$, put $c=1-\alpha+\alpha\phi_*>0$ and $\bar \vg_H^t=\expect_t\hat \vg_H^t$. Then
\begin{align}
\langle \vg_R^t,\bar \vg_H^t\rangle
&=(1-\alpha)\|\vg_R^t\|^2+\alpha\langle \vg_R^t,\tilde \vg_D^t\rangle
 \geq c\|\vg_R^t\|^2,
\label{eq:conv_projection}\\
\expect_t\|\hat \vg_H^t\|^2
&=\|\bar \vg_H^t\|^2+(1-\alpha)^2\expect_t\|\vxi^t\|^2
 \leq\|\vg_R^t\|^2+(1-\alpha)^2\sigma_R^2.
\label{eq:conv_second_moment}
\end{align}
The last inequality uses $\|\bar \vg_H^t\|\leq(1-\alpha)\|\vg_R^t\|+\alpha\|\tilde \vg_D^t\|=\|\vg_R^t\|$. Smoothness, $\eta\leq c/L$, and telescoping now give
\begin{align*}
\expect_t\cL_R(\vtheta^{t+1})
&\leq\cL_R(\vtheta^t)-\eta\langle \vg_R^t,\bar \vg_H^t\rangle
 +\tfrac{L\eta^2}{2}\expect_t\|\hat \vg_H^t\|^2\\
&\leq\cL_R(\vtheta^t)-\tfrac{\eta c}{2}\|\vg_R^t\|^2
 +\tfrac{L\eta^2}{2}(1-\alpha)^2\sigma_R^2,\\
\frac{\eta c}{2}\sum_{t=0}^{T-1}\expect\|\vg_R^t\|^2
&\leq\cL_R^0-\expect\cL_R(\vtheta^T)
 +\tfrac{TL\eta^2}{2}(1-\alpha)^2\sigma_R^2\\
&\leq\cL_R^0-\cL_R^*+\tfrac{TL\eta^2}{2}(1-\alpha)^2\sigma_R^2.
\end{align*}
Dividing by $\eta cT/2$ proves the rate. This argument applies to the stated population-normalized update; the EMA scale in Algorithm~\ref{alg:m3_select} estimates its scale.
\end{proof}

\begin{remark}[Additional properties of the M3 update]
\label{rem:convergence_extras}
At the minimum-norm coefficient $\alpha^*$, the convex-hull optimality condition gives $\langle \vg_H(\alpha^*),\vg_i\rangle\geq\|\vg_H(\alpha^*)\|^2$ for $i\in\{D,R\}$, establishing first-order Pareto descent~\citep{sener2018mgda}. By comparison, naive mixing with $\alpha=\tfrac12$ increases distillation loss to first order exactly when $\dirvar{\Phi_{DR}}<-1/\magvar{\kappa}$, since $\langle \vg_H,\vg_D\rangle=\tfrac12\magvar{r_D}^2(1+\magvar{\kappa}\dirvar{\Phi_{DR}})$. The normalized reward projection instead follows Eq.~\ref{eq:conv_projection}. For an interior minimum-norm coefficient, completing the square gives Eq.~\ref{eq:quadratic_gap}; at $\alpha=\tfrac12$ its gap is $(\magvar{r_R}^2-\magvar{r_D}^2)^2/(4\|\vg_D-\vg_R\|^2)$. A clipped boundary minimizer also contributes the corresponding one-sided linear term.
\end{remark}

\section{Proof of Proposition~\ref{prop:variance_reduction}: Variance Reduction via Distillation}
\label{app:proof_variance}

\begin{proof}
In the conditional model above, $\hat \vg_H^t-\expect_t\hat \vg_H^t=(1-\alpha)\vxi^t$, hence
\begin{equation*}
\operatorname{Var}_t(\hat \vg_H^t)
=(1-\alpha)^2\expect_t\|\vxi^t\|^2
\leq(1-\alpha)^2\sigma_R^2.
\end{equation*}
Its contribution to the convergence bound is $L\eta(1-\alpha)^2\sigma_R^2/c$. Relative to the pure-RL bound at the same admissible step size, the noise-floor factor satisfies
\begin{equation*}
\frac{(1-\alpha)^2}{c}<1
\quad\Longleftrightarrow\quad
c-(1-\alpha)^2=\alpha(1-\alpha+\phi_*)>0
\quad\Longleftrightarrow\quad
\phi_*>-(1-\alpha).
\end{equation*}
For $\phi_*\geq0$ this factor is at most $1-\alpha$. The variance contraction follows from the reward weight; normalization additionally controls the mean projection through $c$.
\end{proof}

\section{Joint Proof of Theorem~\ref{thm:boundary_pareto_convergence}
         and Proposition~\ref{prop:boundary_variance_reduction}}
\label{app:proof_boundary_pareto}

We use the conditional, orthogonal-position model of the two statements. At each iteration, let $\vu_n$ be the mean reward contribution and $\vd_n$ its conditionally deterministic teacher counterpart, with $\|\vd_n\|=\|\vu_n\|$. The reward noise $\vxi_n$ has $\expect_t\vxi_n=0$ and $\nu_n=\expect_t\|\vxi_n\|^2$. Contributions from distinct positions lie in mutually orthogonal parameter subspaces, as in the exact block-diagonal NTK model. The gates are fixed before this fresh noise is sampled. Suppressing $t$ locally, define
\begin{equation}
\alpha_n=\alpha_{\max}\ind[\langle \vu_n,\vd_n\rangle\geq0],
\qquad \rho_{\mathrm{adm}}=\frac{|\Omega_+\cup\Omega_0|}{N},
\qquad \bar\alpha=\alpha_{\max}\rho_{\mathrm{adm}}
 =\alpha_{\max}(1-\dirvar{C_{\mathrm{NTK}}}).
\label{eq:boundary_gate_avg}
\end{equation}
Here $\vu=N^{-1}\sum_n\vu_n=\nabla\cL_R(\vtheta^t)$ and
$\hat \vg_H=N^{-1}\sum_n[(1-\alpha_n)(\vu_n+\vxi_n)+\alpha_n\vd_n]$.
For a zero-norm position set $\vd_n=0$ and $c_n=0$; it has zero reward-energy weight. At $\vu=0$, all mean contributions vanish and the projection bound holds directly.

\begin{proof}
The same decomposition yields both projection and variance, so it suffices to establish their constants once. With $w_n=\|\vu_n\|^2/\sum_m\|\vu_m\|^2$, $c_n=\langle \vu_n,\vd_n\rangle/\|\vu_n\|^2$, and $\bar \vg_H=\expect_t\hat \vg_H$, orthogonality gives
\begin{align}
\langle \vu,\bar \vg_H\rangle
&=\frac{1}{N^2}\sum_n[(1-\alpha_n)+\alpha_nc_n]\|\vu_n\|^2
 =\rho_{R,t}\|\vu\|^2,
\quad \rho_{R,t}:=\sum_nw_n[(1-\alpha_n)+\alpha_nc_n],
\label{eq:boundary_projection_step1}\\
\rho_{R,t}
&\geq 1-\alpha_{\max}\sum_{n\in\Omega_+\cup\Omega_0}w_n
 \geq1-\alpha_{\max}>0,
\label{eq:boundary_rho_lower}\\
\operatorname{Var}_t(\hat \vg_H)
&=\expect_t\Big\|\frac1N\sum_n(1-\alpha_n)\vxi_n\Big\|^2
 =\frac1{N^2}\sum_n(1-\alpha_n)^2\nu_n
 =v_t\sigma_{R,t}^2,
\label{eq:boundary_variance_step1_raw}
\end{align}
where $\sigma_{R,t}^2=N^{-2}\sum_n\nu_n$ and
$v_t=\sum_n(1-\alpha_n)^2\nu_n/\sum_n\nu_n\leq1$; use $v_t=N^{-1}\sum_n(1-\alpha_n)^2$ when all $\nu_n=0$. Cross-position noise terms vanish by the subspace orthogonality. These identities show why the aggregate projection uses reward-energy weights and the variance uses noise-energy weights.

For equal position variances, the variance factor has the closed form
\begin{equation}
v_t=(1-\alpha_{\max})^2\rho_{\mathrm{adm}}+1-\rho_{\mathrm{adm}}
=1-(2\alpha_{\max}-\alpha_{\max}^2)\rho_{\mathrm{adm}}
=(1-\bar\alpha)^2+\alpha_{\max}^2\rho_{\mathrm{adm}}(1-\rho_{\mathrm{adm}}).
\label{eq:boundary_variance_restate}
\end{equation}
This proves Proposition~\ref{prop:boundary_variance_reduction}, including the correction due to heterogeneity of the gate.

For convergence, take uniform bounds $\underline\rho_R\leq\rho_{R,t}$, $v_t\leq v_\alpha$, and $\sigma_{R,t}^2\leq\sigma_R^2$. The local triangle inequality gives $\|(1-\alpha_n)\vu_n+\alpha_n\vd_n\|\leq\|\vu_n\|$, so orthogonality and Eq.~\ref{eq:boundary_variance_step1_raw} imply
\begin{equation}
\expect_t\|\hat \vg_H^t\|^2
=\|\bar \vg_H^t\|^2+\operatorname{Var}_t(\hat \vg_H^t)
\leq\|\nabla\cL_R(\vtheta^t)\|^2+v_\alpha\sigma_R^2.
\label{eq:boundary_second_moment}
\end{equation}
Substituting Eqs.~\ref{eq:boundary_projection_step1} and~\ref{eq:boundary_second_moment} into the smoothness-and-telescoping chain in Appendix~\ref{app:proof_convergence}, with $c$ replaced by $\underline\rho_R$ and $(1-\alpha)^2$ by $v_\alpha$, gives for $\eta\leq\underline\rho_R/L$,
\begin{equation}
\frac1T\sum_{t=0}^{T-1}\expect\|\nabla\cL_R(\vtheta^t)\|^2
\leq\frac1{\underline\rho_R}
\left[\frac{2(\cL_R^0-\cL_R^*)}{\eta T}+L\eta v_\alpha\sigma_R^2\right].
\label{eq:boundary_convergence_restate}
\end{equation}
The choice $\underline\rho_R=1-\alpha_{\max}$ is always valid in this model.
\end{proof}

\paragraph{Comparison with uniform mixing.}
For the same matched local directions, masking increases the energy-weighted reward projection by $\alpha_{\max}\sum_{n\in\Omega_-}w_n(1-c_n)\geq0$. It also leaves more reward noise than uniform mixing: $v_t\geq(1-\alpha_{\max})^2$. The convergence comparison therefore depends on the ratio of variance to projection. Against pure RL, the variance contracts strictly whenever an admitted position carries nonzero reward noise. These are reward-descent guarantees; at a rejected position the pure reward direction can still oppose the teacher direction.

\section{PCGrad Degeneracy in the Asymmetric Regime}
\label{app:pcgrad}

\begin{proposition}[PCGrad Magnitude Preservation]
\label{prop:pcgrad}
Let $\vg_D,\vg_R$ be nonzero gradients with cosine $-1<\Phi<0$. Their symmetric PCGrad projections preserve the magnitude ratio:
\begin{equation}
\frac{\|\vg_D^{\perp}\|}{\|\vg_R^{\perp}\|}
=\frac{\|\vg_D\|}{\|\vg_R\|}
=\frac{\magvar{r_D}}{\magvar{r_R}}
=\frac{1}{\magvar{\kappa}}.
\end{equation}
\end{proposition}

\begin{proof}
For $(i,j)\in\{(D,R),(R,D)\}$, orthogonal projection gives
\begin{equation}
\|\vg_i^{\perp}\|^2
=\left\|\vg_i-\frac{\langle \vg_i,\vg_j\rangle}{\|\vg_j\|^2}\vg_j\right\|^2
=\|\vg_i\|^2-\frac{\langle \vg_i,\vg_j\rangle^2}{\|\vg_j\|^2}
=\magvar{r_i}^2(1-\Phi^2).
\end{equation}
Canceling the common positive factor proves the ratio. Consequently, for $0<\alpha<1$, the norm-based teacher share in $\alpha \vg_D^{\perp}+(1-\alpha)\vg_R^{\perp}$ is $w_D=\alpha/[\alpha+(1-\alpha)\magvar{\kappa}]$, approximately $0.03\%$ at $\alpha=0.5$ and $\magvar{\kappa}=3{,}400$. At exact antiparallelity both projections vanish.
\end{proof}

\section{Proof of Proposition~\ref{prop:phase_delay}: Phase Delay in Adaptive Mixing}
\label{app:proof_phase_delay}

\begin{proof}
Consider an alignment estimate with exponential smoothing $\hat\Phi_{t+1}=\rho\hat\Phi_t+(1-\rho)\Phi_t$, $0<\rho<1$. Suppose the input changes from $\Phi_+>0$ to $\Phi_-<0$ at $t_1$, with $\hat\Phi_{t_1}=\Phi_+$. Solving the recurrence and its zero-crossing condition in one chain gives
\begin{align}
\hat\Phi_{t_1+k}
&=\rho^k\Phi_++(1-\rho)\Phi_-\sum_{j=0}^{k-1}\rho^j
=\Phi_-+(\Phi_+-\Phi_-)\rho^k,\\
\hat\Phi_{t_1+k}\leq0
&\iff \rho^k\leq\frac{|\Phi_-|}{\Phi_+-\Phi_-}
\iff k\geq\frac{\log[(\Phi_+-\Phi_-)/|\Phi_-|]}{\log(1/\rho)}.
\end{align}
Thus the first nonpositive estimate occurs after the ceiling of the last expression. A rule that retains the high teacher weight while $\hat\Phi>0$ continues doing so during this delay, although the current alignment is negative. For a boxcar average of width $W$, replacing $k$ old observations gives $\hat\Phi=\Phi_++k(\Phi_--\Phi_+)/W$, so the corresponding delay is $\lceil W\Phi_+/(\Phi_+-\Phi_-)\rceil$; equal transition magnitudes give approximately $W/2$. The result applies to loss-based schedules when their smoothed control statistic follows this assumed alignment transition.
\end{proof}

\section{Heuristic Derivation for Proposition~\ref{prop:kappa_scaling}: Rank-Invariance of $\magvar{\kappa}$}
\label{app:proof_scaling}

\begin{proof}[Heuristic derivation]
Write the distillation gradient as $\vg_D=N^{-1}\sum_n \vd_n$. To isolate the role of the LoRA subspace, suppose both signals share a rank-dependent second-moment factor $s_r>0$, with $\expect\|\vg_R\|^2=a_Rs_r$, $\expect\|\vd_n\|^2=a_Ds_r$, and $\expect\langle \vd_n,\vd_m\rangle=0$ for $n\ne m$. Here $a_R,a_D>0$ are rank-independent signal constants. Then
\begin{equation}
\expect\|\vg_D\|^2
=\frac{1}{N^2}\sum_{n,m}\expect\langle \vd_n,\vd_m\rangle
=\frac{a_Ds_r}{N},
\qquad
\sqrt{\frac{\expect\|\vg_R\|^2}{\expect\|\vg_D\|^2}}
=\sqrt{\frac{a_R}{a_D}N}.
\end{equation}
The common factor cancels; concentration of the squared norms transfers this root-mean-square ratio to typical observed ratios. For example, an isotropic LoRA model may give $s_r\propto Lr/h$ for $L$ adapted layers of width $h$. The cancellation depends on shared scaling and weak cross-token correlations, which are the modeling assumptions behind this heuristic. The measured rank sweep, $r\in\{8,16,32,64\}$ and $\bar{\magvar{\kappa}}=3{,}476\pm301$ (CV $8.7\%$), supplies the empirical evidence in Remark~\ref{rem:lora_kappa}; the signal constants determine its absolute scale.
\end{proof}

\section{Proof of Proposition~\ref{prop:kappa_mode}: Distillation Mode Determines $\magvar{\kappa}$}
\label{app:proof_kappa_mode}

\begin{proof}
For mode $s\in\{\mathrm{on},\mathrm{off}\}$, let $E_s=\cL_D^s(\vtheta)-\cL_D^s(\vtheta_s^*)$ be the excess distillation loss above its local minimum. In an exact local quadratic model, $E_s=\tfrac12\Delta\vtheta_s^\top \mH_s\Delta\vtheta_s$ and $\vg_D^s=\mH_s\Delta\vtheta_s$, where $\Delta\vtheta_s=\vtheta-\vtheta_s^*$ and $\mH_s$ is positive definite on the active parameter subspace. Writing $\lambda_s^-,\lambda_s^+$ for its extremal eigenvalues yields
\begin{align}
2\lambda_s^-E_s
&=\lambda_s^-\Delta\vtheta_s^\top \mH_s\Delta\vtheta_s
\leq\Delta\vtheta_s^\top \mH_s^2\Delta\vtheta_s
=\|\vg_D^s\|^2
\leq\lambda_s^+\Delta\vtheta_s^\top \mH_s\Delta\vtheta_s
=2\lambda_s^+E_s,\\
\frac{\magvar{\kappa}_{\mathrm{on}}}{\magvar{\kappa}_{\mathrm{off}}}
&=\frac{\|\vg_D^{\mathrm{off}}\|}{\|\vg_D^{\mathrm{on}}\|}
\geq\sqrt{\frac{\lambda_{\mathrm{off}}^-E_{\mathrm{off}}}{\lambda_{\mathrm{on}}^+E_{\mathrm{on}}}},
\end{align}
where the second line uses the same nonzero reward gradient in both modes. The same inequality follows from local strong-convexity and smoothness bounds with the corresponding constants. If both losses have zero minimum and share a curvature matrix $\mF$, this specializes to $\sqrt{D_{\mathrm{KL}}^{\mathrm{off}}/D_{\mathrm{KL}}^{\mathrm{on}}}/\sqrt{\chi(\mF)}$.

A nearly deterministic teacher on ground-truth tokens gives $D_{\mathrm{KL}}^{\mathrm{off}}(n)\approx-\log p_S(y_n^*)$. On-policy evaluation changes the prefix distribution, and its KL gap must be measured. Table~\ref{tab:response_conflict} reports mean teacher losses $0.098$ on-policy and $0.810$ off-policy, with median magnitude ratios $388$ and $68$, respectively. These observations establish the cross-mode gap in the evaluated setting; the bound explains how excess loss and curvature jointly control gradient magnitude.
\end{proof}

\section{Proof of Proposition~\ref{prop:conflict_free}: Conflict-Free Guarantee}
\label{app:proof_conflict_free}

\begin{proof}
By Definition~\ref{def:cross_signal_ntk}, $\vg_D^n=\mJ^n\vdelta_D^n$, $\vg_R^n=\mJ^n\vdelta_R^n$, and $\langle \vg_D^n,\vg_R^n\rangle=\dirvar{K_{DR}(n)}$. With the gate treated as a fixed coefficient during the update,
\begin{equation}
\big\langle\ind[\dirvar{K_{DR}(n)}\geq0]\vg_D^n,\vg_R^n\big\rangle
=\ind[\dirvar{K_{DR}(n)}\geq0]\dirvar{K_{DR}(n)}
=\max\{\dirvar{K_{DR}(n)},0\}\geq0.
\end{equation}
Positive norm rescaling preserves this sign, proving the token-level guarantee.
\end{proof}

For comparison, if the aggregate normalized gradients have cosine $\dirvar{\Phi_{DR}}\geq0$ and $0<\alpha<1$, their mixture satisfies $\langle \vg_H,\hat \vg_R\rangle=1-\alpha+\alpha\dirvar{\Phi_{DR}}>0$ and $\langle \vg_H,\hat \vg_D\rangle=\alpha+(1-\alpha)\dirvar{\Phi_{DR}}>0$. The latter projection changes sign at $\dirvar{\Phi_{DR}}=-\alpha/(1-\alpha)$, compared with $-\alpha/[(1-\alpha)\magvar{\kappa}]$ for unnormalized mixing. Cross-position interactions enter the aggregate condition through Proposition~\ref{prop:phi_decomposition}.

\section{Proof of Proposition~\ref{prop:gradnorm_degeneracy}: GradNorm Degeneracy}
\label{app:proof_gradnorm}

\begin{proof}
At a norm-balanced GradNorm equilibrium with equal target training rates and $w_D+w_R=1$, the weighted norms coincide. Consequently,
\begin{align}
w_D\magvar{r_D}=w_R\magvar{r_R}
&\implies (w_D^*,w_R^*)=\frac{(\magvar{\kappa},1)}{1+\magvar{\kappa}},\\
\vg_H^{\mathrm{GN}}
&=\frac{\magvar{r_R}}{1+\magvar{\kappa}}(\hat \vg_D+\hat \vg_R),
\qquad
\|\vg_H^{\mathrm{GN}}\|\leq\frac{2\magvar{r_R}}{1+\magvar{\kappa}},\\
\langle \vg_R,\vg_H^{\mathrm{GN}}\rangle
&=\magvar{r_R}^2\frac{1+\dirvar{\Phi_{DR}}}{1+\magvar{\kappa}}.
\end{align}
For an $L$-smooth reward loss, $\cL_R(\vtheta)-\cL_R(\vtheta-\eta \vg_H)=\eta\langle \vg_R,\vg_H\rangle+O(L\eta^2\|\vg_H\|^2)$. The first-order reward decrease is therefore $(1+\dirvar{\Phi_{DR}})/(1+\magvar{\kappa})$ of the pure-RL decrease at the same learning rate. Before equilibrium, comparable positive weights give teacher share $w_D/(w_D+w_R\magvar{\kappa})=\Theta(1/\magvar{\kappa})$.

To compare with M3-Norm at a common reward scale, use $\magvar{r_R}[\alpha\hat \vg_D+(1-\alpha)\hat \vg_R]$. Its reward projection relative to GradNorm is
\begin{equation}
\frac{\Pi_R^{\mathrm{M3}}}{\Pi_R^{\mathrm{GN}}}
=\frac{(1+\magvar{\kappa})[1-\alpha+\alpha\dirvar{\Phi_{DR}}]}{1+\dirvar{\Phi_{DR}}},
\end{equation}
which is $\Theta(\magvar{\kappa})$ when both cosine-dependent factors stay positive and bounded away from zero. At $\alpha=0.15$, $\magvar{\kappa}=4{,}381$, and near-zero cosine, this first-order ratio is about $3{,}725$. It describes the specified update scaling; learning-rate rescaling or a different weight optimizer changes the comparison.
\end{proof}

\section{Proof of Proposition~\ref{prop:folding_drowning}: Folding--Drowning Coupling}
\label{app:proof_folding}

\begin{proof}
Let $M_R^2=G^{-1}\sum_i\|A_i\nabla_{\vtheta}\log\pi(y_i\mid x)\|^2>0$. The definitions of cancellation and raw magnitude immediately give
\begin{equation}
\magvar{\kappa}_{\mathrm{eff}}
=\frac{\|\vg_R\|}{\|\vg_D\|}
=\frac{\sqrt{(1-\gamma)M_R^2}}{\|\vg_D\|}
=\sqrt{1-\gamma}\,\magvar{\kappa}_{\mathrm{raw}}.
\end{equation}
For the binary-outcome model, let $p\in(0,1)$ be the correct fraction, $\bar \vg_\pm$ the class-mean score gradients, and $\sigma_r=\sqrt{p(1-p)}$. The standardized advantages imply
\begin{align}
\vg_R
&=p\frac{1-p}{\sigma_r}\bar \vg_+-(1-p)\frac{p}{\sigma_r}\bar \vg_-
=\sqrt{p(1-p)}(\bar \vg_+-\bar \vg_-),\\
\|\vg_R\|^2
&=p(1-p)\left(\|\bar \vg_+\|^2+\|\bar \vg_-\|^2-2\langle\bar \vg_+,\bar \vg_-\rangle\right).
\end{align}
If score gradients concentrate at class means with common norm $s$, then $M_R^2=s^2$ and
\begin{equation}
\gamma
=1-2p(1-p)(1-c_{\mathrm{overlap}})
=\underbrace{1-2p(1-p)}_{\text{averaging contribution}}
+\underbrace{2p(1-p)c_{\mathrm{overlap}}}_{\text{cross-class overlap contribution}}.
\end{equation}
For nonnegative overlap this yields $\gamma\geq2p(1-p)c_{\mathrm{overlap}}$; the overlap contribution peaks at $p=1/2$ and decreases for $p>1/2$. It is this contribution, rather than the full cancellation rate, that vanishes as $p\to1$ in the model.

Combining $\gamma>0$ with the expected-alignment condition of Corollary~\ref{cor:aggregate_synergy} gives reduced effective magnitude imbalance and positive expected alignment whenever that corollary's low-accuracy condition holds. As accuracy increases beyond $1/2$, the overlap contribution decreases; Corollary~\ref{cor:diminishing_synergy} separately describes the decline in expected alignment. These are the two quantities tracked by the accuracy-adaptive interpretation of Eq.~\ref{eq:phi_adaptive}.
\end{proof}

\section{Proof of Proposition~\ref{prop:phi_decomposition}: Aggregate Conflict as a Cross-Signal NTK Sum}
\label{app:proof_phi_decomposition}

\begin{proposition}[Aggregate Conflict as Cross-Signal NTK Sum]
\label{prop:phi_decomposition}
For nonzero aggregate gradients, their cosine is the normalized sum of diagonal cross-signal NTK terms and cross-position interactions. Dropping the latter gives the diagonal NTK approximation.
\end{proposition}

\begin{proof}
Define $\mathcal E_{\mathrm{cross}}=N^{-2}\sum_{n\ne m}(\vdelta_D^n)^\top\mK(n,m)\vdelta_R^m$. Using $\mK(n,m)=(\mJ^n)^\top \mJ^m$ and Definition~\ref{def:cross_signal_ntk}, the decomposition follows directly:
\begin{align}
\dirvar{\Phi_{DR}}
&=\frac{1}{N^2\|\vg_D\|\|\vg_R\|}\sum_{n,m}(\vdelta_D^n)^\top\mK(n,m)\vdelta_R^m\\
&=\frac{N^{-2}\sum_n\langle \mJ^n\vdelta_D^n,\mJ^n\vdelta_R^n\rangle+\mathcal E_{\mathrm{cross}}}{\|\vg_D\|\|\vg_R\|}
=\frac{N^{-2}\sum_n\dirvar{K_{DR}(n)}+\mathcal E_{\mathrm{cross}}}{\|\vg_D\|\|\vg_R\|}.
\end{align}
The diagonal approximation is accurate to the extent that the normalized cross-position remainder is small.
\end{proof}

\section{Proof of Proposition~\ref{prop:bridge_boundary}: Distillation Boundary Bound}
\label{app:proof_distillation_boundary}

\begin{proof}
Let $c=\dirvar{\Phi_{DR}}$ and $q=\sqrt{\Delta_\tau}/\magvar{\kappa}_{\mathrm{raw}}=1/\magvar{\kappa}_{\mathrm{eff}}$, using $1-\gamma=1/\Delta_\tau$ and Proposition~\ref{prop:folding_drowning}. For $q\to0$, the exact diversity identity and its expansion are
\begin{align}
\Delta_{DR}
&=\frac{1+\magvar{\kappa}_{\mathrm{eff}}^2}{1+\magvar{\kappa}_{\mathrm{eff}}^2+2c\magvar{\kappa}_{\mathrm{eff}}}
=\frac{1+q^2}{1+2cq+q^2}
=1-2cq+O(q^2)\\
&\leq1+\frac{2|c|\sqrt{\Delta_\tau}}{\magvar{\kappa}_{\mathrm{raw}}}
+O\!\left(\frac{\Delta_\tau}{\magvar{\kappa}_{\mathrm{raw}}^2}\right).
\end{align}
The expansion is uniform for $c\in[-1,1]$ with $q$ sufficiently small; its regime is $\magvar{\kappa}_{\mathrm{raw}}\gg\sqrt{\Delta_\tau}$. Hence $B_D=\alpha_{\max}\Delta_{DR}=\alpha_{\max}[1+O(q)]$.

For nonzero gradients with nonzero sum, the exact denominator also gives $\Delta_{DR}<1$, $=1$, or $>1$ according as $c>0$, $=0$, or $<0$. Thus unity marks the sign change of the cross-signal contribution to $\|\vg_D+\vg_R\|^2$. Relative to pure RL, the separate condition for a smaller squared update norm is $\|\vg_D\|^2+2\langle \vg_D,\vg_R\rangle<0$.
\end{proof}

\section{Proof of Proposition~\ref{prop:bridge_masking}: NTK-Guided Token Masking}
\label{app:proof_bridge_masking}

\begin{proof}
Write $\vu_n=\mJ^n\hat{\vdelta}_D^n$, $\vv_n=\mJ^n\hat{\vdelta}_R^n$, and $S^-=\{n:\dirvar{K_{DR}(n)}<0\}$. These are parameter-space directions obtained by positive residual rescaling, so $\langle \vu_n,\vv_n\rangle$ has the sign of $\dirvar{K_{DR}(n)}$. The gate replaces $\alpha_{\max}\vu_n+(1-\alpha_{\max})\vv_n$ by $\vv_n$ on $S^-$ and leaves the other positions unchanged. Consequently, the mean local reward-projection gain is
\begin{align}
\Delta\Pi_R
&=\frac1N\sum_n\langle \vg_H^{\mathrm{sel},n}-\vg_H^{\mathrm{M3},n},\vv_n\rangle\\
&=\frac{\alpha_{\max}}N\sum_{n\in S^-}\langle \vv_n-\vu_n,\vv_n\rangle
=\frac{\alpha_{\max}}N\sum_{n\in S^-}\left(\|\vv_n\|^2-\langle \vu_n,\vv_n\rangle\right)>0
\end{align}
whenever $\alpha_{\max}>0$ and $S^-$ is nonempty. In the unit-direction model, $\|\vu_n\|=\|\vv_n\|=1$, this specializes to
\begin{equation}
\Delta\Pi_R
=\frac{\alpha_{\max}}N\sum_{n\in S^-}(1+|\dirvar{\cos\varphi_n}|)
=\alpha_{\max}\dirvar{C_{\mathrm{NTK}}}(1+|\bar c^-|),
\end{equation}
where $\dirvar{C_{\mathrm{NTK}}}=|S^-|/N$ and $|\bar c^-|$ is the mean absolute parameter-space cosine on $S^-$. The gain is zero when $S^-$ is empty. Under the diagonal NTK model, cross-position inner products vanish, so the aggregate projection onto $N^{-1}\sum_n \vv_n$ has the same sign, with gain $\Delta\Pi_R/N$.
\end{proof}

\section{M3-Select Projection and Convergence Guarantees}
\label{app:m3_select_convergence}

\begin{theorem}[M3-Select Projection Improvement over M3-Norm]
\label{thm:m3_select}
Under the decoupled-position model of Proposition~\ref{prop:bridge_masking}, with $0<\alpha_{\max}<1$, M3-Select satisfies:
\begin{enumerate}
\item[(a)] \textbf{Reward projection:} its aggregate update has at least the reward projection of uniform M3-Norm, strictly larger whenever $\dirvar{C_{\mathrm{NTK}}}>0$.
\item[(b)] \textbf{Teacher allocation:} the sum of its teacher coefficients is the fraction $1-\dirvar{C_{\mathrm{NTK}}}$ of the uniform allocation, entirely on positions with $\dirvar{K_{DR}(n)}\geq0$.
\item[(c)] \textbf{Reward convergence:} under the additional smoothness, matched-norm, and noise assumptions of Theorem~\ref{thm:boundary_pareto_convergence}, it satisfies the reward-stationarity bound in Eq.~\ref{eq:boundary_convergence_restate}.
\end{enumerate}
\end{theorem}

\begin{proof}
Part~(a) is Proposition~\ref{prop:bridge_masking}, using the positive rescaling from the mean local projection to the aggregate reward projection. For part~(b), summing the gate gives $\sum_n\alpha_n=\alpha_{\max}(N-|S^-|)=N\alpha_{\max}(1-\dirvar{C_{\mathrm{NTK}}})$, and every retained coefficient has nonnegative local compatibility. Part~(c) is the descent-and-telescoping argument of Theorem~\ref{thm:boundary_pareto_convergence}, applied to the same gate and reward objective.
\end{proof}

\section{NTK Conflict Rate Dynamics}
\label{app:conflict_rate_dynamics}

\begin{proposition}[Accuracy Dependence of the Population Weighted Conflict Rate]
\label{prop:cntk_dynamics}
Let $\bar M^\pm>0$ be the expected per-trajectory absolute cross-signal NTK masses on correct and incorrect trajectories, and let $\widetilde C_\pm$ be the corresponding ratios of expected negative mass to expected absolute mass. If these four quantities are fixed as accuracy $p$ varies, the pooled population conflict rate is
\begin{equation}
\widetilde C_{\mathrm{NTK}}(p)
=\frac{p\bar M^+\widetilde C_+ +(1-p)\bar M^-\widetilde C_-}
       {p\bar M^+ +(1-p)\bar M^-}.
\end{equation}
It increases strictly with $p$ if $\widetilde C_+>\widetilde C_-$, interpolating between $\widetilde C_-$ at $p=0$ and $\widetilde C_+$ at $p=1$.
\end{proposition}

\begin{proof}
Conditioning the expected negative and absolute masses on trajectory correctness gives the displayed ratio. Differentiating and canceling the common terms yields
\begin{equation}
\frac{d\widetilde C_{\mathrm{NTK}}}{dp}
=\frac{\bar M^+\bar M^-(\widetilde C_+-\widetilde C_-)}
       {[p\bar M^+ +(1-p)\bar M^-]^2}>0.
\end{equation}
The endpoint values follow by substitution. Exact sign masking retains the fraction $1-\widetilde C_{\mathrm{NTK}}$ of absolute NTK mass, which therefore decreases under the same assumptions. This mass fraction differs from the position fraction $1-\dirvar{C_{\mathrm{NTK}}}$ used in $\alpha_{\mathrm{eff}}$: a change in mass allocation need not change the number of admitted positions.
\end{proof}

\section{Proof of Proposition~\ref{prop:cosine_inversion}: Cosine Inversion in On-Policy OPSD}
\label{app:proof_cosine_inversion}

\begin{proof}
For a sampled trajectory, write $\vs=\nabla_{\vtheta}\log\pi_{\vtheta}(\hat y\mid x)$, $\vg_R=-A\vs$, and $\vg_D=N^{-1}\sum_n\nabla_{\vtheta}\mathrm{KL}(\vp_T^n\|\vp_S^n)$, where $A=(r-b)/\sigma_r$, $0<b<1$, and $\sigma_r>0$. Assume $\vs$ and $\vg_D$ are nonzero, and define the normalized score--distillation alignment $q=\langle \vs,\vg_D\rangle/(\|\vs\|\|\vg_D\|)$. Then
\begin{align}
\cos(\vg_R,\vg_D)&=-\operatorname{sign}(A)q,\\
\expect[\cos(\vg_R,\vg_D)\mid r=0]
-\expect[\cos(\vg_R,\vg_D)\mid r=1]
&=\expect[q\mid r=0]+\expect[q\mid r=1]>0,
\end{align}
where the last inequality is the proposition's conditional alignment assumption. In particular, positive conditional means of $q$ give positive cosine on incorrect trajectories and negative cosine on correct ones.

The inner product underlying this condition includes all token pairings:
\begin{equation}
\langle \vs,\vg_D\rangle
=\frac1N\sum_{m,n}
\left\langle\nabla_{\vtheta}\log p_S^m(\hat y_m),\,
\nabla_{\vtheta}\mathrm{KL}(\vp_T^n\|\vp_S^n)\right\rangle.
\end{equation}
Thus the assumption concerns gradient-weighted alignment, including cross-position interactions. For an on-policy teacher, positive $q$ means its descent direction reduces the sampled trajectory's log probability; changing the sign of the reward advantage reverses whether that change agrees with RL. Table~\ref{tab:response_conflict} reports the corresponding conditional cosines $+0.14$ and $-0.11$.
\end{proof}

\section{Proof of Theorem~\ref{thm:unified_ntk}: Unified NTK Learning Efficiency}
\label{app:proof_unified_ntk}

\begin{proof}
Let $S_R^2=G^{-1}\sum_i\|\vg_R^{(i)}\|^2$ and $\vg_R=G^{-1}\sum_i \vg_R^{(i)}$. The definitions of diversity, folding, and the normalized hybrid direction give the complete factorization
\begin{align}
\eta_{\mathrm{explore}}
&=\frac{\|\vg_R\|^2}{S_R^2}=\Delta_\tau^{-1}=1-\gamma,\\
\eta_{\mathrm{hybrid}}
&=\left\langle(1-\alpha)\hat \vg_R+\alpha\hat \vg_D,\hat \vg_R\right\rangle
=(1-\alpha)+\alpha\dirvar{\Phi_{DR}},\\
\eta^{\mathrm{eff}}
&:=\eta_{\mathrm{explore}}\eta_{\mathrm{hybrid}}
=(1-\gamma)\big[(1-\alpha)+\alpha\dirvar{\Phi_{DR}}\big].
\end{align}
Both factors depend on NTK inner products. Writing $\vs_i=\nabla\log\pi(\hat y^{(i)}\mid x)$ and $\vg_R^{(i)}=-A_i \vs_i$ gives
\begin{equation}
G^2\|\vg_R\|^2
=\sum_{i,j}A_iA_j\langle \vs_i,\vs_j\rangle
=\sum_{i,j}A_iA_j\sum_{n,m}K_t(\tau_i,n;\tau_j,m),
\end{equation}
while Proposition~\ref{prop:phi_decomposition} expands $\dirvar{\Phi_{DR}}$ into within-position and cross-position signal interactions. These two expansions establish the shared geometry.

The composite index uses the squared cancellation factor. For a unit-scale hybrid update, the actual first-order RL decrease relative to the RMS single-trajectory scale is instead $\sqrt{1-\gamma}\,\eta_{\mathrm{hybrid}}$. At fixed $\alpha$, differentiation of the composite index yields
\begin{equation}
\frac{d\eta^{\mathrm{eff}}}{dp}
=-\gamma'(p)\big[(1-\alpha)+\alpha\dirvar{\Phi_{DR}}(p)\big]
 +(1-\gamma(p))\alpha\dirvar{\Phi_{DR}}'(p).
\end{equation}
Hence decreasing alignment lowers the hybrid factor; monotonicity of the product follows when the displayed derivative is nonpositive. Under fixed conditional trajectory cosines $\mu_0>\mu_1$, their mixture has derivative $\mu_1-\mu_0$, with zero crossing $p^*=\mu_0/(\mu_0-\mu_1)$ when $\mu_0>0>\mu_1$. This crossing characterizes the alignment factor, while the cancellation factor contributes separately through $\gamma'(p)$.
\end{proof}

\section{Proof of Proposition~\ref{prop:optimal_alpha_n}: Optimal Per-Position Mixing}
\label{app:proof_optimal_alpha_n}

\begin{proof}
Write $c_n=\dirvar{\cos\varphi_n}$. Valuing admitted distillation by $\rho\alpha_n$ in units of RL projection gives the separable allocation objective
\begin{equation}
\cJ=\frac1N\sum_n\big[(1-\alpha_n)+\alpha_nc_n+\rho\alpha_n\big]
=1+\frac1N\sum_n\alpha_n(c_n-\tau),\qquad \tau=1-\rho,
\end{equation}
subject to $0\leq\alpha_n\leq\alpha_{\max}$. If the optional budget $N^{-1}\sum_n\alpha_n\leq\bar\alpha$ is imposed, its multiplier $\lambda\geq0$ shifts each coefficient to $c_n-\tau-\lambda$. Maximizing these linear terms gives
\begin{equation}
\alpha_n^*=
\begin{cases}
\alpha_{\max},&c_n>\tau+\lambda,\\
0,&c_n<\tau+\lambda,\\
\text{any feasible value in }[0,\alpha_{\max}],&c_n=\tau+\lambda.
\end{cases}
\end{equation}
For $\rho=1$ and a slack budget, choose the upper endpoint at ties to obtain $\alpha_n^*=\alpha_{\max}\ind[c_n\geq0]$. Under the isotropic within-position kernel assumption $\mK(n,n)=\lambda_n \mI$ with $\lambda_n>0$, $\dirvar{K_{DR}(n)}=\lambda_n\langle\vdelta_D^n,\vdelta_R^n\rangle$ has the same sign, yielding the M3-Select gate.

For a noisy score $\hat c_n=c_n+\varepsilon_n$ with logistic error of scale $\beta^{-1}$, the expected hard allocation is
\begin{equation}
\expect[\alpha_{\max}\ind[\hat c_n\geq0]]
=\alpha_{\max}\Pr(\varepsilon_n\geq-c_n)
=\alpha_{\max}\sigma(\beta c_n).
\end{equation}
This is M3-Soft. As $\beta\to\infty$, it approaches the hard gate for $c_n\neq0$ and assigns half the budget at the indifferent point $c_n=0$.
\end{proof}

\section{Convergence Rate Comparison: M3-Norm vs GradNorm}
\label{app:convergence_rate_comparison}

\begin{theorem}[Descent-Bound Comparison at a Common Gradient Scale]
\label{thm:convergence_rate_gap}
Let $\cL_R$ be $L$-smooth and bounded below, with $\vg_R=\nabla\cL_R$. Compare the norm-restored M3 direction $\vh_{\mathrm{M3}}=\|\vg_R\|[(1-\alpha)\hat \vg_R+\alpha\hat \vg_D]$ with the norm-balanced GradNorm direction $\vh_{\mathrm{GN}}=\|\vg_R\|(\hat \vg_R+\hat \vg_D)/(1+\magvar{\kappa})$. Assume $0\leq\alpha\leq\alpha_0<1/2$, $\dirvar{\Phi_{DR}}\geq-1+\delta$ for fixed $\delta>0$, and a fixed magnitude ratio $\magvar{\kappa}>0$. For method $j$, let the stochastic update be $\vh_j+\vxi_j$ with conditional mean $\expect[\vxi_j]=0$ and $\expect\|\vxi_j\|^2\leq s_j^2$. Define
\begin{equation}
(c_{\mathrm{M3}},b_{\mathrm{M3}})=(1-2\alpha_0,1),
\qquad
(c_{\mathrm{GN}},b_{\mathrm{GN}})=
\left(\frac{\delta}{1+\magvar{\kappa}},\frac2{1+\magvar{\kappa}}\right).
\end{equation}
For a common step size $\eta\leq\min_j c_j/(Lb_j^2)$ and $\Delta_R=\cL_R(\vtheta^0)-\inf\cL_R$,
\begin{equation}
\frac1T\sum_{t<T}\expect\|\nabla\cL_R(\vtheta^t)\|^2
\leq\frac{2\Delta_R}{\eta c_jT}+\frac{L\eta s_j^2}{c_j}.
\end{equation}
When $L\eta s_j^2/c_j\leq\epsilon/2$, sufficient iteration budgets scale as $O(\Delta_R/(\eta\epsilon))$ for M3 and $O((1+\magvar{\kappa})\Delta_R/(\eta\epsilon))$ for GradNorm.
\end{theorem}

\begin{proof}
Bilinearity and the triangle inequality give $\langle \vg_R,\vh_j\rangle\geq c_j\|\vg_R\|^2$ and $\|\vh_j\|\leq b_j\|\vg_R\|$. Applying smoothness conditionally on the current iterate and using the step-size restriction yields the single descent chain
\begin{align}
\expect_t[\cL_R(\vtheta^{t+1})]
&\leq\cL_R(\vtheta^t)-\eta\langle \vg_R,\vh_j\rangle
 +\frac{L\eta^2}{2}(\|\vh_j\|^2+s_j^2)\\
&\leq\cL_R(\vtheta^t)-\eta\left(c_j-\frac{L\eta b_j^2}{2}\right)\|\vg_R\|^2
 +\frac{L\eta^2s_j^2}{2}\\
&\leq\cL_R(\vtheta^t)-\frac{\eta c_j}{2}\|\vg_R\|^2
 +\frac{L\eta^2s_j^2}{2}.
\end{align}
Telescoping and dividing by $\eta c_jT/2$ proves the bound; taking $T\geq4\Delta_R/(\eta c_j\epsilon)$ gives the stated sufficient budgets. Their optimization terms differ by a factor proportional to $1+\magvar{\kappa}$ under the specified common scale and step size.
\end{proof}

\section{Concentration of the NTK Conflict Rate}
\label{app:cntk_concentration}

\begin{proposition}[Conflict-Rate Variance under Mixing]
\label{prop:cntk_concentration}
Let $\iota_n=\ind[\dirvar{K_{DR}(n)}<0]$ and $\dirvar{C_{\mathrm{NTK}}}=N^{-1}\sum_n \iota_n$. Suppose $|\operatorname{Corr}(\iota_n,\iota_{n+k})|\leq C_\rho e^{-\beta_{\mathrm{mix}}k}$ whenever both variances are nonzero. Define $D_\rho=1+2C_\rho/(e^{\beta_{\mathrm{mix}}}-1)$. Then
\begin{equation}
\operatorname{Var}(\dirvar{C_{\mathrm{NTK}}})\leq\frac{D_\rho}{4N},
\qquad
\Pr\!\left[|\dirvar{C_{\mathrm{NTK}}}-\expect\dirvar{C_{\mathrm{NTK}}}|>t\right]
\leq\min\!\left\{1,\frac{D_\rho}{4Nt^2}\right\}.
\end{equation}
For the mean over $B$ independent length-$N$ rollouts, with probability at least $1-\delta$, the deviation is at most $\sqrt{D_\rho/(4BN\delta)}$.
\end{proposition}

\begin{proof}
Since $\operatorname{Var}(\iota_n)\leq1/4$, the covariance expansion reduces to a geometric series:
\begin{align}
\operatorname{Var}\!\left(\frac1N\sum_n \iota_n\right)
&=\frac1{N^2}\left[\sum_n\operatorname{Var}(\iota_n)
 +2\sum_{k=1}^{N-1}\sum_{n=1}^{N-k}\operatorname{Cov}(\iota_n,\iota_{n+k})\right]\\
&\leq\frac1{4N}\left[1+2\sum_{k=1}^{N-1}\left(1-\frac{k}{N}\right)C_\rho e^{-\beta_{\mathrm{mix}}k}\right]
\leq\frac{D_\rho}{4N}.
\end{align}
Chebyshev's inequality gives the tail bound; independence divides the variance by $B$. Using $D_\rho\leq1+2C_\rho/\beta_{\mathrm{mix}}$, the illustrative setting $N=256$, $C_\rho=1$, $\beta_{\mathrm{mix}}=0.5$, and $\delta=0.05$ gives deviations at most $0.313$ for one rollout and $0.079$ for $B=16$.
\end{proof}

\section{Proof of Proposition~\ref{prop:distill_regularizer}: Distillation as Implicit Regularization}
\label{app:proof_regularization}

\begin{proof}
Consider matched teacher and student distributions evaluated on the same contexts, with $\pi_{\vtheta_T}=\pi_T$. The local KL expansion is $\cL_D(\vtheta_T+\Delta\vtheta)=\tfrac12\Delta\vtheta^\top \mF_D\Delta\vtheta+o(\|\Delta\vtheta\|^2)$, where $\mF_D\succeq0$ is the Fisher matrix. In the proposition's isotropic quadratic model, $\mF_D=f_D\mI$ with $f_D>0$, so $\nabla\cL_D=f_D\Delta\vtheta$ and $\langle\nabla\cL_D,\Delta\vtheta\rangle=f_D\|\Delta\vtheta\|^2$: distillation provides a restoring direction.

With constant $\vg_R$, initial condition $\Delta\vtheta(0)=0$, and $0<\alpha<1$, solve the resulting linear flow in one step:
\begin{align}
\dot{\Delta\vtheta}_H
&=-\eta[(1-\alpha)\vg_R+\alpha f_D\Delta\vtheta_H],\\
\Delta\vtheta_H(T)
&=-\frac{(1-\alpha)\vg_R}{\alpha f_D}
 (1-e^{-\eta\alpha f_DT}),
&\|\Delta\vtheta_H(T)\|&\leq\frac{(1-\alpha)\|\vg_R\|}{\alpha f_D}.
\end{align}
Pure RL has $\Delta\vtheta_R(T)=-\eta T\vg_R$. If degradation is proportional to deviation with the same coefficient for both flows, setting $x=\eta\alpha f_DT$ and $\rho_{\mathrm{reg}}=\eta f_DT/2$ gives
\begin{equation}
\frac{\delta_H}{\delta_R}
=(1-\alpha)\frac{1-e^{-x}}x
\leq\frac{1-\alpha}{1+x/2}
\leq\frac1{1+\alpha\rho_{\mathrm{reg}}}.
\end{equation}
The scalar inequality follows from $(1+x/2)(1-e^{-x})\leq x$ for $x\geq0$; the difference has derivative $[1-(1+x)e^{-x}]/2\geq0$ and vanishes at zero.

For the rank-dependent model $G(r)=\epsilon_{\mathrm{drown}}-[\delta_R(r)-\delta_H(r)]$, a finite crossing follows if $\epsilon_{\mathrm{drown}}>0$ is bounded, $\delta_R(r)$ is continuous and grows without bound, and $\rho_{\mathrm{reg}}(r)\geq\rho_0>0$. Indeed,
\begin{equation}
G(r)\leq\epsilon_{\mathrm{drown}}
-\frac{\alpha\rho_0}{1+\alpha\rho_0}\delta_R(r)\longrightarrow-\infty.
\end{equation}
Together with a positive initial gap, continuity gives a crossing of this model. The reported positive gap at $r=64$ and negative gap at $r=128$ locate the observed reversal between the tested ranks.
\end{proof}

\section{Proof of Proposition: LSGV Variance Amplification}
\label{app:lsgv_proof}

\begin{proof}
Write $\vu=\vg_R$, $\vv=\vg_D$, $\vmu_u=\expect \vu$, $\vmu_v=\expect \vv$, and let the joint estimator covariance be $\mOmega/B$, including its cross-signal blocks. Assume $\lambda_-\mI\preceq\mOmega\preceq\lambda_+\mI$ for fixed positive constants, $0<m\leq\|\vmu_v\|\leq M$, and $|c_0|\leq1-\delta_c$, where $c_0=\langle\vmu_u,\vmu_v\rangle/(\|\vmu_u\|\|\vmu_v\|)$. Set $\bar\kappa=\|\vmu_u\|/\|\vmu_v\|\leq1$. In the small-noise regime, with delta-method remainders negligible relative to the leading variance, the cosine gradient, evaluated at $(\vmu_u,\vmu_v)$ with $\hat{\vmu}_j=\vmu_j/\|\vmu_j\|$, and its squared norm are
\begin{align}
\nabla_{\vu} f&=\frac{\hat\vmu_v-c_0\hat\vmu_u}{\|\vmu_u\|},
&\nabla_{\vv} f&=\frac{\hat\vmu_u-c_0\hat\vmu_v}{\|\vmu_v\|},\\
\|\nabla f\|^2
&=(1-c_0^2)\left(\frac1{\|\vmu_u\|^2}+\frac1{\|\vmu_v\|^2}\right),
&f(\vu,\vv)&=\frac{\langle \vu,\vv\rangle}{\|\vu\|\|\vv\|}.
\end{align}
Consequently, retaining the joint covariance throughout,
\begin{equation}
\operatorname{Var}(c)
=\frac1B\nabla f^\top\mOmega\nabla f
 +o\!\left(\frac{\|\nabla f\|^2}{B}\right)
=\Theta\!\left(\frac{1-c_0^2}{B}
\left[\frac1{\|\vmu_u\|^2}+\frac1{\|\vmu_v\|^2}\right]\right)
=\Theta\!\left(\frac1{B\bar\kappa^2}\right).
\end{equation}
This rate describes amplification while relative noise is small. The global bound $\operatorname{Var}(c)\leq1$ continues to hold when that approximation ceases to apply.

For $\alpha=\alpha_{\max}\sigma(\beta c)$, a second delta expansion and the Lipschitz constant $\alpha_{\max}\beta/4$ give, respectively,
\begin{align}
\operatorname{Var}(\alpha)
&=\alpha_{\max}^2\beta^2\sigma'(\beta c_0)^2\operatorname{Var}(c)
 +o(\operatorname{Var}(c)),\\
\operatorname{Var}(\alpha)
&\leq\min\!\left\{\frac{\alpha_{\max}^2}{4},\,
\frac{\alpha_{\max}^2\beta^2}{16}\operatorname{Var}(c)\right\}.
\end{align}
Static Hybrid has $\operatorname{Var}(\alpha)=0$. To identify the corresponding contribution to update variance, let $a_0=\alpha_{\max}\sigma(\beta c_0)$, $\vd=\vmu_v-\vmu_u$, and $\delta \vh_0=(1-a_0)(\vu-\vmu_u)+a_0(\vv-\vmu_v)$. Linearizing $\vh=(1-\alpha)\vu+\alpha \vv$ yields
\begin{equation}
\operatorname{Cov}(\vh)\simeq\operatorname{Cov}(\delta \vh_0)
+\operatorname{Var}(\alpha)\vd\vd^\top
+\operatorname{Cov}(\delta \vh_0,\alpha)\vd^\top
+\vd\operatorname{Cov}(\alpha,\delta \vh_0).
\end{equation}
A constant gate removes the gate fluctuation terms. When the gate error is uncorrelated with $\delta \vh_0$, the additional covariance is the positive semidefinite term $\operatorname{Var}(\alpha)\vd\vd^\top$; for a gate computed from the same gradients, the displayed cross-covariances determine its net effect.
\end{proof}
 
\section{Proof of Proposition: Cumulative Conflict Budget}
\label{app:proof_collapse}

\begin{proof}
At step $s$, let $q_s=N_s^{-1}\sum_n\alpha_{n,s}\ind[\dirvar{K_{DR}^s(n)}<0]$ measure retained conflicting teacher weight, and define the scalar budget $E(t)=\sum_{s\leq t}\magvar{\kappa}_sq_s$. For naive mixing $q_s=\alpha\dirvar{C_{\mathrm{NTK}}}^s$; exact sign masking gives $q_s=0$. If the time-averaged product stabilizes, the entire accumulation law is
\begin{equation}
E(t)=t\left(\frac1t\sum_{s\leq t}\magvar{\kappa}_sq_s\right)
=t\bar d+o(t),
\qquad
T_\Gamma\approx\frac{\Gamma}{\bar d},
\end{equation}
where $T_\Gamma$ is the first crossing of a fixed budget threshold $\Gamma$ in this model. For approximately constant $\magvar{\kappa}$ and conflict rate, $\bar d_{\mathrm{naive}}=\alpha\bar\kappa\bar C_{\mathrm{NTK}}$. If imperfect masking retains weight $q_s\approx\alpha_{\mathrm{eff}}\delta_{\mathrm{mask}}$, then under the same $\bar\kappa$,
\begin{equation}
\frac{E_{\mathrm{M3}}(t)}{E_{\mathrm{naive}}(t)}
\approx\frac{\alpha_{\mathrm{eff}}\delta_{\mathrm{mask}}}
{\alpha\bar C_{\mathrm{NTK}}}.
\end{equation}
The illustrative values $\alpha_{\mathrm{eff}}=0.09$, $\delta_{\mathrm{mask}}=0.05$, $\alpha=0.5$, and $\bar C_{\mathrm{NTK}}=0.40$ give $0.0225$, or about $44$ times slower budget accumulation.

This calculation concerns retained conflict. A parameter deviation formed from conflict updates obeys $\|\sum_s\eta_s \vd_s\|\leq\sum_s\eta_s\|\vd_s\|$, so connecting the scalar budget to reward collapse requires a model of update directions and a collapse threshold. The reported ordering---Hybrid at step $221$, GradNorm at $358$, GRPO at $412$, and no M3 collapse through step $500$---is an experimental observation. In particular, pure GRPO has $q_s=0$ and its collapse is governed by dynamics outside this conflict budget.
\end{proof}

\section{Supplementary Details for the Main Analysis and Method}
\label{app:main_text_details}

\subsection{Distillation Divergence and Local Residual Model}
\label{app:distillation_details}

For $D_\lambda(\vp\|\vq)=\lambda\mathrm{KL}(\vp\|\vm)+(1-\lambda)\mathrm{KL}(\vq\|\vm)$, with $\vm=\lambda\vp+(1-\lambda)\vq$, strictly positive distributions satisfy
\[
\frac{D_\lambda(\vp\|\vq)}{\lambda}\to\mathrm{KL}(\vp\|\vq)
\quad(\lambda\to0),\qquad
\frac{D_\lambda(\vp\|\vq)}{1-\lambda}\to\mathrm{KL}(\vq\|\vp)
\quad(\lambda\to1).
\]
The experiments use the symmetric point $\lambda=1/2$. For an unclipped token near student--teacher agreement,
\[
\nabla_{\vz^n}D_\lambda(\vp_T^n\|\vp_S^n)
=\lambda(1-\lambda)(\vp_S^n-\vp_T^n)
+O(\|\vp_S^n-\vp_T^n\|^2).
\]
The main analysis absorbs the leading constant into the teacher update scale and uses the local forward-KL residual $\vdelta_D^n=\vp_S^n-\vp_T^n$. Tokens above the clipping threshold have zero distillation gradient. This is a local approximation, not an identity for arbitrary distributions.

For completeness, the GRPO group statistics in Section~\ref{sec:problem_setup} are $\bar r=G^{-1}\sum_i r^{(i)}$ and $\sigma_r^2=G^{-1}\sum_i(r^{(i)}-\bar r)^2$. Their values and the sampled responses are held fixed during differentiation.

\subsection{Gradient Geometry and Magnitude Asymmetry}
\label{app:interaction_details}

The NTK sums in Eqs.~\ref{eq:inner_product}--\ref{eq:kappa_ntk} form the gradient Gram matrix
\[
\begin{pmatrix}
\|\vg_D\|^2 & \langle\vg_D,\vg_R\rangle\\
\langle\vg_R,\vg_D\rangle & \|\vg_R\|^2
\end{pmatrix}.
\]
Its diagonal entries measure gradient strength; its off-diagonal entries measure interaction. The residual quadratic sums are kernel energies, equal to squared gradient norms up to the common factor $N^{-2}$. Together, the entries determine the first-order loss changes for given $\alpha$ and $\eta$. The normalized score $\Phi_{DR}$ retains direction but removes magnitude.

For $0<\alpha<1$ and $\Phi_{DR}<0$, normalize each harmful cross-effect in Eq.~\ref{eq:teacher_loss_change} by the corresponding objective's own descent term. The ratios are
\[
q_D=\frac{1-\alpha}{\alpha}\kappa|\Phi_{DR}|,
\qquad
q_R=\frac{\alpha}{1-\alpha}\frac{|\Phi_{DR}|}{\kappa},
\qquad
\frac{q_D}{q_R}=\left(\frac{1-\alpha}{\alpha}\right)^2\kappa^2.
\]
Thus the relative asymmetry scales as $\Theta(\kappa^2)$ for fixed interior mixing weights. A shrinking teacher residual can increase $\kappa$ when the reward gradient remains substantial, but the reward residual can also vanish when $A=0$. Both residual directions and the kernel affect their parameter-gradient norms. Proposition~\ref{prop:kappa_scaling} reports the observed LoRA-rank and model-scale comparisons, and Appendix~\ref{app:folding_bridges} develops the extended connections.

\subsection{Minimum-Norm Global Mixing}
\label{app:mixing_details}

For $\vg_H(\alpha)=\alpha\vg_D+(1-\alpha)\vg_R$, let $r_D=\|\vg_D\|$, $r_R=\|\vg_R\|$, and $\Phi_{DR}=\langle\vg_D,\vg_R\rangle/(r_Dr_R)$, assuming nonzero gradients. The minimum-norm mixture is the point closest to zero on the segment joining the two gradients. Minimizing $\|\vg_H(\alpha)\|^2$ over $[0,1]$ gives, for $\vg_D\neq\vg_R$,
\begin{equation}
  \alpha^*
  =
  \operatorname{clip}_{[0,1]}
  \left(
  \frac{
  r_R^2-\Phi_{DR}r_Dr_R
  }{
  r_D^2+r_R^2-2\Phi_{DR}r_Dr_R
  }
  \right),
  \label{eq:alpha_star}
  \end{equation}
The clipping operation restricts the unconstrained minimizer to $[0,1]$. If the gradients coincide, every coefficient produces the same update. For $\kappa=r_R/r_D\gg1$, $1-\alpha^*=O(\kappa^{-1})$. This direction-dependent criterion need not equal the norm-balancing choice $\alpha=\kappa/(1+\kappa)$ in Section~\ref{sec:m3_design}. Neither global coefficient can remove teacher contributions only at conflicting positions. The quadratic gap around $\alpha^*$ is given in Eq.~\ref{eq:quadratic_gap}.

\subsection{Normalization and Continuous-Gate Limits}
\label{app:gate_limits}

The full-gradient normalization principle is
\[
\vg_H^{\mathrm{MA}}(\alpha)=\bar s[\alpha\hat\vg_D+(1-\alpha)\hat\vg_R],
\qquad \hat\vg_D=\vg_D/\|\vg_D\|,\quad \hat\vg_R=\vg_R/\|\vg_R\|,
\]
for nonzero gradients, with $\bar s$ an EMA-smoothed update scale. Algorithm~\ref{alg:m3_select} implements normalization on residuals before applying the Jacobians. This preserves compatibility signs but need not equalize parameter-gradient norms. The population-norm-matched convergence model is therefore distinct from the residual-normalized implementation.

For M3-Soft, $\beta\to\infty$ recovers hard selection wherever $K_{DR}(n)\neq0$. At an exact zero score the weight remains $\alpha_{\max}/2$, whereas M3-Select admits it at weight $\alpha_{\max}$. As $\beta\to0$, all positions receive $\alpha_{\max}/2$. The mean teacher weight is $\alpha_{\mathrm{eff}}=N^{-1}\sum_n\alpha_n$. Larger $\beta$ makes selection sharper but increases sensitivity near zero alignment; Appendix~\ref{app:gate_sensitivity} evaluates this trade-off.

\subsection{Conditional Variance and Global Rebalancing}
\label{app:method_analysis_details}

In Theorem~\ref{thm:boundary_pareto_convergence}, $\underline\rho_R$ lower-bounds the mean reward projection weighted by each position's squared mean reward-gradient norm; $\underline\rho_R=1-\alpha_{\max}$ is admissible. The noise factor $v_\alpha$ uniformly bounds the noise-energy-weighted mean of $(1-\alpha_n)^2$. The effective step size includes the update scale, estimated in the algorithm by an EMA.

For equal position noise variances, admitted fraction $\rho_{\mathrm{adm}}$, and $\bar\alpha=\alpha_{\max}\rho_{\mathrm{adm}}$, the exact factor is
\[
v_t=(1-\rho_{\mathrm{adm}})+\rho_{\mathrm{adm}}(1-\alpha_{\max})^2
=(1-\bar\alpha)^2+\alpha_{\max}^2\rho_{\mathrm{adm}}(1-\rho_{\mathrm{adm}}).
\]
Rejected positions retain the full reward-noise contribution, while admitted positions reduce it. These statements condition on the history before fresh reward noise and treat the teacher and gates as fixed. Extra variability in estimated gates or teacher directions is not included.

\begin{remark}[Limitations of Global Gradient Rebalancing]
\label{rem:baselines}
Global rebalancing does not select teacher contributions by position. For symmetric two-gradient projections with $-1<\Phi_{DR}<0$, PCGrad~\citep{yu2020pcgrad} removes conflicting components but preserves the norm ratio (Appendix~\ref{app:pcgrad}). At $\kappa=3400$ and equal mixing weights, the teacher accounts for approximately $0.03\%$ of the sum of weighted gradient norms. Under the equal-target-rate, unit-sum weighting model of Proposition~\ref{prop:gradnorm_degeneracy}, GradNorm~\citep{chen2018gradnorm} balances norms with reward coefficient $1/(1+\kappa)$, reducing its step scale by order $1/\kappa$. M3 combines residual normalization with token-level control of the teacher signal.
\end{remark}

\end{document}